\documentclass[twoside,11pt]{article}

\usepackage{jmlr2e}

\newcommand{\new}[1]{{\color{black} #1}}
\newcommand{\stepsize}{\eta}

\usepackage{graphicx} 
\usepackage{subcaption} 
\usepackage{multirow}
\usepackage{mathtools}
\usepackage{relsize}
\usepackage{url}
\usepackage{nicefrac}
\usepackage{xspace}
\usepackage{algorithm}
\usepackage{algorithmic}
\usepackage{multicol}
\usepackage{float}
\usepackage{bbm}

\usepackage[dvipsnames]{xcolor}
\usepackage{pifont}
\newcommand{\cmark}{{\color{PineGreen}\ding{51}}}%
\newcommand{\xmark}{{\color{BrickRed}\ding{55}}}%

\usepackage[utf8]{inputenc} 
\usepackage[T1]{fontenc}    
\usepackage{hyperref}       
\usepackage{url}            

\usepackage{nicefrac}       
\usepackage{microtype}      
 
\usepackage{layout}
\usepackage{multirow}

\usepackage{amsfonts}
\usepackage{amsmath}
\usepackage{amssymb} 

\usepackage{mdframed} 
\usepackage{thmtools}

\usepackage[font=normalsize, labelfont=bf]{caption}

\usepackage{color}
\usepackage{graphicx}
\graphicspath{{plots/}}

\usepackage{xcolor}

\usepackage{algorithm}
\usepackage{verbatim}
\usepackage[algo2e]{algorithm2e}

\usepackage[symbol]{footmisc}

\def\<{\left\langle}
\def\>{\right\rangle}
\def\({\left(}
\def\){\right)}

\newcommand{\NC}{\cC_{\rm nat}}

\newcommand{\reals}{\mathbb{R}}

\DeclareMathOperator{\sign}{sign} 

\newcommand{\cC}{{\cal C}}

\newcommand{\cE}{{\cal E}}

\newcommand{\cO}{{\cal O}}
\newcommand{\cP}{{\cal P}}

\newcommand{\R}{\mathbb{R}} 

\newcommand{\Z}{\mathbb{Z}}

\newcommand{\B}{\mathbb{B}}
\newcommand{\U}{\mathbb{U}}

\newcommand{\eqdef}{\coloneqq}

\newcommand{\lin}[1]{\langle #1 \rangle} 
\newcommand{\dotprod}[1]{\left< #1\right>} 
\newcommand{\norm}[1]{\left\| #1 \right\|_2}      
\newcommand{\onenorm}[1]{\left\| #1 \right\|_1}      
\newcommand{\twonorm}[1]{\left\| #1 \right\|_2}      
\newcommand{\threenorm}[1]{\left\| #1 \right\|_3}      

\DeclareMathOperator{\Var}{Var}         
\DeclareMathOperator{\Prob}{Prob}

\DeclarePairedDelimiter\ceil{\lceil}{\rceil}

\newcommand{\Diag}[1]{\mathbf{Diag}\left( #1\right)}

\newcommand{\Exp}[1]{{{\rm E}}\left[#1\right] }    
\newcommand{\ExpC}[1]{{{\rm E_C}}\left[#1\right] }
\newcommand{\Expg}[1]{{{\rm E_{\nabla}}}\left[#1\right] }

\newcommand{\EE}[2]{{\rm E}_{#1}\left[#2\right] } 

\newcommand{\abs}[1]{| #1 |}

\newtheorem{assumption}{Assumption}

\usepackage[colorinlistoftodos,bordercolor=orange,backgroundcolor=orange!20,linecolor=orange,textsize=scriptsize]{todonotes}

\usepackage{lastpage}
\jmlrheading{24}{2023}{1-\pageref{LastPage}}{12/21; Revised
8/23}{10/23}{21-1548}{Aleksandr Beznosikov and Samuel Horv\'{a}th and Peter Richt\'{a}rik and  Mher Safaryan}

\ShortHeadings{On Biased Compression for Distributed Learning}{Beznosikov and Horv\'{a}th and Richt\'{a}rik and Safaryan}

\firstpageno{1}

\begin{document}

\title{On Biased Compression for Distributed Learning}

\author{\name Aleksandr Beznosikov\footnotemark[1] \email anbeznosikov@gmail.com \\
       \addr Computer, Electrical and Math. Sciences and Engineering Division\\
       King Abdullah University of Science and Technology, 23955, Thuwal, KSA\\
       Skolkovo Institute of Science and Technology, 121205, Moscow, Russia \\
       School of Applied Mathematics and Informatics\\
       Moscow Institute of Physics and Technology, 141701, Moscow, Russia\\
       \AND
       \name Samuel Horv\'{a}th\footnotemark[2] \email samuel.horvath@kaust.edu.sa \\
       \addr Computer, Electrical and Math. Sciences and Engineering Division\\
       King Abdullah University of Science and Technology, 23955, Thuwal, KSA\\
       \AND
       \name Peter Richt\'{a}rik \email peter.richtarik@kaust.edu.sa \\
       \addr Computer, Electrical and Math. Sciences and Engineering Division\\
       King Abdullah University of Science and Technology, 23955, Thuwal, KSA\\
       \AND
       \name Mher Safaryan\footnotemark[3] \email mher.safaryan.1@kaust.edu.sa \\
       \addr Computer, Electrical and Math. Sciences and Engineering Division\\
       King Abdullah University of Science and Technology, 23955, Thuwal, KSA\\}

\footnotetext[1]{The work in Sections 1-5 was conducted while A.\ Beznosikov was a research intern in the Optimization and Machine Learning Lab of Peter Richt\'{a}rik at KAUST; this visit was funded by the KAUST Baseline Research Funding Scheme. The work of A.\ Beznosikov in Section 6 was conducted in Skoltech and was supported by Ministry of Science and Higher Education grant No. 075-10-2021-068. 
Alphabetical author order.
}
\footnotetext[2]{Samuel Horv\'{a}th is currently affiliated with Mohamed bin Zayed University of Artificial Intelligence (MBZUAI).} 
\footnotetext[3]{Mher Safaryan is currently affiliated with Institute of Science and Technology Austria (ISTA).} 

\editor{Silvia Villa}

\maketitle

\begin{abstract}
In the last few years, various communication  compression techniques have  emerged  as an indispensable  tool helping to alleviate the communication bottleneck in distributed learning. However, despite the fact {\em biased} compressors often show superior performance in practice when compared to the much more studied and understood {\em unbiased} compressors,  very little is known about them. In this work we study three classes of biased compression operators, two of which are new, and their performance when applied to  (stochastic) gradient descent and distributed (stochastic) gradient descent.  We show for the first time that biased compressors can lead to linear convergence rates both in the single node and distributed settings. We prove that distributed compressed SGD method, employed with error feedback mechanism, enjoys the ergodic rate $\cO\left( \delta L \exp[-\frac{\mu K}{\delta L}]  + \frac{(C + \delta D)}{K\mu}\right)$, where $\delta\ge1$ is a compression parameter which grows when more compression is applied, $L$ and $\mu$ are the smoothness and strong convexity constants, $C$ captures stochastic gradient noise ($C=0$ if full gradients are computed on each node) and $D$ captures the variance of the gradients at the optimum ($D=0$ for over-parameterized models).  Further,  via a theoretical study of several synthetic and empirical distributions of communicated gradients, we shed light on why and by how much  biased compressors outperform  their unbiased variants. Finally, we propose several new biased compressors with promising theoretical guarantees and practical performance.
\end{abstract}


\section{Introduction}\label{sec:intro}

In order to achieve state-of-the-art performance, modern machine learning models need to be trained using large corpora of training data, and often feature an even larger number of trainable parameters \citep{Vaswani2019-overparam,Brown2020fewshot}. The data is typically collected in a distributed manner and stored across a network of edge devices, as is the case in federated learning \citep{FEDLEARN, FL2017-AISTATS, FL_survey_2019, FL-big}, or collected centrally   in a data warehouse composed of a large collection of commodity clusters. In either scenario, communication among the workers is typically the bottleneck.

Motivated by the need for more efficient training methods in traditional distributed and emerging federated environments, we consider optimization problems of the form
\begin{align}
\min \limits_{x \in \R^d}  \left\{f(x) \eqdef \frac{1}{n} \sum \limits_{i=1}^n f_i(x) \right\} \,, \label{eq:probR}
\end{align}
where  $x\in \R^d$ collects the parameters of a statistical  model to be trained, $n$ is the number of workers/devices, and $f_i(x)$ is the loss incurred by model $x$ on data stored on worker $i$.  The loss function $f_i:\R^d\to \R$ often has the form $$f_i(x) \eqdef \EE{\xi \sim \cP_i}{f_\xi(x)},$$ 
with $\cP_i$ being the distribution of training data owned by worker $i$. In federated learning applications, these local distributions can be very different and we do not impose any similarity assumption for them.

\subsection{Distributed optimization}  A fundamental baseline for  solving problem  \eqref{eq:probR} is (distributed) gradient descent (GD), performing updates of the form
\[ x^{k+1} = x^k - \frac{\stepsize^k}{n}\sum \limits_{i=1}^n \nabla f_i(x^k),\]
where $\stepsize^k>0$ is a stepsize. Due to the communication issues inherent to distributed systems, several enhancements to this baseline have been proposed that can better deal with the communication cost challenges of distributed environments, including
acceleration \citep{nesterov2013introductory, beck2009fista, allen2017katyusha}, reducing the number of iterations via momentum,
local methods \citep{FL2017-AISTATS,localSGD-AISTATS2020,Karimireddy2019}, reducing the number of communication rounds via performing multiple local updates before each communication round, and
communication compression \citep{1bit, qsgd,zipml,lim20183lc, alistarh2018convergence, deep, sign_descent_2019}, reducing the size of communicated messages via compression operators.

\subsection{Contributions} In this paper we contribute to a better  understanding of the latter approach to alleviating the communication bottleneck: {\em communication compression}. In particular, we study the theoretical properties of gradient-type methods which employ {\em biased} gradient compression operators, such as Top-$k$ sparsification~\citep{alistarh2018convergence}, or deterministic rounding~\citep{switchML}. Surprisingly, current
\footnote{Here we refer to the \href{https://arxiv.org/abs/2002.12410}{initial online appearance of our work} on February of 2020, after which several enhancements were developed. See Section \ref{sec:related-work} for more details.}
theoretical understanding of such methods is very limited. For instance, there is no general theory of such methods even in the $n=1$ case, only a handful of biased compression techniques have been proposed in the literature, we do not have any theoretical understanding of why biased compression operators could outperform their unbiased counterparts and when. More importantly, there is no good convergence theory for any gradient-type method with a biased compression in the crucial $n>1$ setting.

In this work we address all of the above problems. In particular, our main contributions are:

\begin{itemize}

\item[(a)] We define and study three parametric classes of  biased compression operators (see Section~\ref{sec:biased_compression_operators}), which we denote $\mathbb{B}^1(\alpha,\beta)$, $\mathbb{B}^2(\gamma,\beta)$ and $\mathbb{B}^3(\delta)$, the first two of which are new. We prove that, despite they are alternative parameterization of the same collection of operators, the last two are more favorable than the first one, thus highlighting the importance of parametrization and providing further reductions. We show how is the commonly used class of unbiased compression operators, which we denote $\mathbb{U}(\zeta)$, relates to these biased classes. We also study scaling and compositions of such compressors.

\item[(b)] We then proceed to give a long list of new and known  biased (and some unbiased) compression operators which belong to the above classes in Section~\ref{sec:examples}. A summary of all compressors considered can be found in Table~\ref{table:compressor-examples}.

\item[(c)]  In Section~\ref{sec:analysis_of_biased_GD} we analyze  compressed GD in the $n=1$ case for compressors belonging to all three classes under smoothness and strong convexity assumption. Our theorems generalize existing results which hold for unbiased operators in a tight manner, and also recover the rate of GD in this regime. Our linear convergence results are summarized in Table~\ref{table:iter_complexity}.

\begin{table}[t]
\begin{center}
\begin{tabular}{cccc}
\hline
Compressor & $\cC \in  \mathbb{B}^1(\alpha,\beta)$ & $\cC \in  \mathbb{B}^2(\gamma,\beta)$ & $ \cC\in  \mathbb{B}^3(\delta)$\\
Theorem & Theorem~\ref{thm:main-I} &  Theorem~\ref{thm:main-II} &  Theorem~\ref{thm:main-III} \\
Complexity & $\displaystyle \cO\left(\frac{\beta^2}{\alpha} \frac{L}{\mu} \log \frac{1}{\epsilon} \right)$ & $\displaystyle \cO\left(\frac{\beta}{\gamma} \frac{L}{\mu} \log \frac{1}{\epsilon}\right)$ & $\displaystyle \cO\left(\delta \frac{L}{\mu} \log \frac{1}{\epsilon}\right)$ \\
\hline
\end{tabular}
\end{center}
\caption{Complexity results for GD with biased compression. The identity compressor $\cC(x)\equiv x$ belongs to all  classes with $\alpha=\beta=\gamma=\delta=1$; all three results recover standard rate of GD.}
\label{table:iter_complexity}
\end{table}

\item[(d)]  We ask the question: do biased compressors outperform their unbiased counterparts in theory, and by how much? We answer this question by studying the performance of compressors under various synthetic and empirical statistical assumptions on the distribution of the entries of gradient vectors which need to be compressed. Particularly, we quantify the gains of the Top-$k$ sparsifier when compared against the unbiased Rand-$k$ sparsifier in Section~\ref{sec:stat}.

\item[(e)] Finally, we study the important $n>1$ setting in Section~\ref{sec:distributed} and argue by giving a counterexample that a naive application of biased compression to distributed GD might diverge.  We then show that distributed SGD method equipped with an error-feedback mechanism \citep{stich2019} provably handles biased compressors. In our main result (Theorem~\ref{thm:sparsified}; also see Table~\ref{table:distributed_thm}) we consider three learning schedules and iterate averaging schemes to provide three distinct convergence rates. Our analysis provides the first convergence guarantee for distributed gradient-type method which provably converges for biased compressors, and we thus solve a major open problem in the literature. 

\begin{table}
\begin{center}
\begin{tabular}{ccc}
\hline
Stepsizes & Weights & Rate\\
\hline
$\cO(\frac{1}{k})$  & $\cO(k)$ & $\displaystyle \cO \left(\frac{A_1}{K^2}  +  \frac{A_2}{ K} \right)$  \\
$\cO(1)$ & $\cO(e^{-k})$ & $\displaystyle \tilde \cO \left( A_3 \exp \left[- \frac{K}{A_4} \right] + \frac{A_2}{ K}\right)$ \\
$\cO(1)$ & $1$ & $\displaystyle \cO \left(  \frac{A_3}{ K } + \frac{A_5}{\sqrt{K}}  \right)$ \\
\hline
\end{tabular}
\end{center}

\caption{Ergodic convergence of distributed SGD with biased compression and error-feedback (Algorithm~\ref{alg}) for $L$-smooth and $\mu$-strongly convex functions ($K$ communications). Details are given in Theorem~\ref{thm:sparsified}. }
\label{table:distributed_thm}
\end{table}

\end{itemize}

\subsection{Related work}\label{sec:related-work}

There has been extensive work related to communication compression, mostly focusing on unbiased compressions~\citep{qsgd} as these are much easier to analyze. \new{In particular, it was shown~\citep{sigma_k} that both the classical method with unbiased compression~\citep{qsgd} and more advanced modifications~\citep{DIANA,DIANA-VR} can be considered as special versions of SGD. Subsequently, the results of~\citep{sigma_k} for strongly convex problems were transferred to general convex~\citep{khaled2020unified} and non-convex ~\citep{li2020unified} target functions. In the meantime,} works concerning biased compressions show stronger empirical results but with limited or no analysis~\citep{vogels,lin, sun}. There have been several attempts trying to address this issue, e.g., \citet{errorSGD} provided analysis for quadratics in distributed setting, \citet{zhao} gave analysis for momentum SGD with a specific  biased  compression, but under unreasonable assumptions, i.e., bounded gradient norm and memory. The first result that obtained linear rate of convergence for biased compression was done by \citet{karimireddy2019error}, but only for one node and under bounded gradient norm assumption, which was later overcome by \citet{stich2019}.  

After the initial online appearance of our work, there has been several enhancements in the literature. In particular, \citet{AjallStich2021biased} developed theory for non-convex objectives in the single node setup, \cite{Gorbunov2020EF-SGD} designed a novel error compansated SGD algorithm converging linearly in a more relaxed setting with the help of additional unbiased compressor, \cite{horvath2021a} proposed a simple trick to convert any biased compressor to corresponding induced (unbiased) compressor leading to improved theoretical guarantees. Recently, a new variant of error feedback mechanism was introduced in \citep{EF21,EF21-ext} showing an improved rates for distributed non-convex problems.

\subsection{Basic notation and definitions} We use $\lin{ x,y } \eqdef \sum_{i=1}^d x_i y_i$ to denote standard inner product of  $x,y\in\R^d$, where $x_i$ corresponds to the $i$-th component of $x$ in the standard basis in $\R^d$. This induces  the $\ell_2$-norm in $\R^d$ in the following way $\twonorm{x} \eqdef\sqrt{\lin{ x, x }}$.  We denote $\ell_p$-norms as $\|x\|_p \eqdef (\sum_{i=1}^d|x_i|^p)^{\nicefrac{1}{p}}$ for $p\in(1,\infty)$.
By $\Exp{\cdot}$ we denote mathematical expectation.
For a given differentiable function $f:\R^d\to \R$, we say that it is $L$-smooth if
$$f(y) \leq f(x) + \lin{\nabla f(x), y-x} + \frac{L}{2}\norm{y-x}^2,  \qquad \forall x,y\in \R^d.$$
We say that it is $\mu$-strongly convex if
$$f(y) \geq f(x) + \lin{\nabla f(x), y-x} + \frac{\mu}{2} \norm{y-x}^2, \qquad \forall x,y\in \R^d.$$

\section{Biased Compressors}\label{sec:biased_compression_operators}

By compression operator we mean a (possibly random) mapping $\cC\colon\R^d\to\R^d$ with some constraints.
Typically, literature considers {\em unbiased} compression operators $\cC$ with a bounded second moment, i.e.

\begin{definition}
Let $\zeta \geq 1$. We say that $\cC\in \mathbb{U}(\zeta)$ if $\cC$ is unbiased (i.e., $\Exp{\cC(x)}=x$ for all $x$) and if the  second moment is bounded as\footnote{\eqref{eq:unbiased} can be also written as $\Exp{ \twonorm{\cC(x) -x }^2 } \leq (\zeta-1)  \twonorm{x}^2 $.}
\begin{equation}\label{eq:unbiased}
 \Exp{ \twonorm{\cC(x)}^2 } \leq \zeta  \twonorm{x}^2, \qquad \forall x\in\R^d \,.
\end{equation} 

\end{definition}

\subsection{Three classes of biased compressors}

We instead focus on understanding {\em biased} compression operators, or ``compressors'' in short. We now introduce three classes of biased compressors, the first two are new, which can be seen as natural extensions of unbiased compressors.

\begin{definition}\label{def:comp_I}
We say that $\cC\in \mathbb{B}^1(\alpha,\beta)$ for some $\alpha,\beta>0$ if
\begin{equation} \label{eq:alpha-beta} \alpha \twonorm{x}^2 \leq \Exp{ \twonorm{\cC(x)}^2 } \leq \beta  \langle \Exp{\cC(x)},x \rangle, \qquad \forall x\in\R^d \; .
\end{equation}
\end{definition}

As we shall show next, the second inequality in \eqref{eq:alpha-beta} implies $\Exp{ \twonorm{\cC(x)}^2 } \leq \beta^2 \twonorm{x}^2$. 
\begin{lemma}\label{lem:second_ineq_implies} For any $x\in \R^d$, if  $\Exp{ \twonorm{\cC(x)}^2 } \leq  \beta \langle  \Exp{\cC(x)}, x \rangle$, then\begin{equation}\label{eq:noiuhfu93hufbuf}\Exp{ \twonorm{\cC(x)}^2 } \leq \beta^2\twonorm{x}^2.\end{equation} 
\end{lemma}
\begin{proof}
 Fix any $x\in \R^d$.  Applying Jensen's inequality, the second inequality in \eqref{eq:alpha-beta} and Cauchy-Schwarz, we get 
\begin{equation}\label{eq:bd8b38bdhbvhbyf}  \twonorm{ \Exp{\cC(x)}}^2   \leq \Exp{ \twonorm{\cC(x)}^2 } \overset{\eqref{eq:alpha-beta}}{\leq} \beta \langle  \Exp{\cC(x)}, x \rangle  \leq \beta   \twonorm{\Exp{ \cC(x) }} \twonorm{x}  .
\end{equation}
If $\Exp{ \cC(x) } \neq 0$, this implies
$\twonorm{\Exp{\cC(x)}} \leq \beta \twonorm{x}$. Plugging this back into \eqref{eq:bd8b38bdhbvhbyf}, we get \eqref{eq:noiuhfu93hufbuf}. If $\Exp{ \cC(x) } = 0$, then from \eqref{eq:alpha-beta} we see that $\Exp{ \twonorm{\cC(x)}^2 }=0$, and \eqref{eq:noiuhfu93hufbuf} holds trivially.
\end{proof}

In the second class, we require the inner product between uncompressed $x$ and compressed $\cC(x)$ vectors to dominate the squared norms of both vectors in expectation.

\begin{definition}\label{def:comp_II}
We say that $\cC\in \mathbb{B}^2(\gamma,\beta)$ for some $\gamma,\beta>0$ if
\begin{equation} \label{eq:alpha-betaII}
\max\left\{ \gamma \twonorm{x}^2 , \tfrac{1}{\beta} \Exp{\twonorm{\cC(x)}^2}\right\} \leq \lin{\Exp{\cC(x)},x } \qquad \forall x\in\R^d \,.
\end{equation}
\end{definition}

Finally, in the third class, we require the compression error $\twonorm{\cC(x) - x}^2$ to be strictly smaller than the squared norm $\twonorm{x}^2$ of the input vector $x$ in expectation.

\begin{definition}\label{def:comp_III}
We say $\cC\in \mathbb{B}^3(\delta)$ for some $\delta>0$ if
\begin{equation}\label{eq:biasedIII}
\Exp{ \twonorm{\cC(x) - x}^2 } \leq \left(1 - \frac{1}{\delta}\right)\twonorm{x}^2, \qquad  \forall x\in\R^d \; .
\end{equation}
\end{definition}

This last definition was also  considered by \citet{sparsified_sgd, cordonnier2018convex}. All three definitions require the compressed vector $\cC(x)$ to be in the neighborhood of the uncompressed vector $x$ so that initial information is preserved with some accuracy. We now establish several basic properties and connections between the classes. We first show that the three classes of biased compressors defined above are equivalent in the following sense: a compressor from any of those three classes can be shown to belong to all three classes with different parameters and after possible scaling.

\begin{theorem}[Equivalence between biased compressors]\label{thm:compression_properties}
Let $\lambda>0$ be a free scaling parameter.
\begin{enumerate}
\item If $\cC\in \mathbb{B}^1(\alpha,\beta)$, then 
\begin{itemize}
\item[(i)] $\beta^2\geq \alpha$ and $\lambda \cC \in \mathbb{B}^1(\lambda^2 \alpha  , \lambda\beta )$, and
\item[(ii)] \new{$\cC\in \mathbb{B}^2(\alpha/\beta,\beta)$} and $ \frac{1}{\beta} \cC \in \mathbb{B}^3(\beta^2/\alpha) $.
\end{itemize}

\item If $\cC \in \mathbb{B}^2(\gamma, \beta)$, then 
\begin{itemize}
\item[(i)] $\beta \geq \gamma$ and $\lambda \cC \in \mathbb{B}^2(\lambda \gamma, \lambda \beta)$, and
\item[(ii)]  $\cC \in \mathbb{B}^1(\gamma^2, \beta)$ and $\frac{1}{\beta}\cC \in \mathbb{B}^3(\beta/\gamma)$.
\end{itemize}

\item If $\cC \in \mathbb{B}^3(\delta)$, then 
\begin{itemize}
\item[(i)] $\delta \geq 1$, and 
\item[(ii)]  $\cC \in \mathbb{B}^2\left(\frac{1}{2\delta}, 2\right)\subseteq \mathbb{B}^1\left(\frac{1}{4\delta^2}, 2\right)$.
\end{itemize}
\end{enumerate}
\end{theorem}

\begin{proof} 
Let us prove this implications for each class separately.
\begin{enumerate}
\item Case $\cC\in \mathbb{B}^1(\alpha,\beta)$:
\begin{itemize}
\item[(i)] Let us choose any $x\neq 0$ and observe that \eqref{eq:alpha-beta} implies that $\Exp{\cC(x)}\neq 0$. Further, from \eqref{eq:alpha-beta} we get the bounds
$$\frac{\Exp{\twonorm{\cC(x)}^2}}{\langle \Exp{\cC(x)}, x\rangle } \leq \beta, \qquad \alpha \leq  \frac{\Exp{\twonorm{\cC(x)}^2}}{\twonorm{x}^2}.$$
Finally, \[
\beta^2 \geq \left(\frac{\Exp{\twonorm{\cC(x)}^2}}{\langle \Exp{\cC(x)}, x\rangle }\right)^2  \geq \frac{\Exp{\twonorm{\cC(x)}^2} \Exp{\twonorm{\cC(x)}^2}}{ \twonorm{ \Exp{\cC(x)}}^2 \twonorm{x}^2} \geq \alpha \frac{\Exp{\twonorm{\cC(x)}^2}}{ \twonorm{ \Exp{\cC(x)}}^2 } \geq \alpha,
\]
where the second inequality is due to Cauchy-Schwarz, and the last inequality follows by applying Jensen inequality.

The scaling property $\lambda \cC \in \mathbb{B}^1(\alpha \lambda^2 , \beta \lambda)$ follows directly from \eqref{eq:alpha-beta}.

\item[(ii)] \new{The inclusion $\cC\in \mathbb{B}^2(\alpha/\beta,\beta)$ follows directly from inequalities \eqref{eq:alpha-beta} and \eqref{eq:alpha-betaII}.} In view of (i), $\lambda \cC \in \mathbb{B}^1(\lambda^2 \alpha  ,  \lambda \beta)$. If we choose $\lambda \leq \frac{2}{\beta}$, then
\begin{eqnarray*}
\Exp{\twonorm{\lambda \cC(x)-x}^2} &=& \Exp{\twonorm{\lambda \cC(x)}^2} - 2\lin{ \Exp{\lambda \cC(x)}, x } + \twonorm{x}^2 \\
& \overset{\eqref{eq:alpha-beta} }{ \leq } & (\beta \lambda - 2) \lin{ \Exp{\lambda \cC(x)}, x } + \twonorm{x}^2\\
& \overset{\eqref{eq:alpha-beta} }{ \leq } & (\beta \lambda - 2) \frac{\alpha \lambda^2}{\beta \lambda}\twonorm{x}^2 + \twonorm{x}^2\\
& \overset{\eqref{eq:alpha-beta} }{ \leq } & \left( \alpha \lambda^2 - 2\frac{\alpha}{\beta} \lambda + 1\right) \twonorm{x}^2.
\end{eqnarray*}
Minimizing the above expression in $\lambda$, we get $\lambda=\frac{1}{\beta}$, and the result follows.

\end{itemize}

\item Case $\cC \in \mathbb{B}^2(\gamma, \beta)$. 

\begin{itemize}
\item[(i)] Using \eqref{eq:alpha-betaII} we get 
\begin{align*}
\gamma \leq \frac{\dotprod{\Exp{\cC(x)},x}}{\twonorm{x}^2} \leq \frac{\Exp{\twonorm{\cC(x)}^2}}{\sqrt{\Exp{\twonorm{\cC(x)}^2}\twonorm{x}^2}} \leq \beta \frac{\dotprod{\Exp{\cC(x)},x}}{\sqrt{\Exp{\twonorm{\cC(x)}^2}\twonorm{x}^2}} \leq \beta ,
\end{align*}
where the first and third inequalities follow from \eqref{eq:alpha-betaII} and the third and the last from Cauchy-Schwarz inequality with Jensen inequality.

The scaling property $\lambda \cC \in \mathbb{B}^2(\lambda \gamma, \lambda \beta)$ follows directly from \eqref{eq:alpha-betaII}.

\item[(ii)]  If $\cC \in \mathbb{B}^2(\gamma, \beta)$, then 
$
\Exp{ \twonorm{\cC(x)}^2 } \leq \beta  \langle \Exp{\cC(x)},x \rangle
$
and
\[
\gamma^2 \twonorm{x}^4 \overset{\eqref{eq:alpha-betaII}}{\leq} \langle \Exp{\cC(x)},x \rangle^2 \leq   \twonorm{\Exp{\cC(x)}}^2\twonorm{x}^2   \leq  \Exp{\twonorm{\cC(x)}^2} \twonorm{x}^2,
\]
where the second inequality is Cauchy-Schwarz, and the third is Jensen. Therefore, $\cC \in \mathbb{B}^1(\gamma^2, \beta)$.

Further, for any $\lambda >0$, we get
\begin{eqnarray*}
\Exp{\twonorm{\lambda \cC(x)-x}^2} &=& \Exp{\twonorm{\lambda \cC(x)}^2} - 2\lin{ \Exp{\lambda \cC(x)}, x } + \twonorm{x}^2 \\
 &=&\lambda^2  \Exp{\twonorm{ \cC(x)}^2} - 2\lambda \lin{ \Exp{ \cC(x)}, x } + \twonorm{x}^2 \\
& \overset{\eqref{eq:alpha-betaII} }{ \leq } & (\lambda \beta - 2) \lambda \lin{ \Exp{ \cC(x)}, x } + \twonorm{x}^2.
\end{eqnarray*}
If we choose $\lambda = \frac{1}{\beta}$, then we can continue as follows:
\begin{eqnarray*}
\Exp{\twonorm{\lambda \cC(x)-x}^2}  & \leq  & -\frac{1}{\beta} \lin{ \Exp{ \cC(x)}, x } + \twonorm{x}^2\\
& \overset{\eqref{eq:alpha-betaII} }{ \leq } & \left(  1-\frac{\gamma}{\beta}\right) \twonorm{x}^2,
\end{eqnarray*}
whence $\frac{1}{\beta}\cC \in \mathbb{B}^3(\beta/\gamma)$.

\end{itemize}


\item Case $\cC \in \mathbb{B}^3(\delta)$. 

\begin{itemize}
\item[(i)] Pick $x\neq 0$. Since $0\leq \Exp{\twonorm{\cC(x)-x}^2} \leq \left(1-\frac{1}{\delta}\right)\norm{x}^2$ and we assume $\delta>0$, we must necessarily have $\delta\geq 1$.

\item[(ii)] If $\cC \in \mathbb{B}^3(\delta)$ then 
\begin{equation*}
 \Exp{ \twonorm{\cC(x)}^2 } - 2 \dotprod{\Exp{\cC(x)},x} + \frac{1}{\delta}\twonorm{x}^2 \leq 0,
\end{equation*}
which implies  that
\begin{equation*}
\frac{1}{2\delta}\twonorm{x}^2 \leq  \dotprod{\Exp{\cC(x)},x} \qquad \text{and} \qquad  \Exp{ \twonorm{\cC(x)}^2} \leq  2 \dotprod{\Exp{\cC(x)},x}.
\end{equation*}
Therefore, $\cC \in \mathbb{B}^2\left(\frac{1}{2\delta}, 2\right)\subseteq \mathbb{B}^1\left(\frac{1}{4\delta^2}, 2\right)$.
\end{itemize}
\end{enumerate}

\end{proof}

Next, we show that, with a proper scaling, any unbiased compressor also belongs to all the three classes of biased compressors.

\begin{theorem}[From unbiased to biased with scaling]\label{thm:unbiased_to_biased}
If $\cC\in \mathbb{U}(\zeta)$, then for the scaled operator $\lambda\cC$ we have
\begin{itemize}
\item[(i)] $\lambda \cC\in \mathbb{B}^1(\lambda^2,\lambda \zeta )$ for $\lambda>0$,
\item[(ii)] $\lambda \cC\in \mathbb{B}^2(\lambda,\lambda \zeta )$ for $\lambda>0$, 
\item[(iii)] $\lambda \cC\in \mathbb{B}^3\left(\frac{1}{\lambda(2 - \zeta \lambda)}\right)$ for $\zeta \lambda < 2$.
\end{itemize}
\end{theorem}

\begin{proof} 
Let $\cC\in \mathbb{U}(\zeta)$.

\begin{itemize}


\item [(i)] Given any $\lambda>0$, consider the scaled operator $\lambda \cC$. We have
\begin{eqnarray*} \lambda^2 \twonorm{x}^2 = \twonorm{\Exp{\lambda \cC(x)}}^2  \leq    \Exp{\twonorm{\lambda \cC(x)}^2}  \leq \lambda^2 \zeta \twonorm{x}^2  =  \lambda \zeta \langle \Exp{\lambda \cC(x)},x \rangle,
\end{eqnarray*}
whence $\cC\in \mathbb{B}^1(\lambda^2,\lambda \zeta)$.

\item [(ii)] Given any $\lambda>0$, consider the scaled operator $\lambda \cC$. We have
\begin{eqnarray*} \lambda \twonorm{x}^2 &=& \langle \Exp{\lambda \cC(x)},x \rangle, \\
\Exp{\twonorm{\lambda \cC(x)}^2} 
& \leq & \lambda^2 \zeta \twonorm{x}^2 = \lambda \zeta \langle \Exp{\lambda \cC(x)},x \rangle,
\end{eqnarray*}
whence $\lambda \cC\in \mathbb{B}^2(\lambda,\lambda \zeta)$.

\item [(iii)] Given $\lambda>0$ such that $\lambda \zeta < 2$, consider the scaled operator $\lambda \cC$. We have
\begin{eqnarray*} 
\Exp{\norm{\lambda \cC(x)- x}^2} &=& \Exp{\norm{\lambda \cC(x)}^2} - 2 \lin{\Exp{\lambda \cC(x)}, x} + \norm{x}^2 \\
&\leq& (\zeta \lambda^2 - 2 \lambda + 1)\norm{x}^2   
\end{eqnarray*}
whence $\lambda \cC\in \mathbb{B}^3\left(\frac{1}{\lambda(2 - \zeta \lambda)}\right)$.

\end{itemize}

\end{proof}

 \subsection{Examples of biased compressors: old and new} \label{sec:examples}

We  now give some examples of compression operators belonging to the classes $\mathbb{B}^1$,  $\mathbb{B}^2$,  $\mathbb{B}^3$ and  $\mathbb{U}$. Several of them are new. For a summary, refer to Table~\ref{table:compressor-examples}.

\begin{table*}[!t]
{\footnotesize
\resizebox{\linewidth}{!}{
\begin{tabular}{lcccccc}
\hline
Compression Operator $\cC$    &  Unbiased? &  $\alpha$ & $\beta$ & $\gamma$ & $\delta$ & $\zeta$\\
\hline 
Unbiased random sparsification  & \cmark  &  &  &  &  & $\nicefrac{d}{k}$ \\ 
Biased random sparsification  {\bf [NEW]} & \xmark  & $q$ & $1$ & $q$ & $\nicefrac{1}{q}$ & \\ 
Adaptive random sparsification {\bf [NEW]} & \xmark & $\nicefrac{1}{d}$ & $1$ & $\nicefrac{1}{d}$ & $d$ & \\
Top-$k$ sparsification \citep{alistarh2018convergence}  & \xmark & $\nicefrac{k}{d}$ & $1$ & $\nicefrac{k}{d}$ & $\nicefrac{d}{k}$ & \\
General unbiased rounding {\bf [NEW]} & \cmark & & & & & 
$\tfrac{1}{4} \sup\left(\tfrac{a_k}{a_{k+1}} + \tfrac{a_{k+1}}{a_k} + 2\right)$ \\
Unbiased exponential rounding {\bf [NEW]} & \cmark & & & & & $\tfrac{1}{4}\left( b+\frac{1}{b}+2\right)$ \\
Biased exponential rounding {\bf [NEW]} & \xmark  & $\left(\frac{2}{b+1}\right)^2$ & $\frac{2b}{b+1}$ & $\frac{2}{b+1}$ & $\frac{(b+1)^2}{4b}$ & \\
Natural compression \citep{Cnat} & \cmark & & & & & $\nicefrac{9}{8}$ \\
General exponential dithering {\bf [NEW]} & \cmark & & & & & $\zeta_b$ \\
Natural dithering  \citep{Cnat} & \cmark & & & & & $\zeta_2$ \\
Top-$k$ + exponential dithering  {\bf [NEW]} & \xmark & $\nicefrac{k}{d}$ & $\zeta_b$ & $\nicefrac{k}{d}$ & $\zeta_b\nicefrac{d}{k}$ &  \\
\hline
\end{tabular}
}
}
\caption{Compressors $\cC$ described in Section~\ref{sec:examples} and their membership in  $\mathbb{B}^1(\alpha,\beta)$, $\mathbb{B}^2(\gamma,\beta)$,  $\mathbb{B}^3(\delta)$ and $\mathbb{U}(\zeta)$.}
\label{table:compressor-examples}
\end{table*}

\begin{itemize}

\item[(a)] For $k \in [d]\eqdef \{1,\dots,d\}$, the {\bf unbiased random  (aka Rand-$k$) sparsification} operator is defined via
\begin{equation}\label{ex:ur-sparse}
\cC(x) \eqdef \frac{d}{k}\sum \limits_{i\in S}x_ie_i,
\end{equation}
where $S\subseteq [d]$ is the $k$-nice sampling; i.e., a subset of $[d]$ of cardinality $k$ chosen uniformly at random, and $e_1,\dots,e_d$ are the standard unit basis vectors in $\R^d$.

\begin{lemma}\label{lem-ex:ur-sparse}
The Rand-$k$ sparsifier (\ref{ex:ur-sparse}) belongs to $\U(\tfrac{d}{k})$.
\end{lemma}

\item[(b)] Let $S\subseteq [d]$ be a random set, with probability vector $p\eqdef (p_1,\dots,p_d)$, where  $p_i \eqdef \Prob(i\in S)>0$ for all $i$ (such a set is called a proper sampling~\citep{PCDM}). Define {\bf biased random sparsification} operator via
\begin{equation}\label{ex:br-sparse}
 \cC(x) \eqdef \sum \limits_{i\in S} x_i e_i.
\end{equation}

\begin{lemma}\label{lem-ex:br-sparse}
Letting $q \eqdef \min_i p_i$, the biased random sparsification operator (\ref{ex:br-sparse}) belongs to $\mathbb{B}^1(q, 1)$, $\mathbb{B}^2(q, 1)$, $\mathbb{B}^3(\nicefrac{1}{q})$.
\end{lemma}

\item[(c)] {\bf Adaptive random sparsification} is defined via
\begin{equation}\label{ex:ar-sparse}
\cC(x) \eqdef x_i e_i \quad \text{ with probability } \quad \frac{|x_i|}{\onenorm{x}}.
\end{equation}

\begin{lemma}\label{lem-ex:ar-sparse}
Adaptive random sparsification operator (\ref{ex:ar-sparse}) belongs to $\mathbb{B}^1(\frac{1}{d},1)$, $\mathbb{B}^2(\frac{1}{d},1)$, $\mathbb{B}^3(d)$.
\end{lemma}

\item[(d)] {\bf Greedy (aka Top-$k$) sparsification} operator is defined via
\begin{equation}\label{ex:top-sparse}
\cC(x) \eqdef \sum \limits_{i=d-k+1}^d x_{(i)} e_{(i)},
\end{equation}
where coordinates are ordered by their magnitudes so that $\abs{x_{(1)}} \leq \abs{x_{(2)}} \leq \cdots \leq \abs{x_{(d)}}$.

\begin{lemma}\label{lem-ex:top-sparse}
Top-$k$ sparsification operator (\ref{ex:top-sparse}) belongs to $\mathbb{B}^1(\frac{k}{d},1)$, $\mathbb{B}^2(\frac{k}{d},1)$, and $\mathbb{B}^3(\frac{d}{k})$.
\end{lemma}

\item[(e)] Let $(a_k)_{k\in\Z}$ be an arbitrary increasing sequence of positive numbers such that $\inf a_k = 0$ and $\sup a_k = \infty$. Then {\bf general unbiased rounding} $\cC$ is defined as follows: if $a_k\le |x_i|\le a_{k+1}$ for some coordinate $i\in[d]$, then
\begin{equation}\label{ex:gu-rounding}
\cC(x)_i = 
\begin{cases}
    \sign(x_i)a_k     & \text{ with probability } \quad \frac{a_{k+1}-|x_i|}{a_{k+1}-a_k}\\ 
    \sign(x_i)a_{k+1} & \text{ with probability } \quad \frac{|x_i| - a_{k}}{a_{k+1}-a_k}\\
\end{cases}
\end{equation}

\begin{lemma}\label{lem-ex:gu-rounding}
General unbiased rounding operator (\ref{ex:gu-rounding}) belongs to $\U(\zeta)$, where
$$
\zeta = \frac{1}{4} \sup_{k\in\Z}\(\frac{a_k}{a_{k+1}} + \frac{a_{k+1}}{a_k} + 2\).
$$
\end{lemma}

Notice that $\zeta$ is minimizing for exponential roundings $a_k=b^k$ with some basis $b>1$, in which case $\zeta = \tfrac{1}{4}\(b+\nicefrac{1}{b}+2\)$.

\item[(f)] Let $(a_k)_{k\in\Z}$ be an arbitrary increasing sequence of positive numbers such that $\inf a_k = 0$ and $\sup a_k = \infty$. Then {\bf general biased rounding} is defined via
\begin{equation}\label{ex:gb-rounding}
\cC(x)_i \eqdef \sign(x_i)\arg \min \limits_{t\in(a_k)} |t-|x_i||,\qquad i\in[d].
\end{equation}

\begin{lemma}\label{lem-ex:gb-rounding}
General biased rounding operator (\ref{ex:gb-rounding}) belongs to $\B^1(\alpha, \beta)$, $\B^2(\gamma, \beta)$, and $\B^3(\delta)$, where $$\beta = \sup_{k\in\Z}\frac{2a_{k+1}}{a_k+a_{k+1}}, \qquad \gamma = \inf_{k\in\Z}\frac{2a_k}{a_k+a_{k+1}}, \qquad \alpha = \gamma^2, \qquad \delta = \sup_{k\in\Z}\frac{\(a_k + a_{k+1}\)^2}{4a_k a_{k+1}}.$$
\end{lemma}

In the special case of exponential rounding $a_k=b^k$ with some base $b>1$, we get
$$\alpha = \(\frac{2}{b+1}\)^2, \quad \beta = \frac{2b}{b+1}, \quad \gamma = \frac{2}{b+1}, \quad \delta = \frac{(b+1)^2}{4b}.$$

\begin{remark}
Plugging these parameters into the iteration complexities of Table~\ref{table:iter_complexity}, we find that the class $\B^3$ gives the best iteration complexity as $\tfrac{\beta^2}{\alpha} = b^2 > \tfrac{\beta}{\gamma} = b > \delta = \tfrac{(b+1)^2}{4b}$.
\end{remark}

\item[(g)] {\bf  Natural compression} operator $\cC_{nat}$ of \citet{Cnat} is the special case of general unbiased rounding operator (\ref{ex:gu-rounding}) when $b=2$. So, $$\cC_{nat}\in\U\left(\frac{9}{8}\right).$$

\item[(h)] For $b>1$, define {\bf general exponential dithering} operator with respect to $l_p$-norm and with $s$ exponential levels $0<b^{1-s}<b^{2-s}<\dots<b^{-1}<1$ via
\begin{equation}\label{ex:ge-dithering}
\cC(x) \eqdef \|x\|_p \times \sign(x) \times \xi\(\frac{|x_i|}{\|x\|_p}\),
\end{equation}
where the random variable $\xi(t)$ for $t\in[b^{-u-1},b^{-u}]$ is set to either $b^{-u-1}$ or $b^{-u}$ with probabilities proportional to $b^{-u}-t$ and $t-b^{-u-1}$, respectively.

\begin{lemma}\label{lem-ex:ge-dithering}
General exponential dithering operator (\ref{ex:ge-dithering}) belongs to $\U(\zeta_b)$, where, letting $r = \min(p,2)$,
\begin{equation}\label{ge-dithering-zeta}
\zeta_b = \frac{1}{4}\(b+\frac{1}{b}+2\) + d^{\frac{1}{r}}b^{1-s}\min(1, d^{\frac{1}{r}}b^{1-s}).
\end{equation}
\end{lemma}

\item[(i)] {\bf  Natural dithering} introduced by \citet{Cnat} without norm compression is the spacial case of general exponential dithering (\ref{ex:ge-dithering}) when $b=2$.

\item[(j)] \new{{\bf Ternary quantization} of \citet{terngrad} is the extreme case of general exponential dithering \eqref{ex:ge-dithering} with $s=1$ levels and $b=1$.}

\item[(k)] {\bf Top-$k$ combined with exponential dithering.} Let $\cC_{top}$ be the Top-$k$ sparsification operator (\ref{ex:top-sparse}) and $\cC_{dith}$ be general exponential dithering operator (\ref{ex:ge-dithering}) with some base $b>1$ and parameter $\zeta_b$ from (\ref{ge-dithering-zeta}). Define a new compression operator as the composition of these two:
\begin{equation}\label{ex:top-gendith}
\cC(x) \eqdef \cC_{dith}\(\cC_{top}(x)\).
\end{equation}

\begin{lemma}\label{lem-ex:top-gendith}
The composition operator \eqref{ex:top-gendith} of Top-$k$ sparsification and exponential dithering with base $b$ belongs to $\B^1(\tfrac{k}{d}, \zeta_b)$, $\B^2(\tfrac{k}{d}, \zeta_b)$, $\B^3(\tfrac{d}{k}\zeta_b)$, where $\zeta_b$ is as in \eqref{ge-dithering-zeta}.
\end{lemma}

\end{itemize}

\section{Gradient Descent with Biased Compression}\label{sec:analysis_of_biased_GD}

As we discussed in previous section, compression operators can have different equivalent parametrizations. Next, we aim to investigate the influence of those parametrizations on the theoretical convergence rate of an algorithm employing compression operators. To achieve clearer understanding of the interaction of compressor parametrization and convergence rate, we first consider the single node, unconstrained optimization problem 
$$ \min_{x\in \R^d} f(x),$$
where $f:\R^d\to \R$ is $L$-smooth and $\mu$-strongly convex. We study the method
\begin{equation} \label{eq:CGD} \tag{CGD} \boxed{x^{k+1} =x^k -\stepsize \cC^k(\nabla f(x^k))}\, ,\end{equation}
where $\cC^k:\R^d\to \R^d$ are (potentially biased) compression operators belonging to one of the classes $\mathbb{B}^1$, $\mathbb{B}^2$ and $\mathbb{B}^3$ studied in the previous sections, and $\eta>0$ is a stepsize. We refer to this method as  CGD: Compressed Gradient Descent.
 
\subsection{Complexity theory} We now establish three theorems, performing complexity analysis for each of the three classes $\mathbb{B}^1$, $\mathbb{B}^2$ and $\mathbb{B}^3$ individually. Let $\cE_k\eqdef \Exp{f(x^k)} - f(x^\star)$, with $\cE_0 = f(x^0) - f(x^\star)$.
  
\begin{theorem}\label{thm:main-I} Let $\cC\in \mathbb{B}^1(\alpha,\beta)$. Then as long as $0\leq \stepsize \leq \frac{2}{\beta L}$, we have
$\cE_{k} \leq  \left(1- \frac{\alpha}{\beta} \stepsize \mu ( 2 - \stepsize \beta  L ) \right) \cE_{k-1}.$
If we choose $\stepsize = \frac{1}{\beta L}$, then
$$\cE_k\leq  \left(1- \frac{\alpha}{\beta^2}\frac{\mu}{ L}  \right)^k \cE_0.$$
\end{theorem}

\begin{theorem}\label{thm:main-II}  Let $\cC\in \mathbb{B}^2(\gamma,\beta)$. Then as long as $0\leq \stepsize \leq \frac{2}{\beta L}$, we have
$\cE_{k} \leq  \left(1-  \gamma \stepsize \left(2 -\stepsize \beta)  L \right) \right) \cE_{k-1}.$
If we choose $\stepsize = \frac{1}{\beta L}$, then
$$\cE_k \leq  \left(1- \frac{\gamma}{\beta }\frac{\mu}{ L}  \right)^k \cE_0.$$
\end{theorem}

\begin{theorem}\label{thm:main-III} Let $ \cC\in \mathbb{B}^3(\delta)$. Then as long as  $0\leq \stepsize \leq \frac{1}{L}$, we have
$\cE_{k} \leq  \left(1- \frac{1}{\delta} \stepsize \mu \right)^k \cE_{k-1}.$ If we choose $\eta=\frac{1}{L}$, then $$\cE_k \leq  \left(1- \frac{1}{\delta} \frac{\mu}{L}  \right)^k \cE_0.$$
\end{theorem}

The iteration complexity for these results can be found in Table~\ref{table:iter_complexity}. Note that the identity compressor $\cC(x)\equiv x$ belongs to $ \mathbb{B}^1(1,1), \mathbb{B}^2(1,1)$, and  $\mathbb{B}^3(1)$, hence all these result exactly recover the rate of GD. In the first two theorems, scaling the compressor by a positive scalar $\lambda >0$ does not influence the rate (see Theorem~\ref{thm:compression_properties}).

\new{One can note that it is possible to adapt  the results of Theorems \ref{thm:main-I}-\ref{thm:main-III} to the case with the time-varying parameters of compression operators~\citep{lin2017deep,agarwal2020accordion}. In the case of Theorem~\ref{thm:main-III} and $\mathbb{B}^3$ compressors, it is sufficient to introduce a sequence $\{\delta^i\}$ responsible for the changes of the parameter $\delta$ and obtain the following convergence: $\cE_k \leq  \left[\prod_{i=0}^{k-1}\left(1- \frac{1}{\delta^i} \frac{\mu}{L}  \right)\right] \cE_0$. In the case of Theorems~\ref{thm:main-I} and \ref{thm:main-II}, the results are similar, but it is necessary to add in the algorithm \eqref{eq:CGD} the possibility of using time-varying steps $\eta^k = \frac{1}{\beta^k L}$.}

\subsection{$\mathbb{B}^3$ and $\mathbb{B}^2$ are better than $\mathbb{B}^1$} In light of the results above, we make the following observation. If $\cC\in \mathbb{B}^1(\alpha,\beta)$, then by Theorem~\ref{thm:compression_properties}, $\frac{1}{\beta}\cC\in \mathbb{B}^3(\frac{\beta^2}{\alpha})$. Applying Theorem~\ref{thm:main-III}, we get the  bound $\cO \left(\frac{\beta^2}{\alpha}\frac{L}{\mu}\log \frac{1}{\varepsilon} \right) $. This is the same result as that obtained by Theorem~\ref{thm:main-I}. On the other hand, if $ \cC\in \mathbb{B}^3(\delta)$, then by Theorem~\ref{thm:compression_properties}, $ \cC\in \mathbb{B}^1(\frac{1}{4\delta^2},2)$. Applying Theorem~\ref{thm:main-I}, we get the  bound $\cO \left(16 \delta^2 \frac{L}{\mu}\log \frac{1}{\varepsilon} \right) $.  This is a worse result than what Theorem~\ref{thm:main-III} offers by a factor of $16 \delta$.

\new{Similarly, if $\cC\in \mathbb{B}^1(\alpha,\beta)$, then by Theorem~\ref{thm:compression_properties}, $\cC\in \mathbb{B}^2(\frac{\alpha}{\beta}, \beta)$. Applying Theorem~\ref{thm:main-II}, we get the  bound $\cO \left(\frac{\beta^2}{\alpha}\frac{L}{\mu}\log \frac{1}{\varepsilon} \right) $. This is the same result as that obtained by Theorem~\ref{thm:main-I}. On the other hand, if $ \cC\in \mathbb{B}^2(\gamma,\beta)$, then by Theorem~\ref{thm:compression_properties}, $ \cC\in \mathbb{B}^1(\gamma^2,\beta)$. Applying Theorem~\ref{thm:main-I}, we get the  bound $\cO \left(\frac{\beta^2}{\gamma^2} \frac{L}{\mu}\log \frac{1}{\varepsilon} \right) $.  This is a worse result than what Theorem~\ref{thm:main-II} offers by a factor of $\frac{\beta}{\gamma}\ge1$.}

Hence, while $\mathbb{B}^1, \mathbb{B}^2$ and $\mathbb{B}^3$ describe the same classes of compressors, for the purposes of CGD it is better to parameterize them as members of $\mathbb{B}^2$ or $\mathbb{B}^3$.

\section{Superiority of Biased Compressors Under Statistical Assumptions} \label{sec:stat}

Here we highlight some advantages of biased compressors by comparing them with their unbiased cousins. We evaluate compressors by their average capacity of preserving the gradient information or, in other words, by expected approximation error they produce. In the sequel, we assume that gradients have i.i.d.\ coordinates drawn from some distribution.

\subsection{Top-$k$ vs Rand-$k$}
We now compare two  sparsification operators: Rand-$k$  \eqref{ex:ur-sparse} which is unbiased and which we denote as $\cC_{rnd}^k$, and Top-$k$  \eqref{ex:top-sparse} which is biased and which we denote as $\cC_{top}^k$. We define variance of the approximation error of $x$ via
$$\omega_{rnd}^k(x) \eqdef \Exp{\norm{\frac{k}{d}\cC_{rnd}^k(x)-x}^2} = \left(1-\frac{k}{d}\right)\twonorm{x}^2 $$ and
$$\omega_{top}^k(x) \eqdef   \norm{\cC_{top}^k(x)-x}^2 = \sum_{i=1}^{d-k} x_{(i)}^2$$ and the energy ``saving''  via
$$s_{rnd}^k(x) \eqdef \twonorm{x}^2 - \omega_{rnd}^k(x) = \Exp{\norm{\frac{k}{d}\cC_{rnd}^k(x)}^2} = \frac{k}{d}\twonorm{x}^2$$ and
$$s_{top}^k(x) \eqdef \twonorm{x}^2 - \omega_{top}^k(x) = \norm{\cC_{top}^k(x)}^2 = \sum_{i=d-k+1}^d x_{(i)}^2.$$ 

Expectations in these expressions are taken with respect to the randomization of the compression operator rather than input vector $x$.
Clearly, there exists $x$ for which these two operators incur identical variance, e.g. $x_1=\dots=x_d$.  However,  in practice we apply compression to gradients $x$ which evolve in time, and which may have heterogeneous components. In such situations, $\omega_{top}^k(x)$ could be much smaller than $\omega_{rnd}^k(x)$. This motivates a {\em quantitative study} of the {\em average  case} behavior in which we make an {\em  assumption} on the distribution of the coordinates of the compressed vector.

\paragraph{Uniform and exponential distribution.}
We first consider the case of uniform and exponentially distributed entries, and quantify the difference.
\begin{lemma}\label{lem:stat-top-random}
Assume the coordinates of $x\in \R^d$ are i.i.d.\ \\ (a) If they follow uniform distribution over $[0,1]$, then
\begin{equation*}
\frac{ \Exp{\omega^k_{top}} }{ \Exp{\omega^k_{rnd}} } = \(1-\frac{k}{d+1}\)\(1-\frac{k}{d+2}\),\qquad \frac{ \Exp{s^1_{top}} }{ \Exp{s^1_{rnd}} } = \frac{3d}{d+2}.
\end{equation*}
(b) If they follow standard exponential distribution, then
$$
 \frac{ \Exp{s^1_{top}} }{ \Exp{s^1_{rnd}} } = \frac{1}{2}\sum\limits_{i=1}^d \frac{1}{i^2} + \frac{1}{2}\left( \sum\limits_{i=1}^d \frac{1}{i}\right)^2 = \cO(\log^2 d).
$$
\end{lemma}

\renewcommand{\arraystretch}{1.8} 
\renewcommand{\tabcolsep}{0.1cm}   
\begin{table}[th]
\centering
\begin{tabular}{|c|c|c|c|c|c|c|c|c|}
\hline
\multicolumn{1}{@{\setlength{\arrayrulewidth}{1pt}\vline}c@{\setlength{\arrayrulewidth}{0.5pt}\vline}}{} & \multicolumn{4}{@{\setlength{\arrayrulewidth}{0.5pt}\vline}c@{\setlength{\arrayrulewidth}{0.5pt}\vline}}{\small Top-3} & \multicolumn{4}{@{\setlength{\arrayrulewidth}{0.5pt}\vline}c@{\setlength{\arrayrulewidth}{1pt}\vline}}{\small Top-5} \\ \hline
\small$d$ & $10^2$  & $10^3$  & $10^4$  & $10^5$ & $10^2$  & $10^3$  & $10^4$ & $10^5$  \\ \hline
\small$\mathcal{N}(0;1)$ & \multicolumn{4}{|c|}{$3 \cdot (\sigma^2 + \mu^2) = 3$}  & \multicolumn{4}{|c|}{$5 \cdot (\sigma^2 + \mu^2) = 5$}\\ \hline
$\Exp{s_{top}^k(x)}$ & 18.65  & 31.10  & 43.98  & 57.08  &  27.14 & 47.70  & 69.07  & 90.85\\ \hline
\small$\mathcal{N}(2;1)$ & \multicolumn{4}{|c|}{$3 \cdot (\sigma^2 + \mu^2) = 15$}  & \multicolumn{4}{|c|}{$5 \cdot (\sigma^2 + \mu^2) = 25$}\\ \hline
$\Exp{s_{top}^k(x)}$   & 53.45  &  75.27 &  95.81 & 115.53  & 81.60  & 118.56  & 153.13  & 186.22  \\ \hline
\end{tabular}
\caption{Information savings of greedy and random sparsifiers for $k=3$ and $k=5$.}
\label{normtable_main}
\end{table}

\paragraph{Empirical comparison.} Now we compare these two sparsification methods on an empirical bases and show the significant advantage of greedy sparsifier against random sparsifier. We assume that coordinates of to-be-compressed vector are i.i.d.\ Gaussian random variables.

\begin{figure}[t]
\vskip 0.2in
\begin{center}
\centerline{\includegraphics[width=0.5\columnwidth]{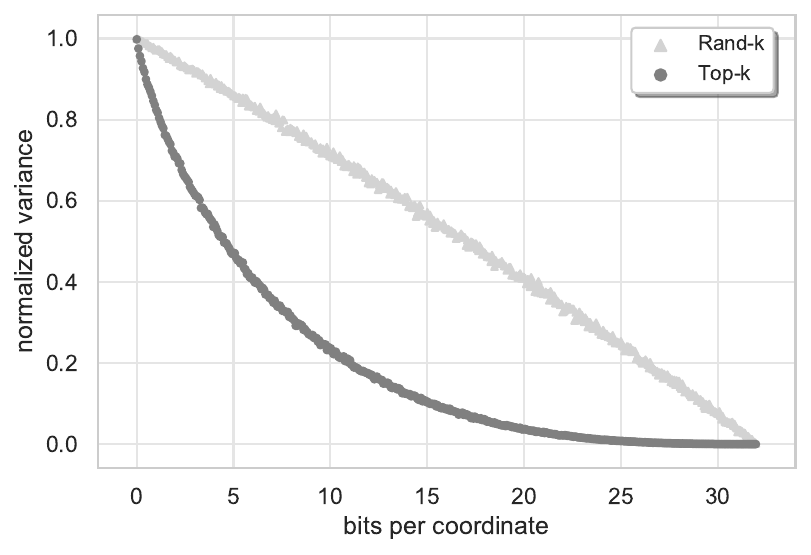}}
\caption{The comparison of Top-$k$ and Rand-$k$ sparsifiers with respect to normalized variance and the number of encoding bits used for each coordinate on average. Each point/marker represents a single $d=10^4$ dimensional vector drawn form Gaussian distribution and then compressed by the specified operator. Each curve was obtained by varying the free parameter $k\in\{1,2,\dots,d\}$.  Plots for different $d$ look very similar. Notice that, for random sparsification the normalized variance is perfectly linear with respect to the number of bit per coordinate. Letting $b$ be the total number of bits to encode the compressed vector (say in {\em binary32} system), the normalized variance produced by random sparsifier is almost $1-\tfrac{\nicefrac{b}{d}}{32}$. However, greedy sparsifier achieves exponentially lower variance $\approx 0.86^{\nicefrac{b}{d}}$ utilizing the same amount of bits.}
\label{fig:vb_top-random}
\end{center}
\vskip -0.2in
\end{figure}

First, we compare the savings $s_{top}^k$ and $s_{rnd}^k$ of these compressions. For random sparsification, we have $$\Exp{s_{rnd}^k(x) }= k \cdot (\sigma^2 + \mu^2),$$ where $\mu$ and $\sigma^2$ are the mean and variance of the Gaussian distribution. For computing $\Exp{s_{top}^k(x)}$, we use the probability density function of $k$-th order statistics (see \eqref{eq:arnold222} or (2.2.2) of \citep{arnold}). Table~\ref{normtable_main} shows that Top-$3$ and Top-$5$ sparsifiers ``save'' $3\times$--$40\times$ more information in expectation and the factor grows with the dimension.

Next we compare normalized variances $\frac{\omega_{top}^k(x)}{\twonorm{x}^2}$ and $\frac{\omega_{rnd}^k(x)}{\twonorm{x}^2}$ for randomly generated Gaussian vectors. In an attempt to give a dimension independent comparison, we compare them against the average number of encoding bits per coordinate, which is quite stable with respect to the dimension. Figure \ref{fig:vb_top-random} reveals the superiority of greedy sparsifier against the random one.

\paragraph{Practical distribution.} We obtained various gradient distributions via logistic regression (\textit{mushrooms} LIBSVM dataset) and least squares. We used the sklearn package and built Gaussian smoothing of the practical gradient density. The second moments, i.e.  energy ``saving'', were already calculated from it by formula for density function of $k$-order statistics, see Appendix \ref{sec:order_stat} or \citep{arnold}. We conclude experiments for Top-5 and Rand-5, see Figure \ref{exper2_main} for details.

\begin{figure}[t]
\centering
\subcaptionbox{\\$s_{rnd}^5 = 9 \cdot 10^5$,\\ $s_{top}^5 = 28 \cdot 10^5$}
  {\includegraphics[width=0.22\linewidth]{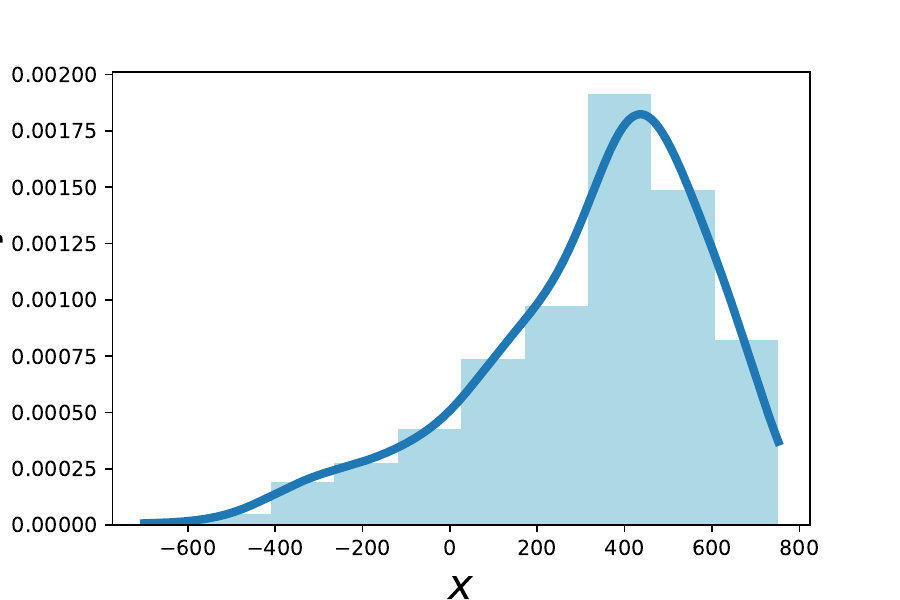}}
\subcaptionbox{\\$s_{rnd}^5 = 624$,\\ $s_{top}^5 = 1357$}
  {\includegraphics[width=0.22\linewidth]{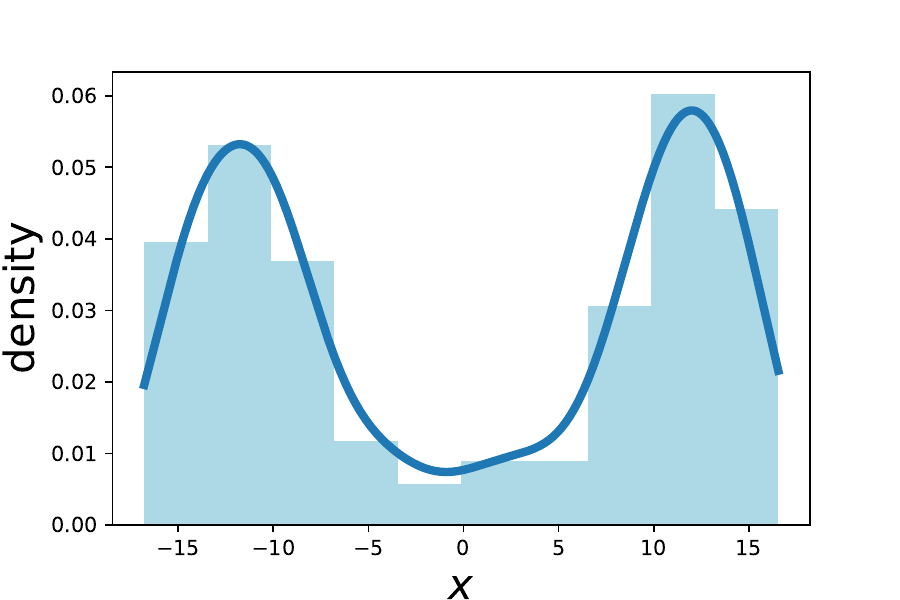}}
  \subcaptionbox{\\$s_{rnd}^5 = 11921$,\\ $s_{top}^5 = 23488$}
  {\includegraphics[width=0.22\linewidth]{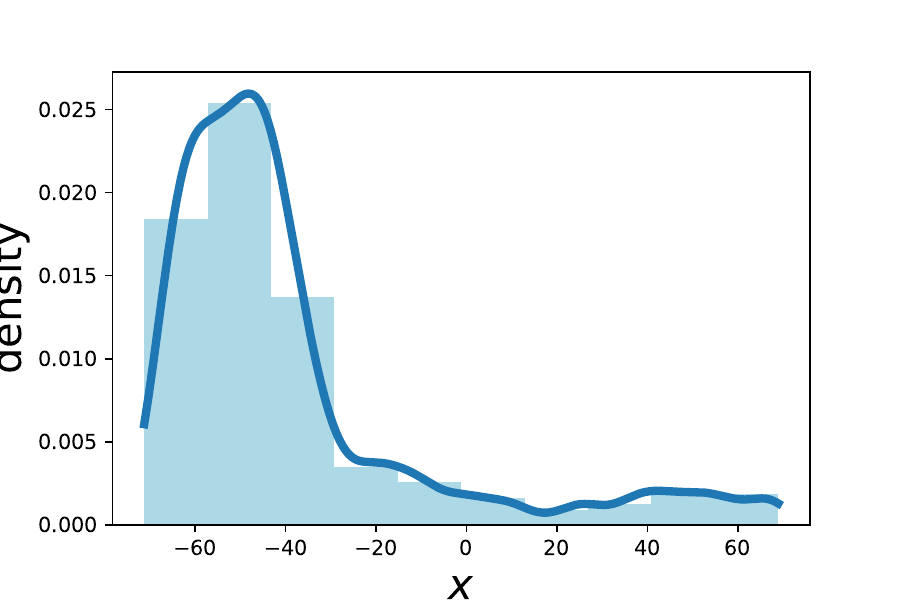}}
  \subcaptionbox{\\$s_{rnd}^5 = 2 \cdot 10^{-4}$,\\ $s_{top}^5 = 17 \cdot 10^{-4}$}
  {\includegraphics[width=0.22\linewidth]{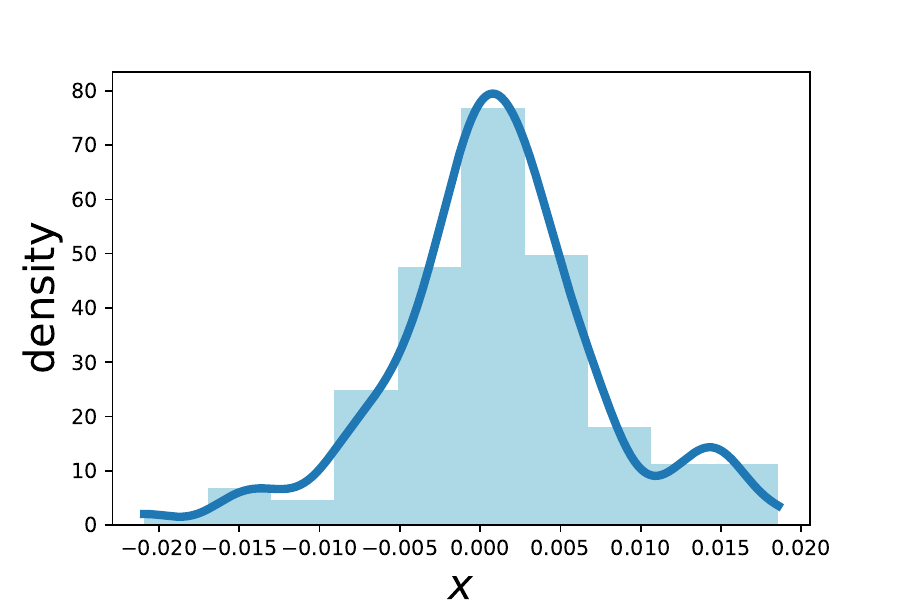}}
\caption{Calculations of the Rand-5 and Top-5 energy ``saving'' for practical gradient distributions ((a),(b),(c):  quadratic problem, (d):  logistic regression). The results of Top-5 are 3--5$\times$ better.}
\label{exper2_main}
\end{figure}

\subsection{New compressor: Top-$k$ combined with dithering}
In Section~\ref{sec:examples} we gave a new biased compression operator (see \eqref{ex:top-gendith}), where we combined Top-$k$ sparsification operator (see \eqref{ex:top-sparse}) with the general exponential dithering (see \eqref{ex:ge-dithering}). Consider the composition operator with  natural dithering, i.e., with base $b=2$. We showed that it belongs to $\B^1(\frac{k}{d},\frac{9}{8})$, $\B^2(\frac{k}{d},\frac{9}{8})$ and $\B^3(\frac{9d}{8k})$. Figure~\ref{fig:vb_comparison} empirically confirms that it attains the lowest compression parameter $\delta\ge1$  among all other known compressors (see \eqref{def:comp_III}). Furthermore, the iteration complexity $\cO\left(\delta\tfrac{L}{\mu}\log\tfrac{1}{\varepsilon}\right)$ of CGD for  $\cC\in \B^3(\delta)$ implies that it enjoys fastest convergence.

\begin{figure}[ht]
\centering
\includegraphics[width=0.5\columnwidth]{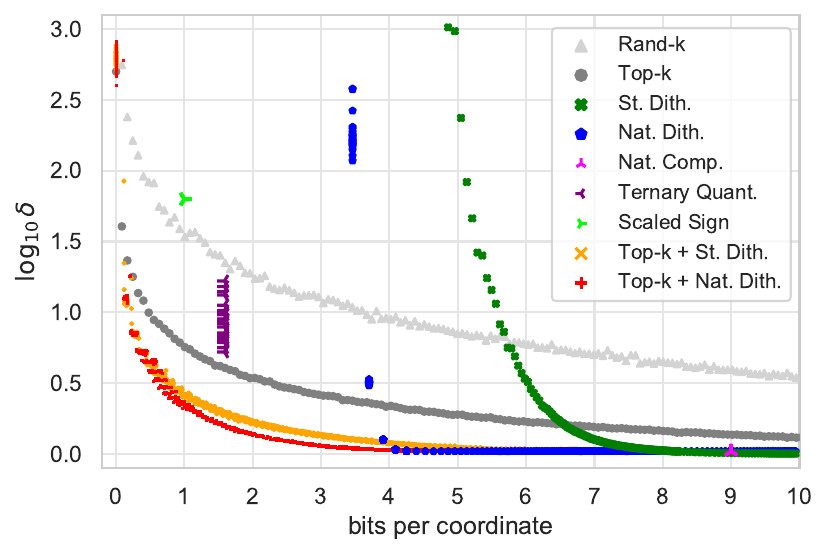}
\caption{\new{Comparison of various compressors (with and without free design parameters) with respect to the parameter $\delta\ge1$ in $\log_{10}-$scale and the number of encoding bits used for each coordinate on average. Each point/marker represents a single $d=10^4$ dimensional vector  $x$ drawn from Gaussian distribution and then compressed by the specified operator. Each curve was obtained by varying free parameters ($k\in\{1,2,\dots,d\}$ for sparsifiers, $s\ge1$ for ditherings) of the specified compressor. Parameter-free compressors, such as ternary quantization and natural compression, have fixed communication budgets which explains the vertical arrangements of the points.}}
\label{fig:vb_comparison}
\end{figure}

\section{Distributed Setting} \label{sec:distributed}

We now focus attention on a distributed setup with $n$ machines, each of which owns {\em non-iid} data defining one loss function $f_i$. Our goal is to minimize the average loss:
\begin{equation}
    \label{distr probl_main}
 \min \limits_{x \in \R^d} \left\{f(x) \eqdef \frac{1}{n} \sum\limits_{i=1}^n f_i(x)\right\}.
\end{equation}

\subsection{Distributed CGD with unbiased compressors}
Perhaps the most straightforward extension of CGD to the distributed setting is to consider the method
\begin{equation}\label{eq:DCGD} \tag{DCGD} x^{k+1} = x^k - \stepsize \frac{1}{n}\sum \limits_{i=1}^n \cC^k_i(\nabla f_i(x^k)).\end{equation}
Indeed, for $n=1$ this method reduces to CGD. For unbiased compressors belonging to $\mathbb{U}(\zeta)$, this method converges under suitable assumptions on the functions. For instance, if $f_i$ are $L$-smooth and $f$ is $\mu$-strongly convex, then as long as the stepsize is chosen appropriately, the method converges  to a $\cO\left(\frac{\eta D(\zeta-1)}{\mu n}\right)$ neighborhood of the (necessarily unique) solution $x^\star$ with the linear rate $$\cO\left(\left(\frac{L}{\mu}+\frac{L(\zeta-1) }{ \mu n}\right) \log \frac{1}{\epsilon}\right),$$ where $D\eqdef \frac{1}{n}\sum \limits_{i=1}^n \norm{\nabla f_i(x^\star)}^2$ \citep{sigma_k}. In particular, in the overparameterized setting when $D=0$, the method converges to the exact solution, and does so at the same rate as GD as long as $\zeta = \cO(n)$. These results hold even if a regularizer is considered, and a proximal step is added to DCGD. Moreover, as shown by \citet{DIANA} and \citet{DIANA-VR}, a variance reduction technique can be devised to remove the neighborhood convergence and replace it by convergence to $x^\star$, at the negligible additional cost of $\cO((\zeta-1)\log \frac{1}{\epsilon})$.

\subsection{Failure of DCGD with biased compressors}  However, as we now demonstrate by giving some counter-examples,  DCGD may fail if the compression operators are allowed to be biased. In the first example below, DCGD used with the Top-1 compressor diverges at an exponential rate.

\begin{example}\label{ex:counter-example}
Consider $n=d=3$ and define
$$f_1(x) = \lin{a,x}^2 + \frac{1}{4}\twonorm{x}^2, \qquad f_2(x) = \lin{b,x}^2 + \frac{1}{4}\twonorm{x}^2, \qquad f_3(x) = \lin{c,x}^2 + \frac{1}{4}\twonorm{x}^2,$$ where $a=(-3,2,2)$, $b=(2,-3,2)$ and $c=(2,2,-3)$.
Let the starting iterate be $x^0=(t,t,t)$, where $t > 0$. Then
$$
\nabla f_1(x^0) =\frac{t}{2} (-11, 9, 9), \qquad \nabla f_2(x^0) =\frac{t}{2} (9,-11, 9), \qquad \nabla f_3(x^0) = \frac{t}{2} (9,9,-11).$$ Using the Top-1 compressor, we get $\cC(\nabla f_1(x^0)) = \frac{t}{2}(-11, 0, 0)$, $\cC(\nabla f_2(x^0)) = \frac{t}{2}(0,-11, 0)$ and
$\cC(\nabla f_3(x^0)) = \frac{t}{2}(0,0,-11)$. The next iterate of DCGD is
$$ x^1 = x^0 - \stepsize \frac{1}{3} \sum_{i=1}^3 \cC(\nabla f_i(x^0)) =  \left(1+\frac{11 \eta}{6}\right)x^0. $$
Repeated application gives $$x^k = \left(1+\frac{11 \eta}{6}\right)^k x^0.$$  Since $\eta>0$, the entries of $x^k$ diverge exponentially fast to $+\infty$.
\end{example}

The above counter-example can be extended to the case of Top-$d_1$ when $d_1 < \ceil*{\frac{d}{2}}$.

\begin{example}\label{ex:counter-example-ext}
Fix the dimension $d\ge 3$ and let $n=\binom{d}{d_1}$ be the number of nodes, where $d_1 < \ceil*{\frac{d}{2}}$ and $d_2 = d-d_1 > d_1$. Choose positive numbers $b,\;c > 0$ such that
$$
-b d_1 + c d_2 = 1,\; b > c+1.
$$
One possible choice could be $b = d_2+\frac{d_2}{d_1},\; c = d_1+\frac{1}{d_2}+1$. Define vectors $a_j\in\R^d,\;j\in[n]$ via
$$
a_j = \sum_{i\in I_j} (-b) e_i + \sum_{i\in [d]\setminus I_j} c e_i,
$$
where sets $I_j\subset[d],\;j\in[n]$ are all possible $d_1$-subsets of $[d]$ enumerated in some way. Define
$$
f_j(x) = \lin{a_j,x}^2 + \frac{1}{2}\twonorm{x}^2, \quad j\in[n]
$$
and let the initial point be $x^0 = t e,\; t>0$, where $e=\sum_{i=1}^d e_i$ is the vector of all $1$s. Then
$$
\nabla f_j(x^0) = 2\lin{a_j, x^0} \cdot a_j + x^0 = 2t(-b d_1 + c d_2)\cdot a_j + te = t(2 a_j + e).
$$
Since $|2(-b)+1| > |2c+1|$, then using the Top-$d_1$ compressor, we get 
$$
\cC(\nabla f_j(x^0)) = -t(2b-1) \sum_{i\in I_j} e_i.
$$
Therefore, the next iterate of DCGD is
\begin{eqnarray*}
x^1
&=& x^0 - \stepsize \frac{1}{n} \sum_{j=1}^n \cC(\nabla f_j(x^0))
= x^0 + \frac{\stepsize t(2b-1)}{n} \sum_{j=1}^n\sum_{i\in I_j}e_i \\
&=& x^0 + \frac{\stepsize(2b-1)}{n} \binom{d}{d_1-1} x^0
= \(1+ \frac{\stepsize(2b-1)d_1}{d_2+1}\) x^0,
\end{eqnarray*}
which implies
$$x^k = \(1+ \frac{\stepsize(2b-1)d_1}{d_2+1}\)^k x^0.$$
Since $\eta>0$ and $b>1$, the entries of $x^k$ diverge exponentially fast to $+\infty$.
\end{example}

Finally, we present more general counter-example with different type of failure for DCGD when non-randomized compressors are used.

\begin{example}
Let $\cC\colon\R^d\to\R^d$ be {\em any deterministic} mapping for which there exist vectors $v_1,v_2,\dots,v_m\in\R^d$ such that
\begin{equation}\label{eq:vanisher}
\sum_{i=1}^m v_i \ne 0, \qquad\text{and}\qquad \sum_{i=1}^m \cC(v_i) = 0.
\end{equation}
Consider the distributed optimization problem \eqref{distr probl_main} with $n=m$ devices and with the following strongly convex loss functions
$$
f_i(x) = \<v_i, x\> + \frac{1}{2}\|x\|^2, \quad i\in[m].
$$
Then $\nabla f_i(x) = v_i + x$ and $\nabla f(x) = \frac{1}{n}\sum_{i=1}^n v_i + x$. Hence, the optimal point $x^* = - \frac{1}{n}\sum_{i=1}^n v_i \neq 0$. However, it can be easily checked that, with initialization $x_0=0$, we have
$$
x^1 = x^0 - \gamma\frac{1}{n}\sum_{i=1}^n \cC(\nabla f_i(x^0)) = x^0 - \gamma\frac{1}{n}\sum_{i=1}^n \cC(v_i + x^0) = x^0.
$$
Thus, when initialized at $x_0=0$, not only DCGD does not converge to the solution $x^*\neq 0$, it remains stuck at the same initial point for all iterates, namely $x^k = x^0 = 0$ for all $k\ge1$.

Condition \eqref{eq:vanisher} can be easily satisfied for specific biased compressors. For instance, Top-$1$ satisfies \eqref{eq:vanisher} with $v_1 = \left[\begin{smallmatrix} 1 \\ 4 \end{smallmatrix}\right]$, $v_2 = \left[\begin{smallmatrix} -1 \\ -2 \end{smallmatrix}\right]$, $v_3 = \left[\begin{smallmatrix} 1 \\ -2 \end{smallmatrix}\right]$.
\end{example}
 
The above examples suggests that one needs to devise a different approach to solving the distributed problem \eqref{distr probl_main} with biased compressors. We resolve this problem by employing a memory feedback mechanism.

\subsection{Error Feedback} We show that distributed version of Distributed SGD wtih Error-Feedback~\citep{karimireddy2019error}, displayed in Algorithm~\ref{alg}, is able to resolve the issue. Moreover, this algorithm allows for the computation of stochastic gradients. Each step starts with all machines $i$ in parallel computing a stochastic gradient $g_i^k$ of the form
\begin{equation}\label{stgr1_main}
    g_i^k = \nabla  f_i(x^k) + \xi_i^k,
\end{equation}
where $\nabla  f_i(x^k)$ is the true gradient, and $\xi_i^k$ is a stochastic error. Then, this is multiplied by a stepsize $\eta^k$ and added to the memory/error-feedback term $e_i^k$, and subsequently compressed. The compressed messages are communicated and aggregated. The difference of message we wanted to send and its compressed version becomes stored as $e_i^{k+1}$ for further correction in the next communication round. The output $\overline x^K$  is an ergodic average of the form
\begin{equation}
    \label{answer}
  \overline x^K \eqdef \frac{1}{W^K}\sum_{k=0}^K w^k x^k, \qquad W^K \eqdef \sum_{k=0}^K w^k.
\end{equation}

\begin{algorithm} [th]
    \caption{Distributed SGD with Biased Compression and Error Feedback}
    \label{alg}
    \begin{algorithmic}
\STATE 
\noindent {\bf Parameters:}  Compressors  $\cC_i^k \in \mathbb{B}^3(\delta)$; Stepsizes $\{\eta^k\}_{k\geq 0}$; Iteration count $K$  \\
\noindent {\bf Initialization:} Choose  $x^0\in \R^d$ and $e_i^0 = 0$ for all $i$
\FOR {$k=0,1, 2, \ldots, K$ } 
\item Server sends  $x^k$ to all $n$ machines
\item All machines in parallel perform these updates:
\begin{eqnarray}
    \label{comp_grad}
    & \tilde{g}_i^k = \cC^k_i(e_i^k+ \stepsize^k g_i^k) \\
    \label{error}
    & e_i^{k+1} = e_i^k + \stepsize^k g_i^k - \tilde{g}_i^k 
\end{eqnarray}
\item Each machine $i$ sends  $\tilde{g}_i^k$ to the server
\item Server performs aggregation:
\begin{eqnarray}
    \label{step}
    x^{k+1} = x^k -\frac{1}{n} \sum\limits_{i=1}^n \tilde{g}_i^k
\end{eqnarray}
\ENDFOR

\noindent {\bf Output:} Weighted average of the iterates: $\overline x^K$ \eqref{answer}
    \end{algorithmic}
\end{algorithm}

\subsection{Complexity theory} 

We assume the stochastic error $\xi_i^k$ in \eqref{stgr1_main}  satisfies the following condition.
\begin{assumption} Stochastic error  $\xi_i^k$ is unbiased, i.e. $\Exp{\xi_i^k} = 0$, and for some constants $B,C\geq 0$
\begin{equation}
    \label{stgr3_main}
    \Exp{\norm{\xi_i^k}^2} \leq B \norm{\nabla f_i(x^k)}^2 + C, \qquad \forall i\in[n], k\ge0.
\end{equation}
\end{assumption}

Note that this assumption is much weaker than the bounded variance assumption (i.e., $\Exp{\norm{\xi_i^k}^2} \leq  C$) and bounded gradient assumption (i.e., $\Exp{\norm{g_i^k}^2} \leq  C$).
We can now state the main result of this section. To the best of our knowledge, {\em  this was an open problem}: we are not aware of any convergence results for distributed optimization that tolerate general classes of {\em biased} compression operators and have reasonable assumptions on the stochastic gradient.
\begin{theorem}[Convergence guarantees for Algorithm \ref{alg}]
\label{thm:sparsified}
Let $\{x^k\}_{k \geq 0}$ denote the iterates of Algorithm~\ref{alg} for solving problem \eqref{eq:probR}, where each $f_i$ is $L$-smooth and $\mu$-strongly convex. Let $x^\star$ be the minimizer of $f$ and let $f^\star\eqdef f(x^\star)$ and
$$
D\eqdef \frac{1}{n}\sum_{i=1}^n \norm{\nabla f_i(x^\star)}^2.
$$
Assume the compression operator used by all nodes is in $\mathbb{B}^3(\delta)$. Then we have the following convergence rates under three different stepsize and iterate weighting regimes:

\begin{itemize}

\item[(i)] {\bf $\cO(\frac{1}{k})$ stepsizes \& $\cO(k)$ weights.} Let, for all $k\ge0$, the stepsizes and weights be set as $\stepsize^k = \frac{4}{\mu(\kappa + k)}$ and  $w^k = \kappa +k$, respectively, where $\kappa = \frac{56(2\delta + B)L}{\mu}$. Then
\begin{eqnarray*}
   \boxed{\Exp{f(\bar x^K)}-f^\star = 
   \cO \left(\frac{A_1}{K^2}  +  \frac{A_2}{ K} \right)}\,,
\end{eqnarray*}
where $A_1 \eqdef \frac{L^2(2\delta+B)^2}{\mu}\norm{x^0-x^\star}^2 $ and $A_2\eqdef  \frac{C\left(1+\nicefrac{1}{n}\right) + D \left(\nicefrac{2B}{n} +3\delta\right)}{\mu}$.

\item[(ii)] {\bf  $\cO(1)$ stepsizes \& $\cO(e^{-k})$ weights.} Let, for all $k\ge0$, the stepsizes and weights be set as $\stepsize^k = \stepsize$ and $w^k = (1-\mu \stepsize/2)^{-(k+1)}$, respectively, where $\stepsize \leq \frac{1}{14(2\delta + B) L}$. Then
\begin{eqnarray*}
  \boxed{\Exp{f(\bar x^K)}-f^\star = \tilde \cO \left( A_3 \exp \left[- \frac{K}{A_4} \right] + \frac{A_2}{ K}\right)}\,,
\end{eqnarray*}
where $A_3\eqdef L(2\delta + B)\norm{x^0-x^\star}^2$ and  $A_4 \eqdef  \frac{28L(2\delta + B)}{\mu} $.
 
\item[(iii)]  {\bf $\cO(1)$ stepsizes \& equal weights.} Let, for all $k\ge0$, the stepsizes and weights be set as $\stepsize^k = \stepsize$ and $w^k = 1$, respectively,  where $\stepsize \leq \frac{1}{14(2\delta + B) L}$. Then
\begin{eqnarray*}
 \boxed{\Exp{f(\bar x^K)}-f^\star = \cO \left(  \frac{A_3}{ K } + \frac{A_5}{\sqrt{K}}  \right)}\,,
\end{eqnarray*}
where $A_5\eqdef  \sqrt{C\left(1 + \nicefrac{1}{n}\right) + D \left(\nicefrac{2B}{n} +3\delta\right)}\norm{x^0-x^\star} $.

\end{itemize}

\end{theorem}

Let us make a few observations on these results.
First, Algorithm \ref{alg} employing general biased compressors and error feedback mechanism indeed {\em resolves} convergence issues of DCGD method by {\em converging the optimal solution $x^*$}.
Second, note that the choice of stepsizes $\stepsize^k$ and weights $w^k$ leading to convergence is not unique and {\em several schedules are feasible}.
Third, all the rates are sublinear and based on the second rate {\it (ii)} above, {\em linear convergence} is guaranteed if $C=D=0$. Based on \eqref{stgr3_main}, one setup when the condition $C=0$ holds is when all devices compute full local gradients (i.e., $g_i^k = \nabla f_i(x^k)$). Furthermore, the condition $D=0$ is equivalent to $\nabla f_i(x^\star) = 0$ for all $i\in[n]$, which is typically satisfied for over-parameterized models.
Lastly, under these two assumptions (i.e., devices can compute full local gradients and the model is over-parameterized), we show that distributed SGD method with error feedback converges with the same $\cO\left(\delta \frac{L}{\mu} \log \frac{1}{\epsilon}\right)$ linear rate as single node CGD algorithm. {\em To the best of our knowledge, this was the first regime where distributed first order method with biased compression is guaranteed to converge linearly.}

\begin{figure*}[th]
\centering
\includegraphics[width=.35\linewidth]{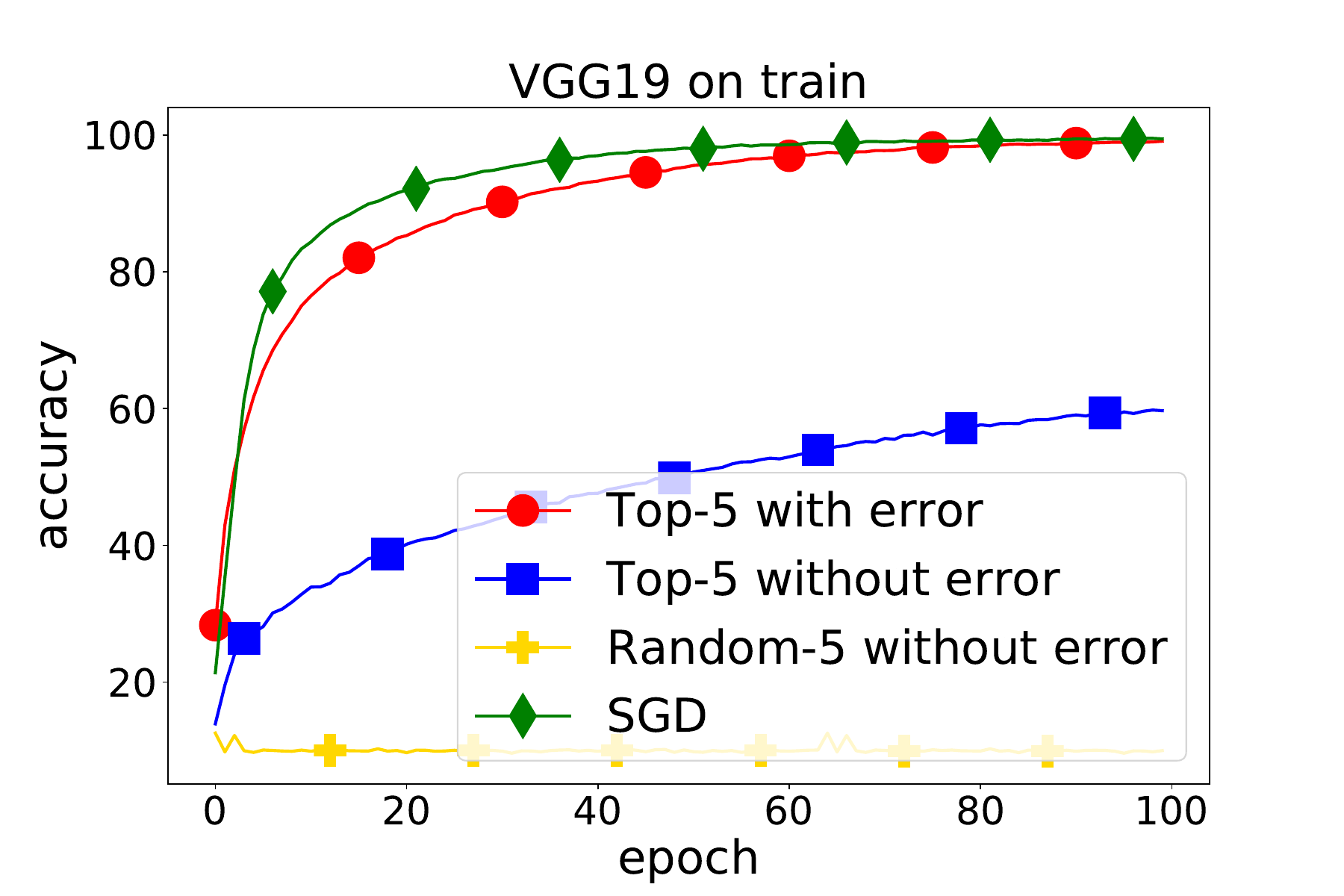}
\includegraphics[width=.35\linewidth]{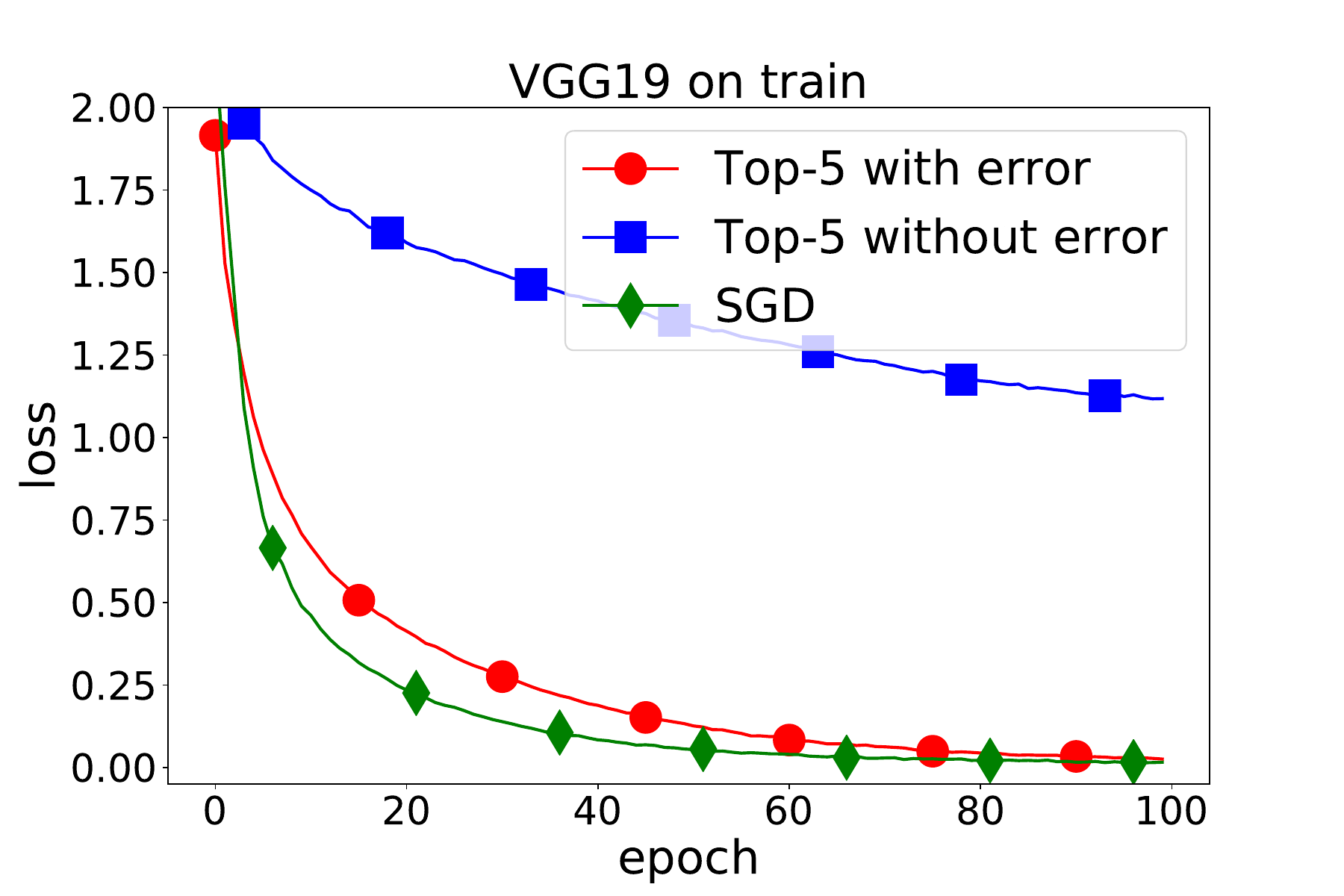}
\includegraphics[width=.35\linewidth]{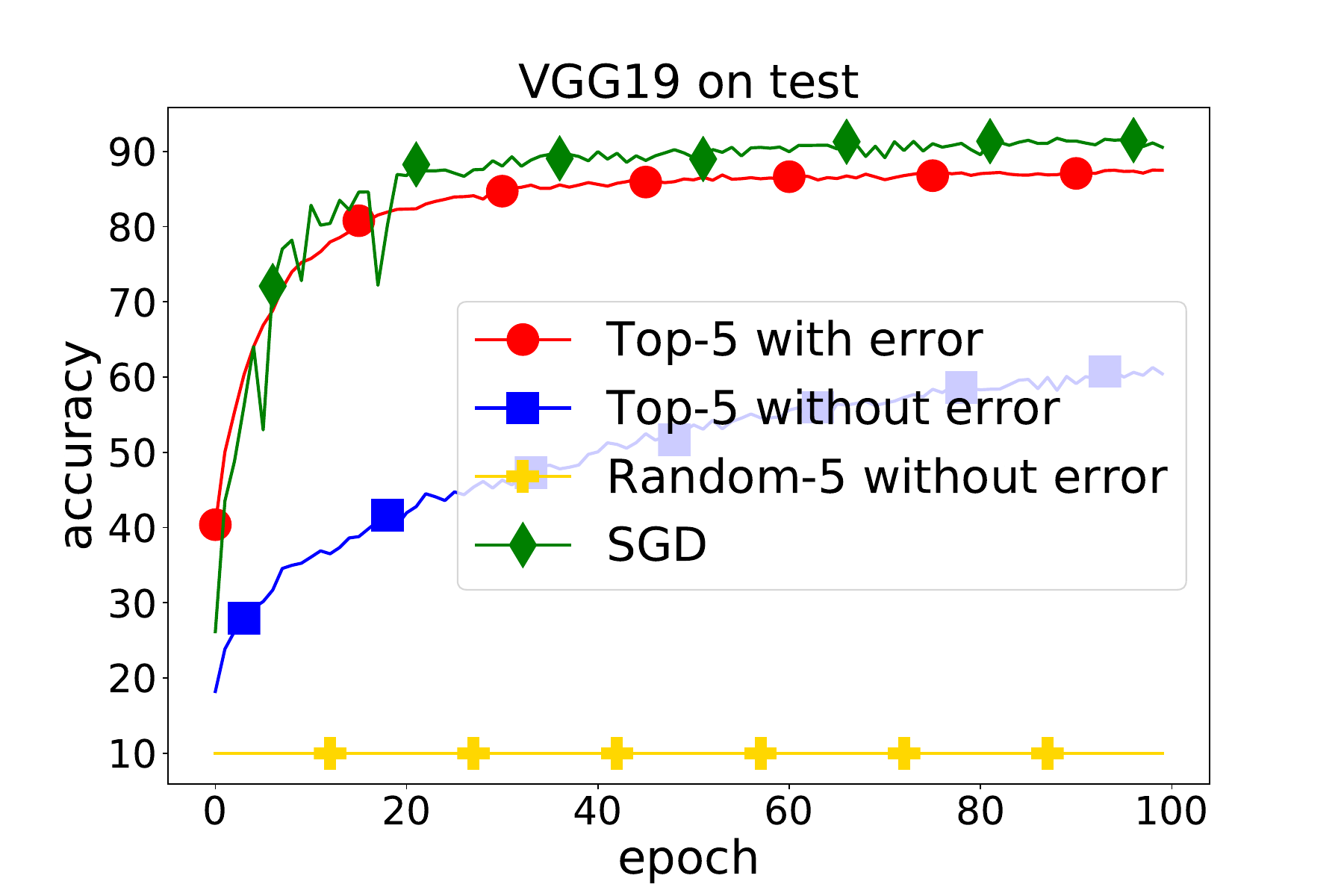}
\includegraphics[width=.35\linewidth]{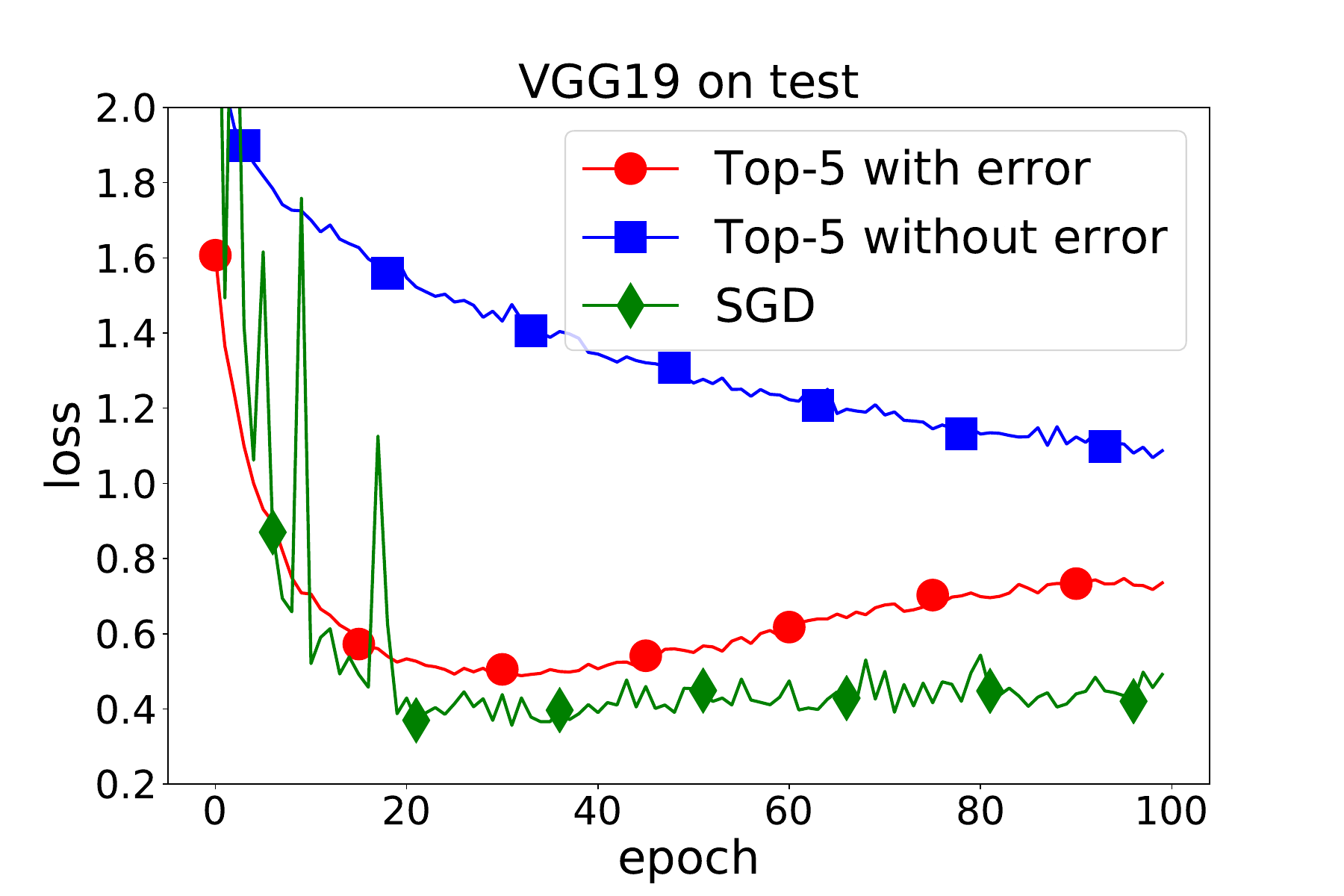}
\caption{Training/Test loss and accuracy  for VGG19 on CIFAR10 distributed among $4$ nodes for $4$ different compression operators.}
\label{fig:exper4_main}
\end{figure*}

\begin{figure*}[th]
\centering
\includegraphics[width=.35\linewidth]{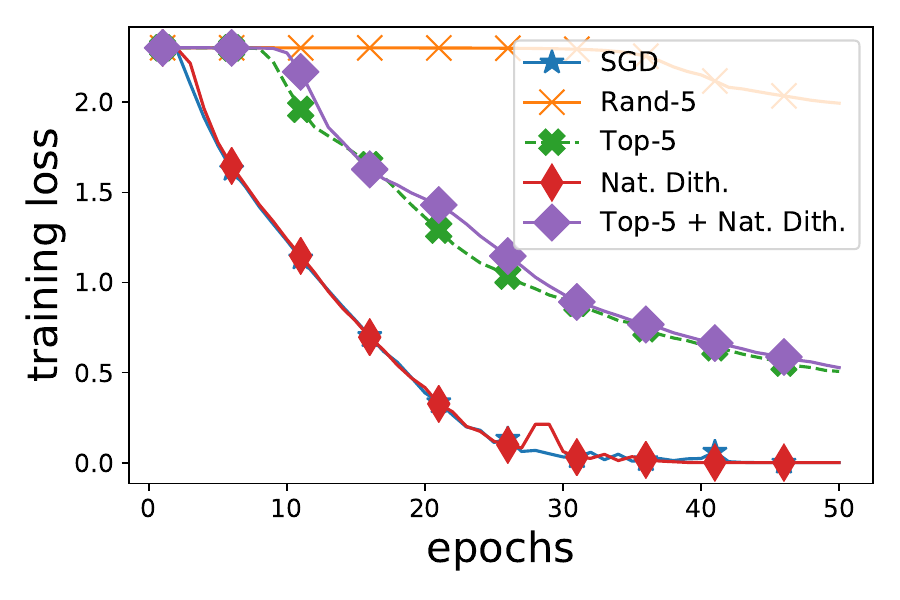}
\includegraphics[width=.35\linewidth]{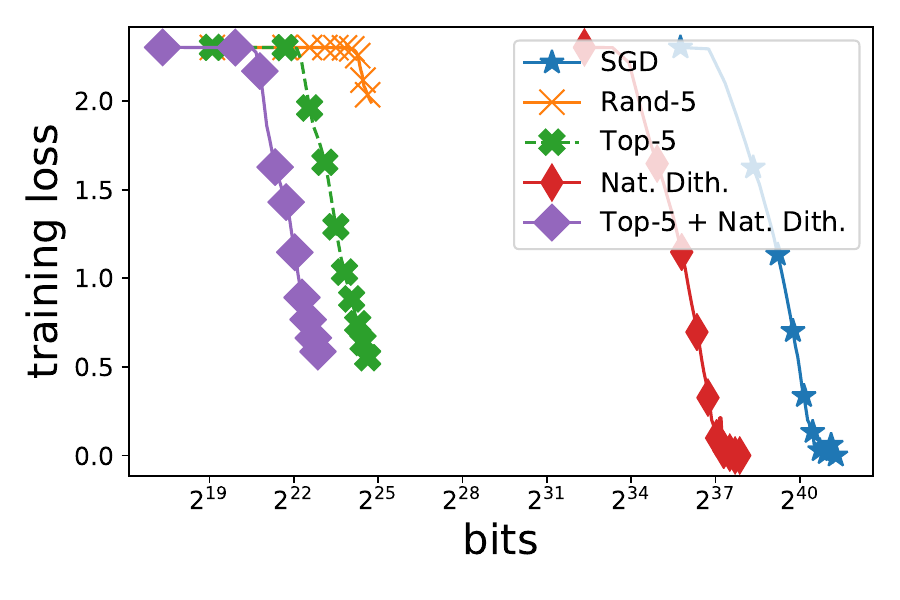}
\includegraphics[width=.35\linewidth]{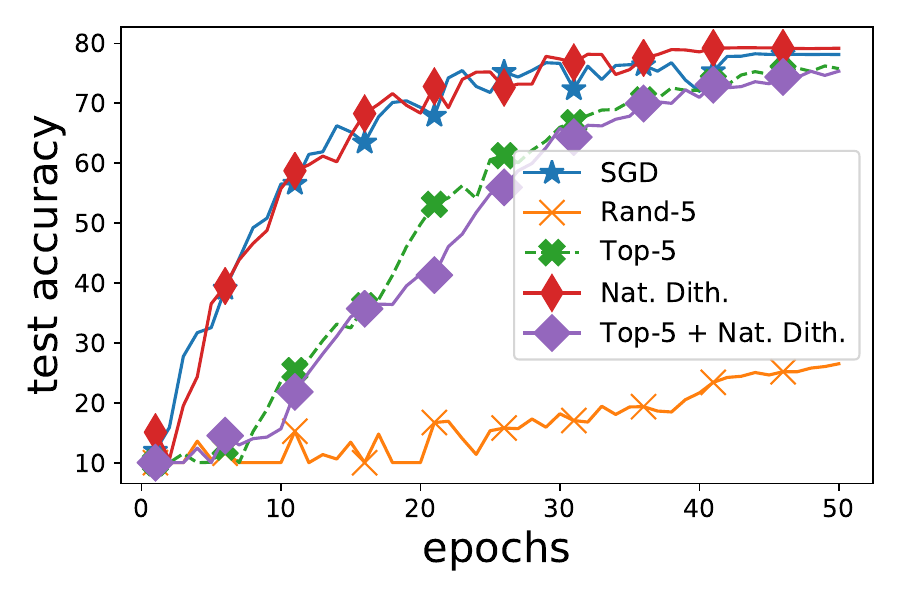}
\includegraphics[width=.35\linewidth]{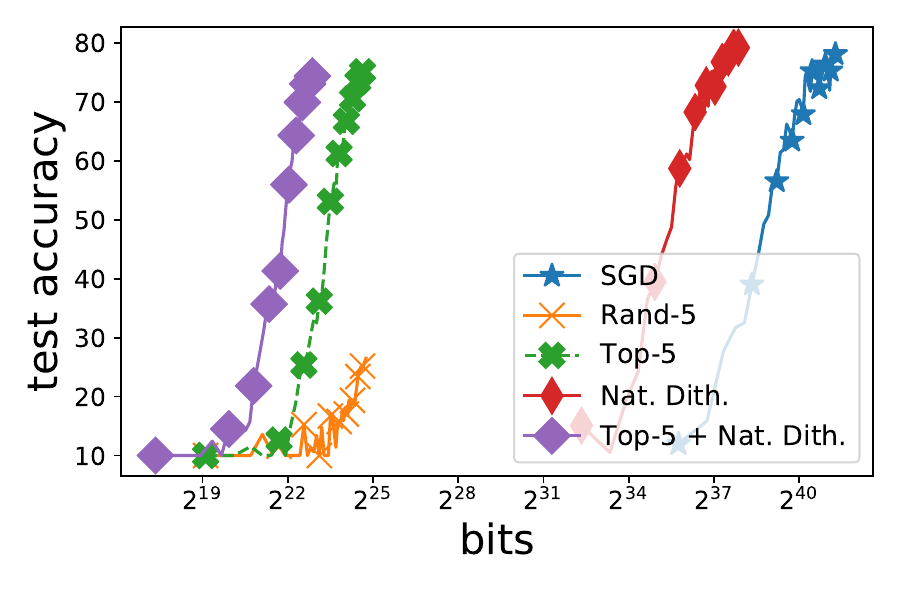}
\caption{Training loss and test accuracy  for VGG11 on CIFAR10 distributed among $4$ nodes for $5$ different compression operators.}
\label{fig:top_k_plus_nat_dit}
\end{figure*}

\section{Experiments}


In Sections \ref{sec:exp1}--\ref{sec:expl}, we present our experiments, which are primarily focused on supporting our theoretical findings. Therefore,  we simulate these experiments on one machine which enable us to do rapid direct comparisons against the prior methods. In more details, we use the machine with 24 Intel(R) Xeon(R) Gold 6146 CPU @ 3.20GHz cores and GPU GeForce GTX 1080 Ti. Section \ref{sec:exp_albert} is devoted to real experiments with a large model and big data. For these experiments, we use a computational cluster with 10 GPUs Tesla T4. We implement all methods in Python 3.7 using Pytorch~\cite{NEURIPS2019_9015}.


\subsection{Lower empirical variance induced by biased compressors during deep network training} \label{sec:exp1}
Motivated by our theoretical results in Section~\ref{sec:stat}, we show that similar behaviour can be seen in the empirical variance of gradients. We run 2 sets of experiments with Resnet18 on CIFAR10 dataset. In Figure~\ref{fig:emp_variance}, we display empirical variance, which is obtained by running a training procedure with specific compression. We compare unbiased and biased compressions with the same communication complexities--deterministic with classic/unbiased $\NC$ and Top-$k$ with Rand-$k$ with  $k$ to be $\nicefrac{1}{5}$ of coordinates. One can clearly see, that there is a gap in empirical variance between biased and unbiased methods, similar to what we have shown in theory, see Section~\ref{sec:stat}. 

\subsection{Error-feedback is needed in distributed training with biased compression}
The next experiment shows the need of error-feedback for methods with biased compression operators. Based on Example 1, error feedback is necessary to prevent divergence from the optimal solution. Figure~\ref{fig:exper4_main} displays training/test loss and accuracy for VGG19 on CIFAR10 with data equally distributed among $4$ nodes. We use plain SGD with a default step size equal to $0.01$ for all methods, i.e. Top-$5$ with and without error feedback, Rand-$5$ and no compression. As suggested by the counterexample, not using error feedback can really hurt the performance when biased compressions are used. Also note, that performance of Rand-$5$ is significantly worse than Top-$5$.

\begin{figure}[t!]
\centering
\includegraphics[width=.3\linewidth]{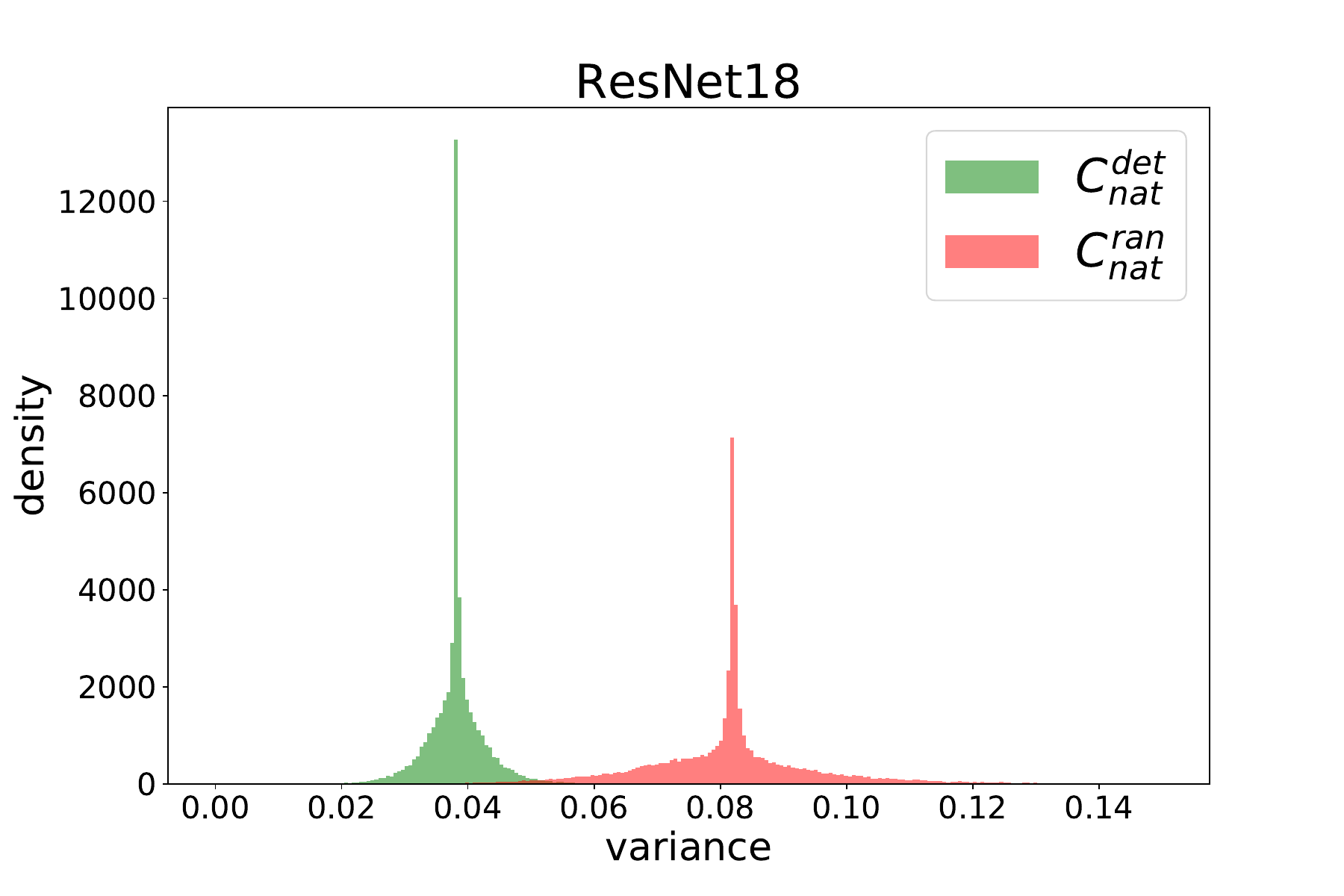}
\includegraphics[width=.3\linewidth]{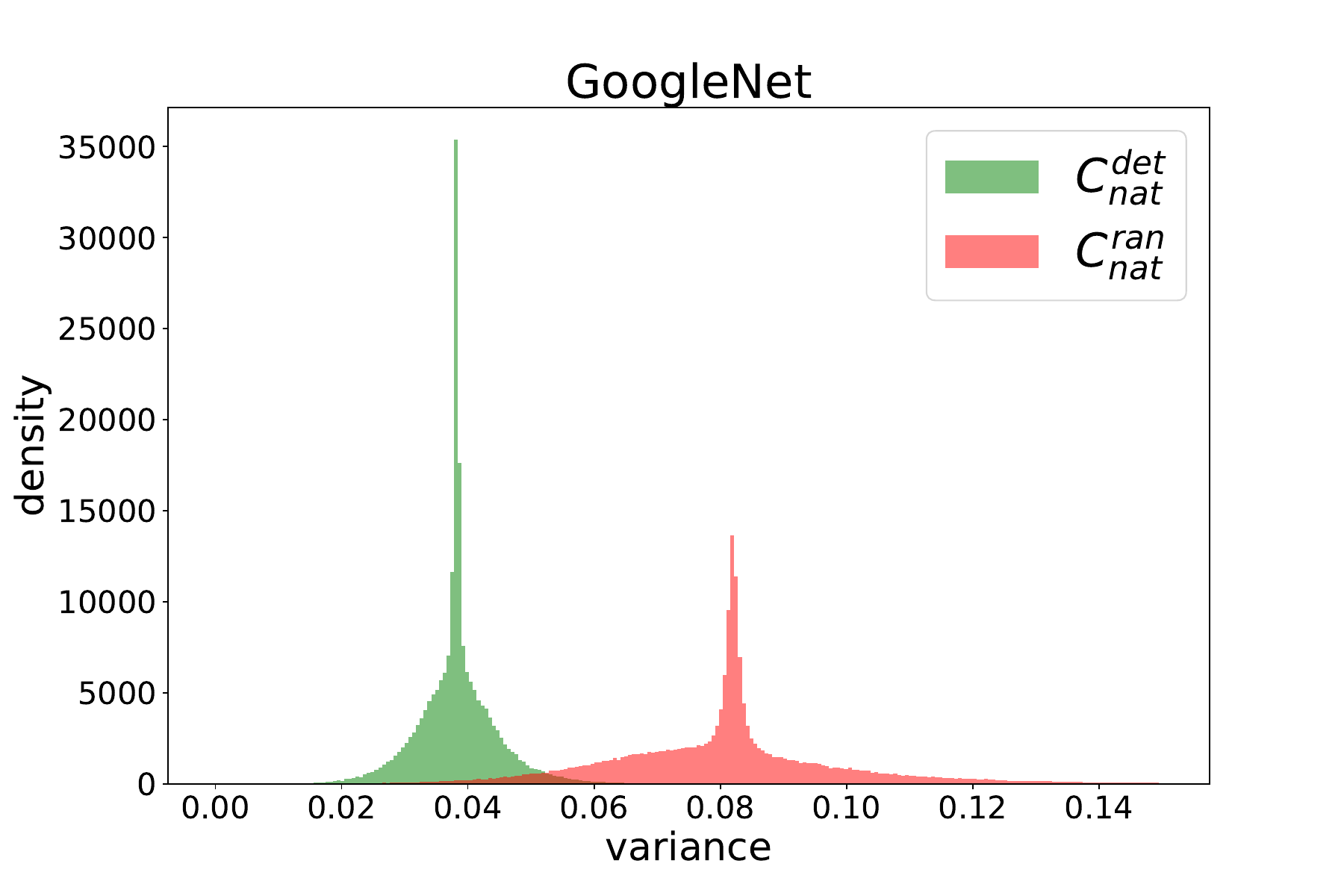}
\includegraphics[width=.3\linewidth]{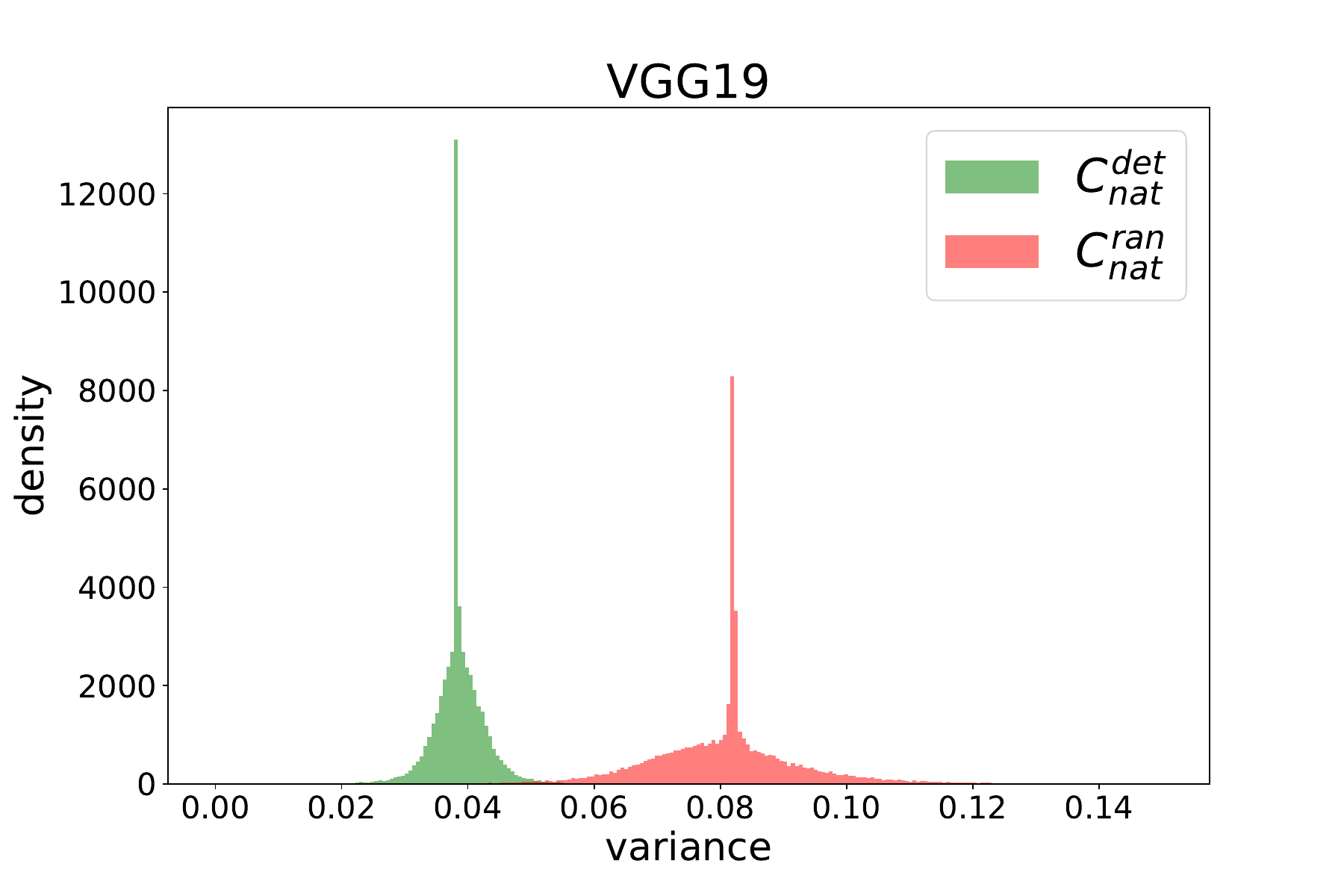}
\\
\includegraphics[width=.3\linewidth]{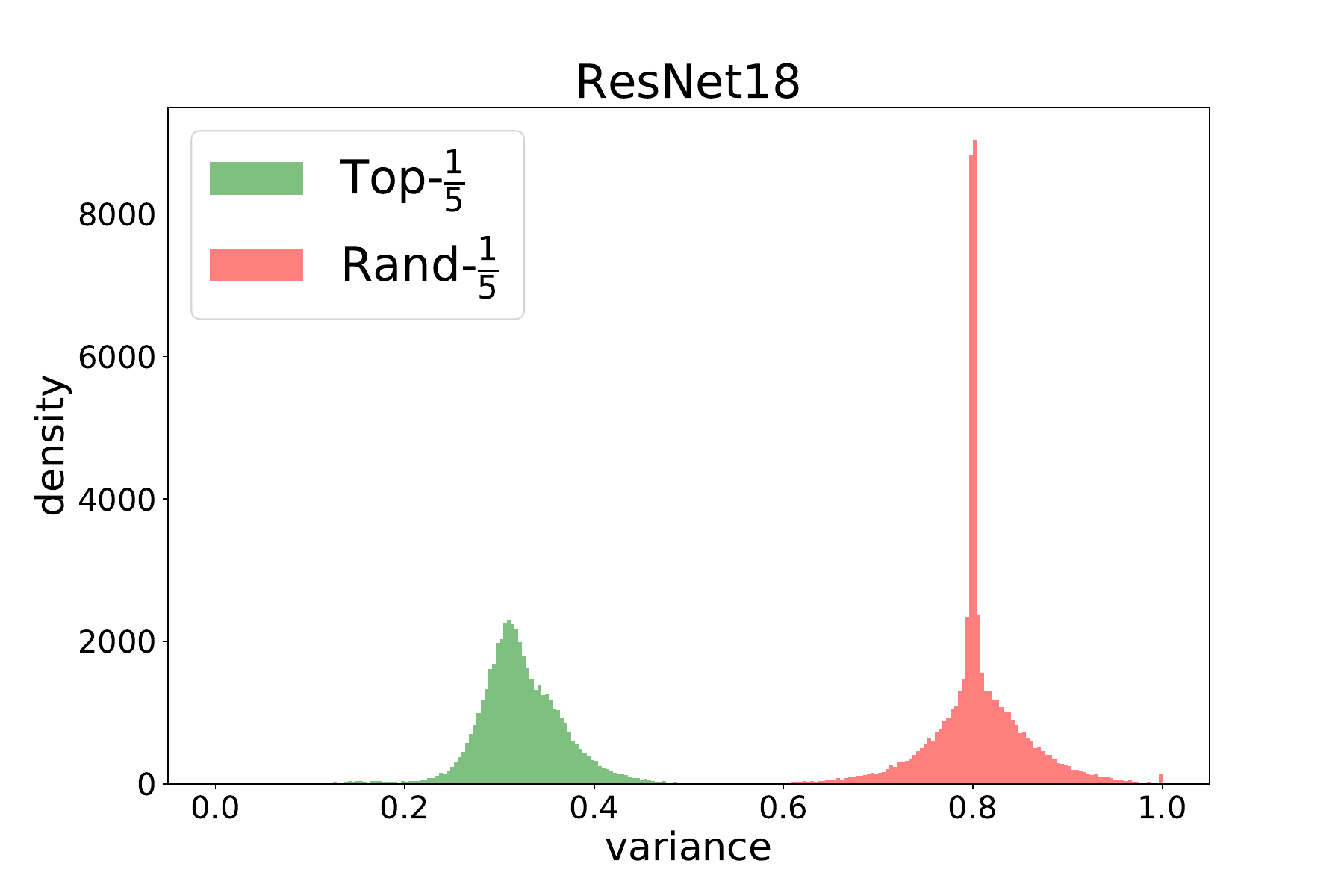}
\includegraphics[width=.3\linewidth]{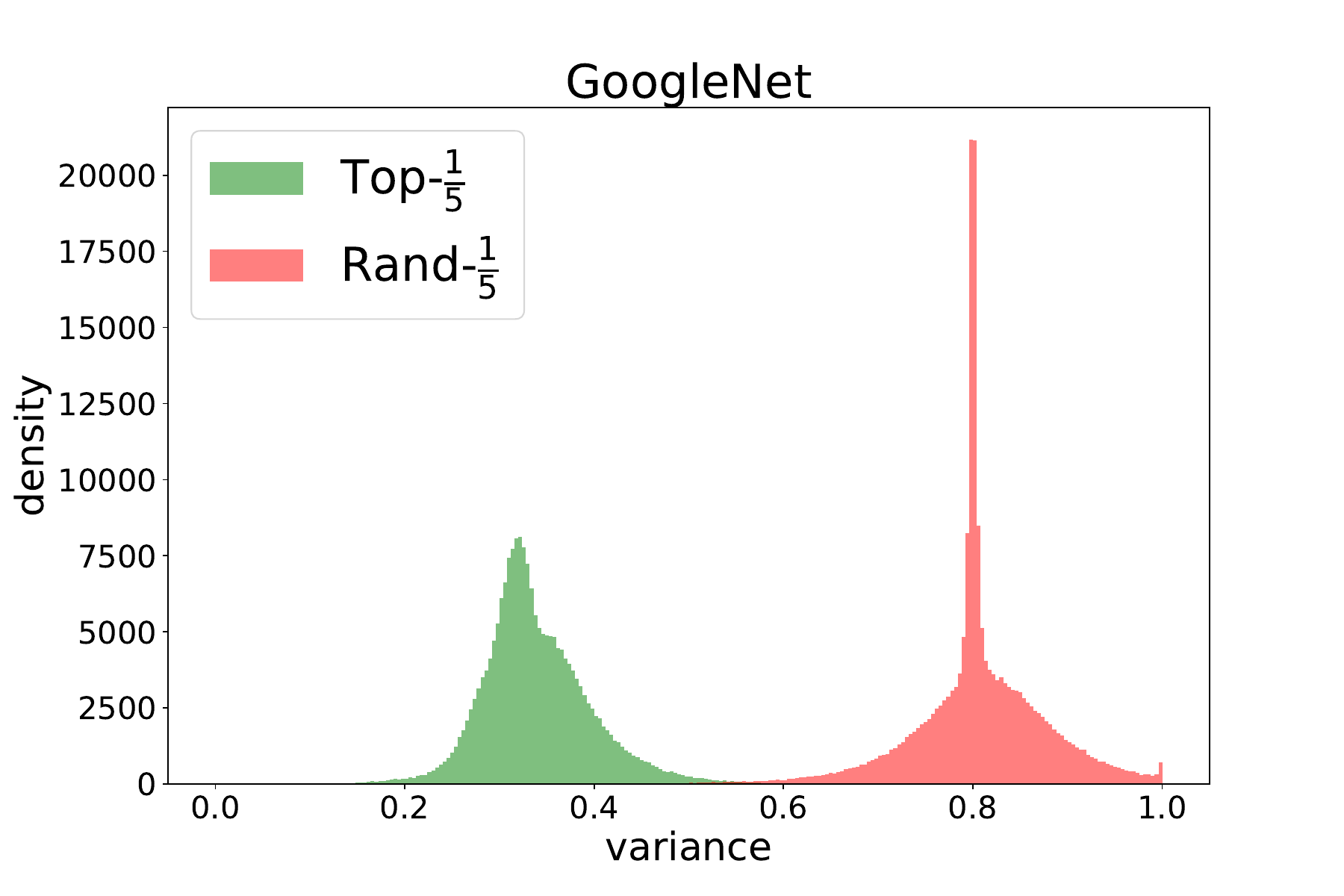}
\includegraphics[width=.3\linewidth]{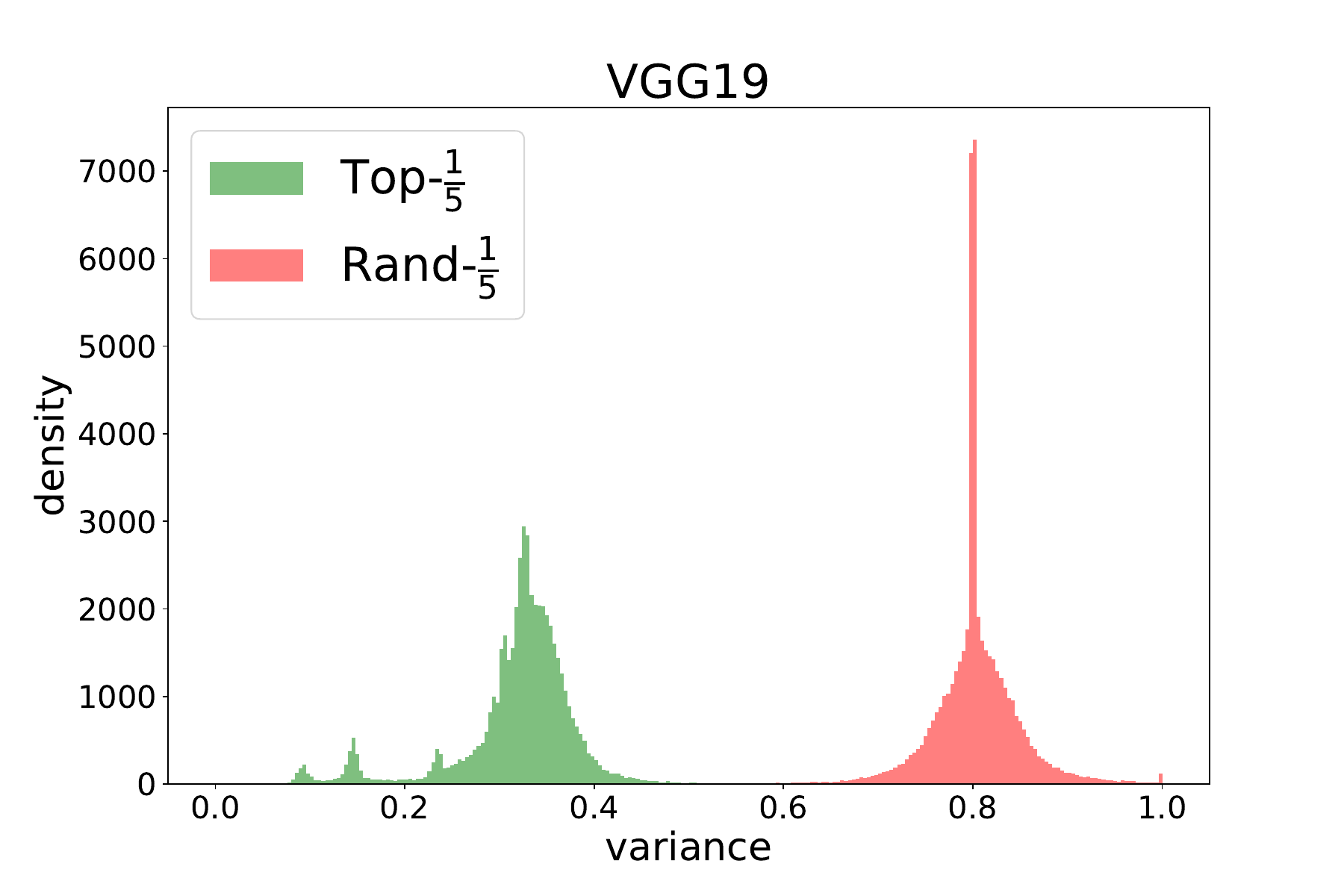}
\caption{Comparison of empirical variance $\norm{\cC(x) - x}^2/\norm{x}^2$ during training procedure for two pairs of method-- deterministic with classic/unbiased $\NC$ and Top-$k$ with Rand-$k$ with Top-$\nicefrac{1}{5}$ of coordinates. Both of the plots were produced using ResNet18, GoogleNet, and VGG19 on CIFAR10 dataset.}
\label{fig:emp_variance}
\end{figure}

\subsection{Top-$k$ mixed with natural dithering saves in communication significantly}
Next, we experimentally show the superiority of our newly proposed compressor--Top-$k$ combined with natural dithering. We compare this against current state-of-the-art for low bandwidth approach Top-$k$ for some small $k$. In Figure~\ref{fig:top_k_plus_nat_dit}, we plot comparison of $5$ methods--Top-$k$, Rand-$k$, natural dithering, Top-$k$ combined with natural dithering and plain SGD. We use $2$ levels with infinity norm for natural dithering and $k=5$ for sparsification methods. For all the compression operators, we train VGG11 on CIFAR10 with plain SGD as an optimizer and default step size equal to $0.01$.  We can see that adding natural dithering after Top-$k$ has the same effect as the natural dithering comparing to no compression, which is a significant reduction in communications without almost no effect on convergence or generalization. Using this intuition, one can come to the conclusion that Top-$k$  with natural dithering is the best compression operator for any bandwidth, where we adjust to given bandwidth by adjusting $k$. This exactly matches with our previous theoretical variance estimates displayed in Figure~\ref{fig:vb_comparison}.

\subsection{Theoretical behavior predicts the actual performance in practice} \label{sec:expl}

In the next experiment, we provide numerical results to further show that our predicted theoretical behavior matches the actual performance observed in practice. We run two regression experiments optimized by gradient descent with step-size $\eta = \frac{1}{L}$. We use a slightly adjusted version of Theorem~\ref{thm:main-III} with adaptive step-sizes, namely
$$
\frac{f(x^k) - f(x^\star)}{f(x^0) - f(x^\star)} \leq  \prod \limits_{i=1}^k \left(1 - \frac{\mu}{L \delta_i}\right),
 $$
where 
$$
1 - \frac{1}{\delta_i} = \frac{\norm{\cC(\nabla f(x^i)) - \nabla f(x^i)}^2}{\norm{\nabla f(x^i)}^2}.
$$
Note that this is the direct consequence of our analysis. We apply this property to display the theoretical convergence. In the first experiment depicted in Figure \ref{fig:exper7}, we randomly generate random square matrix $A$ of dimension $100$ where it is constructed in the following way: we sample random diagonal matrix $D$, which elements are independently sampled from the uniform distribution $(1,10)$,  $(1,100)$, and  $(1,1000)$, respectively. $A$ is then constructed using $Q^\top D Q$, where $P = QR$ is a random matrix and $QR$ is obtained using QR-decomposition. The label $y$ is generated the same way from the uniform distribution $(0,1)$. The optimization objective is then 
 $$
 \min_{x \in \reals^d} x^\top A x - y^\top x.
 $$

For the second experiment shown in Figure \ref{fig:exper8}, we run standard linear regression on two scikit-learn datasets--\textit{Boston} and \textit{Diabetes}--and applied data normalization as the preprocessing step.

\begin{figure}[th]
\centering
\includegraphics[width=.32\linewidth]{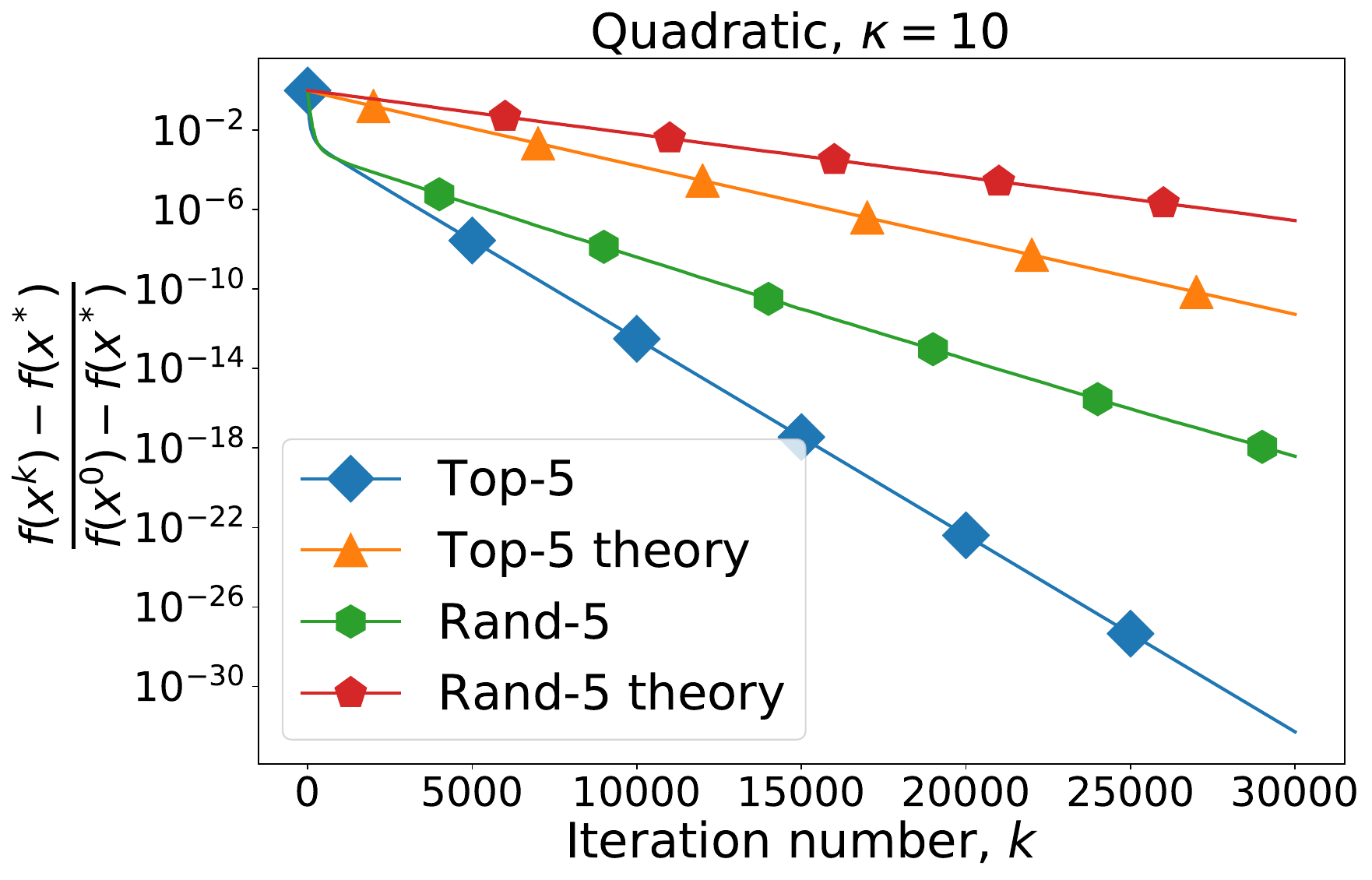}
\includegraphics[width=.32\linewidth]{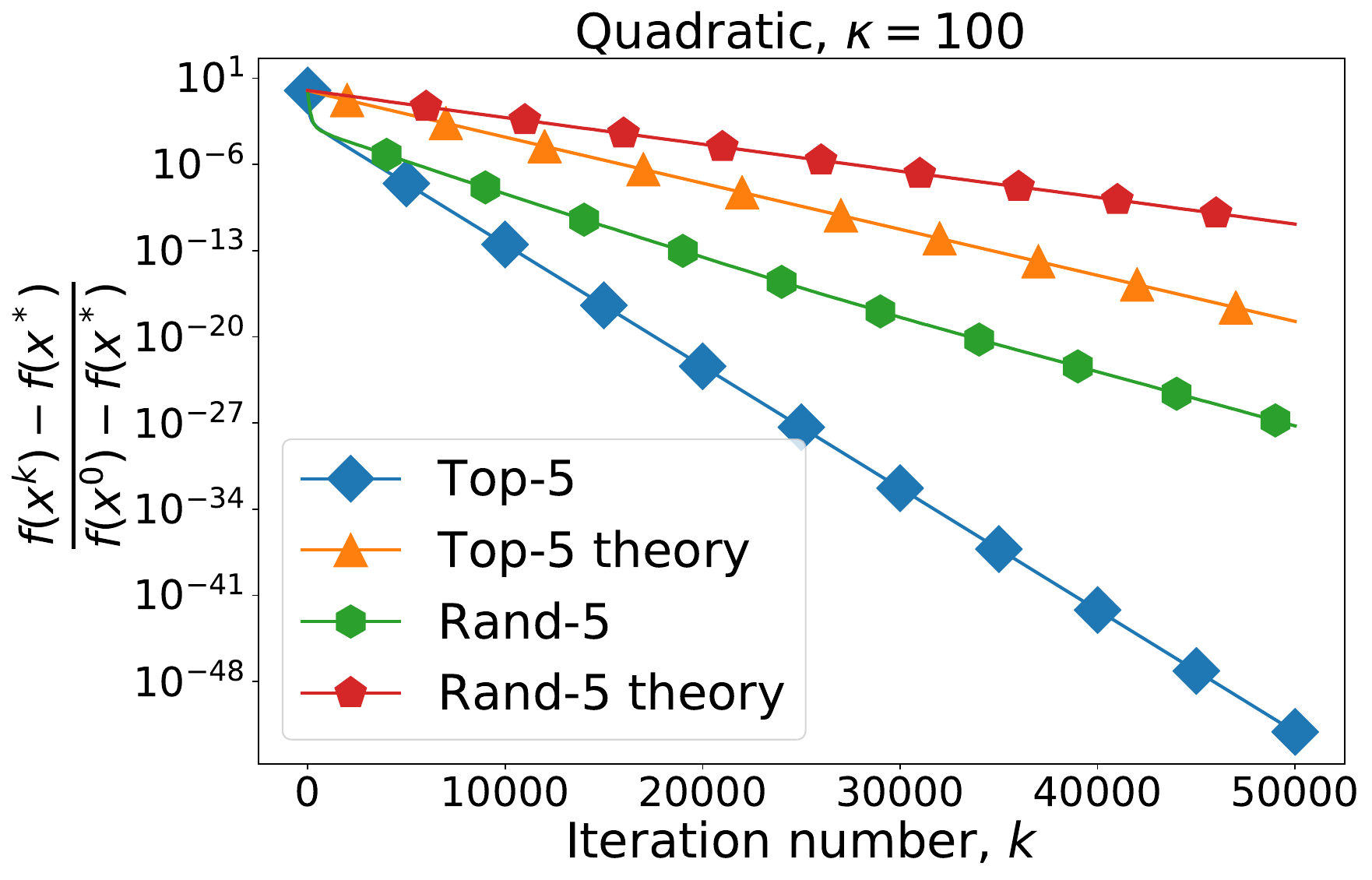}
\includegraphics[width=.32\linewidth]{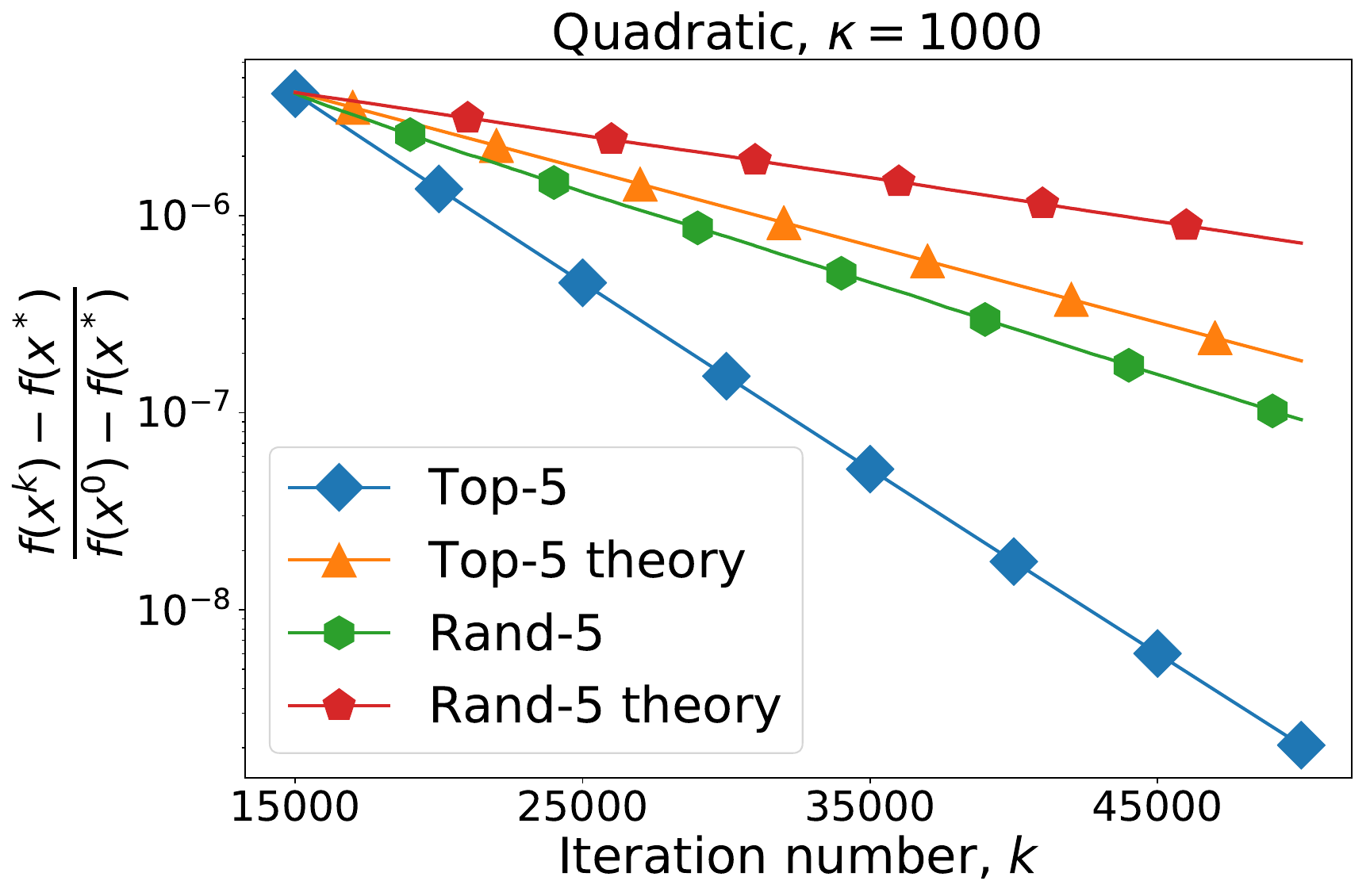}
\caption{Theoretical vs. Practical Convergence of Compressed Gradient Descent on Quadratics problem with different condition number $\kappa$ for Top-5 and Rand-5 compression operators.}
\label{fig:exper7}
\end{figure}

\begin{figure}[th]
\centering
\includegraphics[width=.4\linewidth]{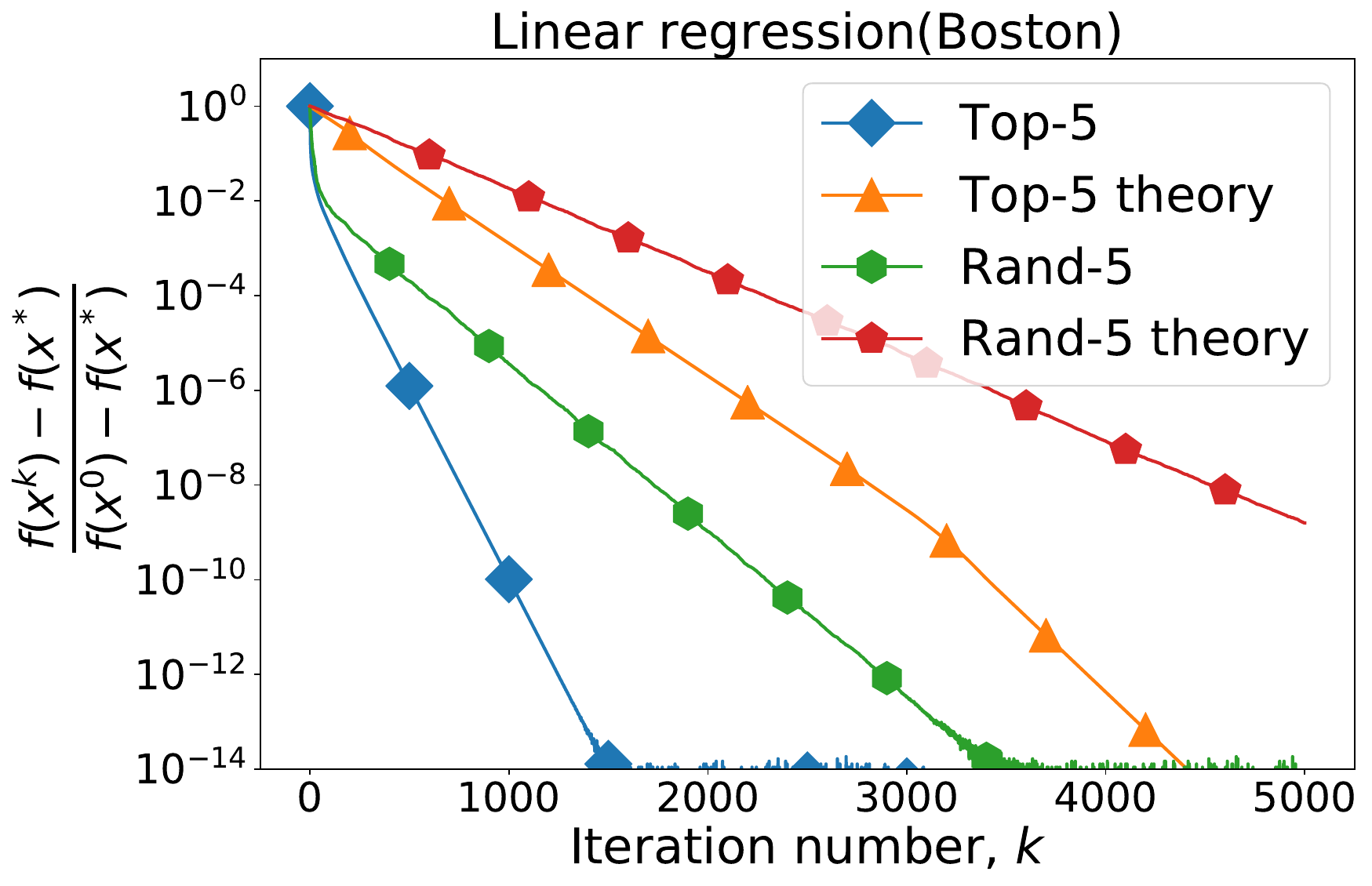}
\includegraphics[width=.4\linewidth]{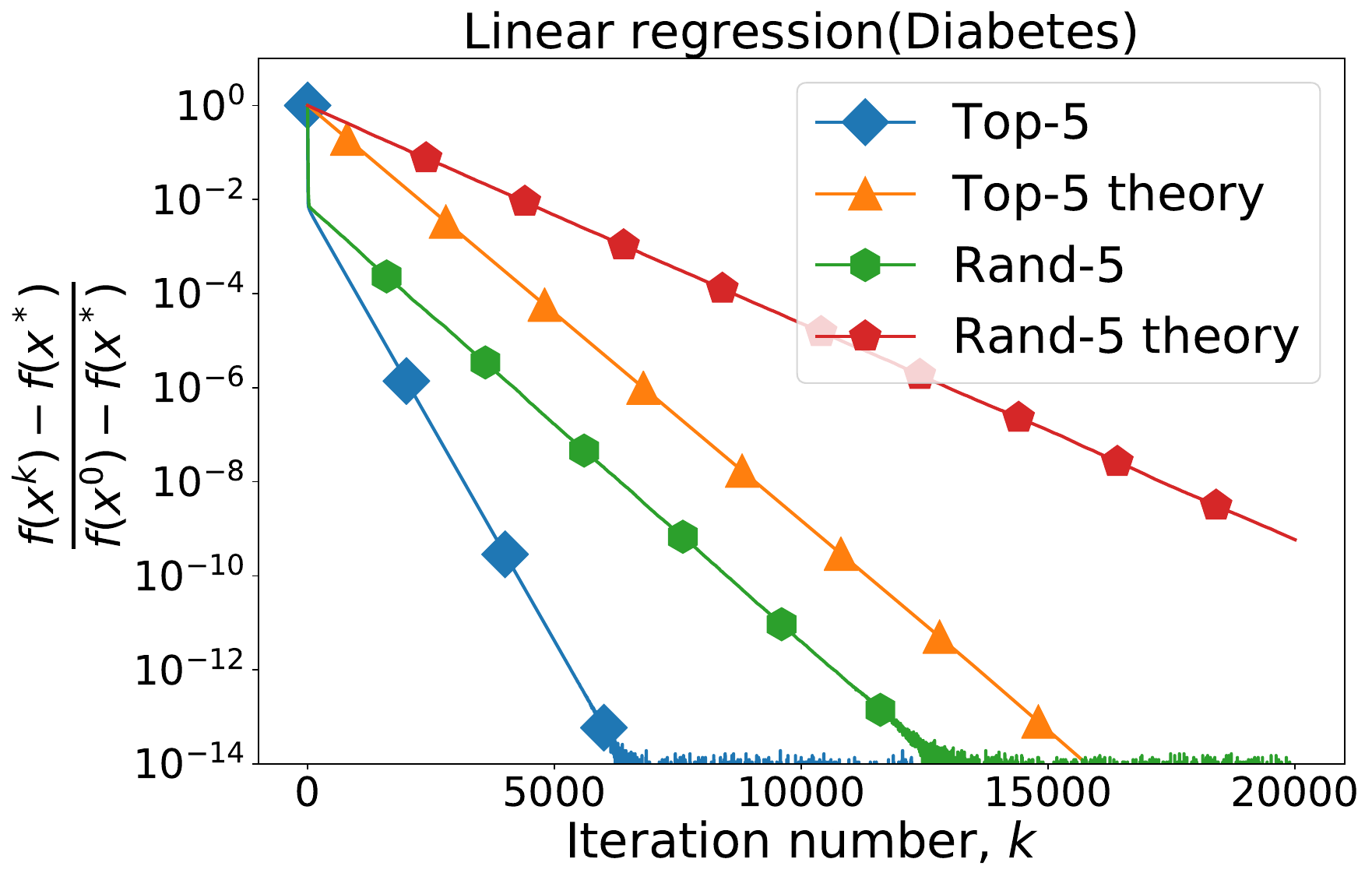}
\caption{Theoretical vs. Practical Convergence of Compressed Gradient Descent on Linear Regression problem for \textit{Boston} and \textit{Diabetes} datasets with Top-5 and Rand-5 compression operators.}
\label{fig:exper8}
\end{figure}

Looking into Figures~\ref{fig:exper7} and \ref{fig:exper8}, one can clearly see that as predicted by our theory, biased compression with less empirical variance leads to better convergence in practice and the gap almost matches the improvement.

\subsection{Transformer training}\label{sec:exp_albert}

\new{In the last experiment, we work with a real big model. In particular, we train ALBERT-large~\citep{albert} (18M parameters) with layer sharing on a combination of Bookcorpus~\citep{zhu2015aligning} and Wikipedia~\citep{devlin2018bert} datasets. We use the same optimizer (LAMB) and the same tuning for it as in the original paper~\citep{albert}. 
Our goal is to find an unbiased and biased operators such that we maximize the improvement in terms of communication cost without losing much in terms of training quality. In this case, we include in the communication cost both the time to perform the communication round and the time to perform the compression and decompression operations. It is important to note that we do not compress packages with gradients corresponding to LayerNorm scales, but this is less than $1$ percent of the whole package. Among unbiased compressors, we try natural compression (Section 2.2 (g)) and random sparsification (Section 2.2 (a)). The best result is shown by natural compression, which compresses the packages by a factor of $4$. The Rand-$25$ operator (which also compresses the information by a factor of $4$) performs much worse even with the use of the error feedback technique. For communications with natural compression, we use the classical allreduce procedure. Among unbiased compressors, we try Top-$k$ sparsification (Section 2.2 (d)) and Power compression~\citep{vogels}. The best result is shown by Power compression with the rank parameter $r{=}8$ and the error feedback. The organization of communications (allreduce procedures) occurs as in the original paper~\citep{vogels}. We measure how the training loss changes (Figure~\ref{fig:albert}) as well as at the end of training we evaluate the final performance for each model on several popular tasks from~\citep{glue} (Table \ref{tab:albert}). The results show that the use of biased compression can significantly reduce the communication time cost compared to uncompressed and even unbiased compression setups. At the same time, the quality of the training does not drop much. }

\begin{figure}[h]
\centering
\includegraphics[width=0.4\textwidth]{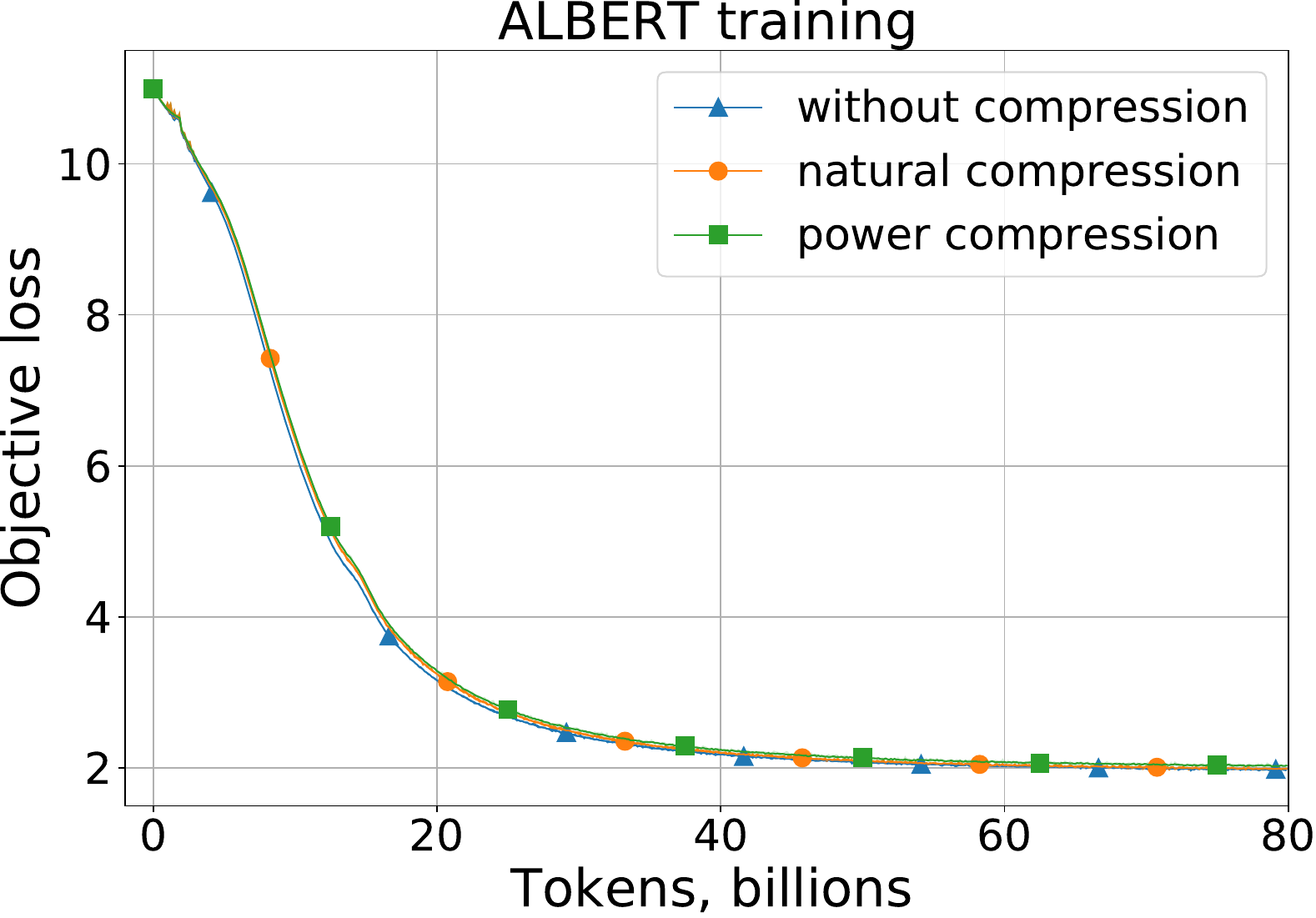}
\includegraphics[width=0.4\textwidth]{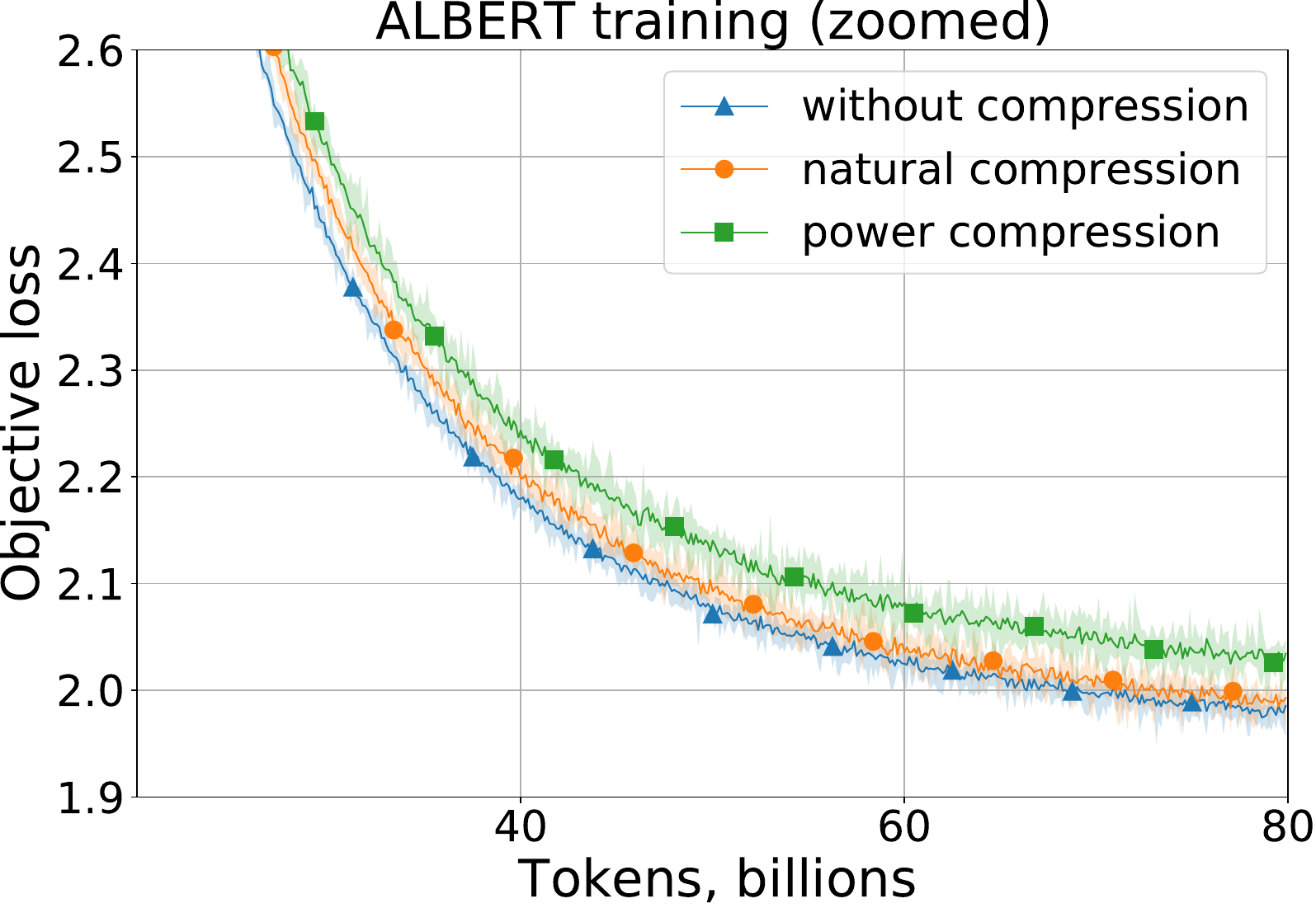}
\caption{Training objective value for ALBERT-large with different compression operators.}
\label{fig:albert}
\end{figure}

\begin{table}
\centering
\footnotesize
\begin{tabular}{|l|c|c|c|c|c|c|c|c|c|}
\hline 
Setup & Avg time & CoLA & MNLI & MRPC & QNLI & QQP & RTE & SST2 & STS-B \\ 
\hline
Without compression & 9.62 $\pm$  0.03 &  46.2 & \textbf{81.1} &  82.5 & {87.9} & \textbf{88.0} & \textbf{66.3} & 85.1 & \textbf{88.0}   \\
Natural compression & 4.05 $\pm$ 0.05 & \textbf{48.3} & 81.0 & \textbf{87.5} & 87.8 & 84.4 & 63.2 & 87.8 & 86.9   \\
Power compression &  \textbf{1.04 $\pm$ 0.04} & 42.4 & 80.3 & 85.1 & \textbf{88.2} & 85.3  & 46.3 & \textbf{88.0} & 87.4  \\
\hline
\end{tabular}
\caption{Comparison of average time per one communication round and downstream evaluation scores on GLUE benchmark tasks.}
\label{tab:albert}
\end{table}


\appendix
\part*{Appendix}

\section{Basic Facts and Inequalities}

\subsection{Strong convexity}

 Function $f$ is  strongly convex on $\R^d$ when it is continuously differentiable and there is a constant $\mu > 0$ such that  the following inequality holds:
\begin{equation}
    \label{strong_conv1}
    \frac{\mu}{2}\norm{x-y}^2 \leq f(x) - f(y) -  \lin{ \nabla f(y), x-y },  \qquad \forall x,y\in \R^d.
\end{equation}

\subsection{Smoothness}

 Function $f$ is called $L$-smooth in $\R^d$ with $L > 0$ when it is differentiable and its gradient is $L$-Lipschitz continuous, i.e.\
\begin{equation*}
    \label{L-smooth1}
    \norm{\nabla f(x) - \nabla f(y) } \leq L \norm{x-y } , \qquad \forall x,y\in \R^d.
\end{equation*}

If convexity is assumed as well, then the following inequalities hold:
\begin{equation}
    \label{L-smooth3}
    \frac{1}{2L}\norm{\nabla f(x) - \nabla f(y)}^2 \leq f(x) - f(y) - \lin{ \nabla f(y), x-y },\qquad \forall x,y\in \R^d
\end{equation}

By plugging $y=x^*$ to \eqref{L-smooth3}, we get
\begin{equation}
    \label{L-smooth4}
    \norm{\nabla f(x)}^2 \leq 2L(f(x) - f(x^*)), \qquad \forall x\in \R^d.
\end{equation}

\subsection{Useful inequalities}

For all $a,b,x_1, \ldots, x_n \in \R^d$ and $\xi > 0$ the following inequalities holds:
\begin{equation}
  \label{inner_prod}
  2\lin{a,b} \leq \frac{\norm{a}^2}{\xi} + \xi
  \norm{b}^2,
\end{equation}
\begin{equation}
  \label{inner_prod_and_sqr}
  \norm{a+b}^2 \leq \left(1 + \frac{1}{\xi}\right)\norm{a}^2 + (1+\xi)
  \norm{b}^2,
\end{equation}
\begin{equation}
  \label{sum_sqr}
  \norm{\sum\limits_{i=1}^n x_i}^2 \leq n \cdot \sum\limits_{i=1}^n \norm{x_i}^2.
\end{equation}

\subsection{\new{Facts from order statistics}}\label{sec:order_stat}

\new{For i.i.d. samples $x_1, x_2, \ldots, x_d$ from an absolutely continuous distribution with probability density function $\phi$ and cumulative distribution function $\Phi$ let $x_{(1)} \leq x_{(2)} \leq \dots \leq x_{(d)}$
be the order statistics obtained by arranging samples in increasing order of magnitude. Then the following expressions give the density function of $x_{(i)}$ ($1\leq i \leq d$)
\begin{equation}
    \label{eq:arnold222}
    \phi_{x_{(i)}} (u) = \frac{d!}{(i-1)! (d-i)!} \{ \Phi(u)\}^{i-1} \{ 1 - \Phi(u) \}^{n-i} \phi(u), \quad -\infty  < u < \infty,
\end{equation}
and the joint density function of all $d$ order statistics
\begin{equation}
    \label{eq:arnold223}
    \phi_{x_{(1)}, \dots, x_{(d)}} (u_1, \dots, u_d) = d! \prod\limits_{i = 1}^d \phi(u_i), \quad -\infty < u_1 \le \dots \le u_d < \infty.
\end{equation}
}

\section{Proofs for Section~\ref{sec:examples}}

\subsection{Proof of Lemma \ref{lem-ex:ur-sparse}: Unbiased Random Sparsification}

From the definition of $k$-nice sampling we have $p_i\eqdef\Prob\(i\in S\) = \tfrac{k}{d}$. Hence
\begin{align*}
\Exp{\cC(x)} &= \frac{d}{k} \Exp{\sum_{i\in S} x_i e_i} = \frac{d}{k}\sum_{i=1}^d p_i x_i e_i = \sum_{i=1}^d x_i e_i = x, \\
\Exp{\norm{\cC(x)}^2} &= \frac{d^2}{k^2} \Exp{\sum_{i\in S} x_i^2} = \frac{d^2}{k^2} \sum_{i=1}^d p_i x_i^2 = \frac{d}{k} \sum_{i=1}^d x_i^2 = \frac{d}{k} \norm{x}^2,
\end{align*}
which implies $\cC\in\U(\tfrac{d}{k})$.

\subsection{Proof of Lemma \ref{lem-ex:br-sparse}: Biased Random Sparsification}

Let $S\subseteq [d]$ be a proper sampling with probability vector $p=(p_1,\dots,p_d)$, where  $p_i \eqdef \Prob(i\in S)>0$ for all $i$. Then
$$\Exp{\cC(x)} = \Diag{p}x = \sum_{i=1}^d p_i x_i e_i \quad\text{and}\quad \Exp{\twonorm{\cC(x)}^2} = \sum_{i=1}^d p_i x_i^2.$$
Letting $q \eqdef \min_i p_i$, we get
\[ q \twonorm{x}^2 \leq   \sum \limits_{i=1}^d p_i x_i^2 = \Exp{\twonorm{\cC(x)}^2}
 =  \left \langle \Exp{\cC(x)}, x \right \rangle . \]
So, $\cC\in \mathbb{B}^1(q, 1)$ and $\cC\in \mathbb{B}^2(q, 1)$. For the third class, note that
$$\Exp{\twonorm{\cC(x) - x}^2} = \sum \limits_{i=1}^d (1-p_i) x_i^2 \leq (1-q)\twonorm{x}^2.$$
Hence, $\cC\in \mathbb{B}^3(\tfrac{1}{q})$.

\subsection{Proof of Lemma \ref{lem-ex:ar-sparse}: Adaptive Random Sparsification}

From the definition of the compression operator, we have
\begin{align*}
\Exp{\twonorm{\cC(x)}^2} &= \Exp{x_i^2} =  \sum \limits_{i=1}^d \frac{|x_i|}{\onenorm{x}} x_i^2 = \frac{\threenorm{x}^3}{\onenorm{x}},\\
\Exp{\lin{\cC(x),x}} &=\Exp{x_i^2} =  \frac{\threenorm{x}^3}{\onenorm{x}},
\end{align*}
whence $\beta=1$. Furthermore, by Chebychev's sum inequality, we have
\begin{equation*}
\tfrac{1}{d^2}\onenorm{x} \twonorm{x}^2 = \left(\sum \limits_{i=1}^d \tfrac{1}{d} |x_i|\right) \left(\sum \limits_{i=1}^d \tfrac{1}{d}x_i^2\right) \leq \sum \limits_{i=1}^d \tfrac{1}{d} |x_i| x_i^2 = \tfrac{1}{d} \threenorm{x}^3,
\end{equation*}
which implies that $\alpha=\frac{1}{d},\,\delta = d$. So, $\cC \in \mathbb{B}^1(\frac{1}{d},1)$, $\cC \in \mathbb{B}^2(\frac{1}{d},1)$, and $\cC \in \mathbb{B}^3(d)$.

\subsection{Proof of Lemma \ref{lem-ex:top-sparse}: Top-$k$ sparsification}

Clearly, $\twonorm{\cC(x)}^2  = \sum_{i=d-k+1}^d x_{(i)}^2$ and $\twonorm{\cC(x)-x}^2  = \sum_{i=1}^{d-k} x_{(i)}^2$. Hence
$$\frac{k}{d} \twonorm{x}^2 \leq \twonorm{\cC(x)}^2 = \lin{\cC(x),x} \leq \twonorm{x}^2,\quad \twonorm{\cC(x)-x}^2 \le \(1-\frac{k}{d}\)\twonorm{x}^2.$$
So, $\cC\in \mathbb{B}^1(\frac{k}{d},1)$, $\cC \in \mathbb{B}^2(\frac{k}{d},1)$, and $\cC \in \mathbb{B}^3(\frac{d}{k})$.

\subsection{Proof of Lemma \ref{lem-ex:gu-rounding}: General Unbiased Rounding}

The unbiasedness follows immediately from the definition (\ref{ex:gu-rounding})
\begin{equation}\label{apx:umbiasedness-rounding}
\Exp{\cC(x)} = \sum_{i=1}^d \Exp{\cC(x)_i}e_i = \sum_{i=1}^d \sign(x_i)\( a_k \frac{a_{k+1}-|x_i|}{a_{k+1}-a_k} + a_{k+1} \frac{|x_i| - a_k}{a_{k+1}-a_k} \)e_i = \sum_{i=1}^d x_i e_i = x.
\end{equation}

Since the rounding compression operator $\cC$ applies to each coordinate independently, without loss of generality we can consider the compression of scalar values $x=t>0$ and show that $\Exp{\cC(t)^2} \le \zeta \cdot t^2$. From the definition we compute the second moment as follows
\begin{equation}\label{apx:eq-it-10}
\Exp{\mathcal{C}(t)^2} = a_k^2 \frac{a_{k+1}-t}{a_{k+1}-a_k} + a_{k+1}^2 \frac{t - a_k}{a_{k+1}-a_k} = (a_k+a_{k+1})t - a_k a_{k+1} = t^2 + (t-a_k)(a_{k+1}-t),
\end{equation}
from which
\begin{equation}\label{apx:eq-it-11}
\frac{\Exp{ \mathcal{C}(t)^2}}{t^2} = 1 + \(1 - \frac{a_k}{t}\) \(\frac{a_{k+1}}{t} - 1\), \quad a_k\le t\le a_{k+1}.
\end{equation}
Checking the optimality condition, one can show that the maximum is achieved at
$$
t_* = \frac{2 a_k a_{k+1}}{a_k+a_{k+1}} = \frac{2}{\frac{1}{a_k} + \frac{1}{a_{k+1}}},
$$
which being the harmonic mean of $a_k$ and $a_{k+1}$, is in the range $[a_k, a_{k+1}]$. Plugging it to the expression for variance we get
$$
\frac{\Exp{ \mathcal{C}(t_*)^2}}{t_*^2} = 1 + \frac{1}{4}\(1-\frac{a_k}{a_{k+1}}\) \(\frac{a_{k+1}}{a_k} - 1\) = \frac{1}{4}\(\frac{a_k}{a_{k+1}} + \frac{a_{k+1}}{a_k} + 2\).
$$

Thus, the parameter $\zeta$ for general unbiased rounding would be
$$
\zeta = \sup_{t>0} \frac{\Exp{ \mathcal{C}(t)^2}}{t^2} = \sup_{k\in\Z}\sup_{a_k\le t\le a_{k+1}} \frac{\Exp{ \mathcal{C}(t)^2}}{t^2} = \frac{1}{4} \sup_{k\in\Z}\(\frac{a_k}{a_{k+1}} + \frac{a_{k+1}}{a_k} + 2\) \ge 1.
$$

\subsection{Proof of Lemma \ref{lem-ex:gb-rounding}: General Biased Rounding}

From the definition (\ref{ex:gb-rounding}) of compression operator $\cC$ we derive the following inequalities
\begin{align*}
\inf_{k\in\Z}\(\frac{2a_k}{a_k+a_{k+1}}\)^2\|x\|_2^2 &\le \|\cC(x)\|_2^2, \\
\|\cC(x)\|_2^2 &\le \sup_{k\in\Z}\frac{2a_{k+1}}{a_k+a_{k+1}}\lin{\cC(x),x},\\
\inf_{k\in\Z}\frac{2a_k}{a_k+a_{k+1}} \|x\|_2^2 &\le \lin{\cC(x),x},
\end{align*}
which imply that $\cC\in\B^1(\alpha, \beta)$ and  $\cC \in \B^2(\gamma, \beta)$, with
$$\beta = \sup_{k\in\Z}\frac{2a_{k+1}}{a_k+a_{k+1}},\qquad \gamma = \inf_{k\in\Z}\frac{2a_k}{a_k+a_{k+1}},\qquad \alpha = \gamma^2.$$

For the third class $\B^3(\delta)$, we need to upper bound the ratio $\twonorm{\cC(x)-x}^2/\twonorm{x}^2$. Again, as $\cC$ applies to each coordinate independently, without loss of generality we consider the case when $x=t>0$ is a scalar.
From  definition (\ref{ex:gb-rounding}), we get
\begin{equation}\label{eq:biased-rounding-nvar}
\frac{(\cC(t)-t)^2}{t^2} = \min\left[\(1-\frac{a_k}{t}\)^2, \(1-\frac{a_{k+1}}{t}\)^2\right], \qquad a_k \le t \le a_{k+1}.
\end{equation}

It can be easily checked that $\(1-\frac{a_k}{t}\)^2$ is an increasing function and $\(1-\frac{a_{k+1}}{t}\)^2$ is a decreasing function of $t\in[a_k, a_{k+1}]$. Thus, the maximum is achieved when they are equal. In contrast to unbiased general rounding, it happens at the middle of the interval,
$$
t_* = \frac{a_k+a_{k+1}}{2} \in [a_k, a_{k+1}].
$$
Plugging $t_*$  into (\ref{eq:biased-rounding-nvar}), we get 
$$
\frac{(\cC(t_*)-t_*)^2}{t_*^2} = \(\frac{a_{k+1}-a_k}{a_{k+1}+a_k}\)^2.
$$

Given this, the parameter $\delta$ can be computed from
$$
1-\frac{1}{\delta} = \sup_{k\in\Z}\sup_{a_k\le t\le a_{k+1}} \frac{(\cC(t)-t)^2}{t^2} = \sup_{k\in\Z} \(\frac{a_{k+1}-a_k}{a_{k+1}+a_k}\)^2,
$$
which gives
$$
\delta = \sup_{k\in\Z}\frac{\(a_k + a_{k+1}\)^2}{4a_k a_{k+1}} \ge 1,
$$
and $\cC \in \mathbb{B}^3(\delta)$.

\subsection{Proof of Lemma \ref{lem-ex:ge-dithering}: General Exponential Dithering}

The proof goes with the same steps as in Theorem 4 of \cite{Cnat}. To show the unbiasedness of $\cC$, first we show the unbiasedness of $\xi(t)$ for $t\in[0,1]$ in the same way as (\ref{apx:umbiasedness-rounding}) was done. Then we note that
$$
\Exp{\cC(x)} = \sign(x) \times \|x\|_p \times \Exp{\xi\(\frac{|x|}{\|x\|_p}\)} = \sign(x) \times \|x\|_p \times \(\frac{|x|}{\|x\|_p}\) = x.
$$

To compute the parameter $\zeta$, we first estimate the second moment of $\xi$ as follows:
\begin{align*}
&\le \mathbbm{1}\(\frac{|x_i|}{\|x\|_p} \ge b^{1-s}\) \cdot \frac{1}{4}\(b+\frac{1}{b}+2\) \frac{x_i^2}{\|x\|^2_p} + \mathbbm{1}\(\frac{|x_i|}{\|x\|_p} < b^{1-s}\) \cdot \frac{|x_i|}{\|x\|_p} b^{1-s} \\
&\le \frac{1}{4}\(b+\frac{1}{b}+2\) \frac{x_i^2}{\|x\|^2_p} + \mathbbm{1}\(\frac{|x_i|}{\|x\|_p} < b^{1-s}\) \cdot \frac{|x_i|}{\|x\|_p} b^{1-s}\,.
\end{align*}

Then we use this bound to estimate the second moment of compressor $\cC$:
\begin{eqnarray*}
\Exp{\twonorm{\cC(x)}^2}
&=& \|x\|_p^2 \sum_{i=1}^d \Exp{\xi\(\frac{|x_i|}{\|x\|_p}\)^2} \\
&\le & \|x\|_p^2 \sum_{i=1}^d \( \frac{1}{4}\(b+\frac{1}{b}+2\) \frac{x_i^2}{\|x\|^2_p} + \mathbbm{1}\(\frac{|x_i|}{\|x\|_p} < b^{1-s}\) \cdot \frac{|x_i|}{\|x\|_p} b^{1-s} \) \\
&=& \frac{1}{4}\(b+\frac{1}{b}+2\) \twonorm{x}^2 + \sum_{i=1}^d \mathbbm{1}\(\frac{|x_i|}{\|x\|_p} < b^{1-s}\) \cdot |x_i| \|x\|_p b^{1-s} \\
&\le & \frac{1}{4}\(b+\frac{1}{b}+2\) \twonorm{x}^2 + \min\(\|x\|_1\|x\|_p b^{1-s},  d \|x\|_p^2 b^{2-2s}\) \\
&\le &\frac{1}{4}\(b+\frac{1}{b}+2\) \twonorm{x}^2 + \min\(d^{\nicefrac{1}{2}}\|x\|_2\|x\|_p b^{1-s},  d \|x\|_p^2 b^{2-2s}\) \\
&\le &\left[ \frac{1}{4}\(b+\frac{1}{b}+2\) + d^{\nicefrac{1}{r}} b^{1-s} \min\(1,  d^{\nicefrac{1}{r}} b^{1-s}\) \right] \twonorm{x}^2 \\
&= &\zeta_b \twonorm{x}^2,
\end{eqnarray*}
where $r = \min(p,2)$ and  H\"{o}lder's inequality is used to bound $\|x\|_p \le d^{\nicefrac{1}{p}-\nicefrac{1}{2}}\twonorm{x}$ in case of $0\le p\le 2$ and $\|x\|_p \le \twonorm{x}$ in the case $p\ge 2$.

\subsection{Proof of Lemma \ref{lem-ex:top-gendith}: Top-$k$ Combined with Exponential Dithering}

From the unbiasedness of general dithering operator $\cC_{dith}$ we have
$$
\Exp{\cC(x)} = \Exp{\cC_{dith}(\cC_{top}(x))} = \cC_{top}(x),
$$
from which we conclude $\lin{\Exp{\cC(x)}, x} = \lin{\cC_{top}(x), x} = \twonorm{\cC_{top}(x)}^2$. Next, using Lemma \ref{lem-ex:ge-dithering} on exponential dithering we get
\begin{equation*}
\Exp{\twonorm{\cC(x)}^2} \leq \zeta_b \cdot \twonorm{\cC_{top}(x)}^2 = \zeta_b \cdot \lin{\Exp{\cC(x)}, x},
\end{equation*}
which implies $\beta=\zeta_b$. Using Lemma \ref{lem-ex:top-sparse} we show $\gamma=\tfrac{k}{d}$ as $\lin{\Exp{\cC(x)}, x}  = \twonorm{\cC_{top}(x)}^2 \ge \tfrac{k}{d}\twonorm{x}^2$. Utilizing the derivations (\ref{apx:eq-it-10}) and (\ref{apx:eq-it-11}) it can be shown that $\Exp{\twonorm{\cC_{dith}(x)}^2} \ge \twonorm{x}^2$ and therefore
$$
\Exp{\twonorm{\cC(x)}^2} \ge \twonorm{\cC_{top}(x)}^2 \ge \tfrac{k}{d}\twonorm{x}^2.
$$

Hence, $\alpha=\tfrac{k}{d}$. To compute the parameter $\delta$ we use Theorem \ref{thm:compression_properties}, which yields $\delta=\tfrac{\beta}{\gamma}=\tfrac{d}{k}\zeta_b$.

\section{Proofs for Section~\ref{sec:analysis_of_biased_GD}}

We now perform analysis of {\tt CGD} for compression operators in $\mathbb{B}^1$, $\mathbb{B}^2$ and $\mathbb{B}^3$, establishing Theorems~\ref{thm:main-I}, \ref{thm:main-II}
 and \ref{thm:main-III}, respectively.

\subsection{Analysis for $\cC\in \mathbb{B}^1(\alpha,\beta)$}

\begin{lemma} \label{lem:1} Assume $f$ is $L$-smooth. Let $\cC\in \mathbb{B}^1(\alpha,\beta)$. Then as long as $0\leq \stepsize \leq \frac{2}{\beta L}$, for each $x\in \R^d$ we have
\[\Exp{f\left(x-\stepsize \cC(\nabla f(x))\right)} \leq f(x) -  \alpha \stepsize \left(1 - \frac{\stepsize \beta  L}{2} \right)  \twonorm{\nabla f(x)}^2.\]
\end{lemma}
\begin{proof}
Letting $g=\nabla f(x)$, we have\footnote{Alternatively, we can write
\begin{eqnarray*}
\Exp{f\left(x-\stepsize \cC(g)\right)} &\leq & 
 f(x) 
 - \stepsize \langle \Exp{\cC(g)}, g \rangle
  + \frac{\stepsize^2 L}{2} \Exp{\twonorm{\cC(g)}^2}\\
&\overset{\eqref{eq:alpha-beta}}{\leq }& f(x) - \frac{\stepsize}{\beta} \Exp{\twonorm{\cC(g)}^2}  
+ \frac{\stepsize^2 L}{2} \Exp{\twonorm{\cC(g)}^2}\\
&=& f(x) - \frac{\stepsize}{\beta} \left( 1-  \frac{\stepsize \beta  L}{2} \right)\Exp{\twonorm{\cC(g)}^2}\\
&\overset{\eqref{eq:alpha-beta}}{\leq }& f(x) -  \frac{\alpha}{ \beta} \stepsize \left(1 - \frac{\stepsize \beta  L}{2} \right)  \twonorm{g}^2.
\end{eqnarray*}
Both approaches lead to the same bound.}
\begin{eqnarray*}
\Exp{f\left(x-\stepsize \cC(g)\right)} &\leq & \Exp{ f(x) + \left \langle g, -\stepsize \cC(g) \right \rangle + \frac{L}{2}\twonorm{-\stepsize \cC(g) }^2} \\
&=& f(x) - \stepsize \langle \Exp{\cC(g)}, g \rangle + \frac{\stepsize^2 L}{2} \Exp{\twonorm{\cC(g)}^2}\\
&\overset{\eqref{eq:alpha-beta}}{\leq }& f(x) - \stepsize \langle \Exp{\cC(g)}, g \rangle + \frac{ \stepsize^2 \beta L}{2} \left\langle \Exp{\cC(g)}, g \right\rangle \\
&=& f(x) - \stepsize \left(1 - \frac{\stepsize \beta  L}{2} \right) \langle \Exp{\cC(g)}, g \rangle  \\
&\overset{\eqref{eq:alpha-beta}}{\leq }& f(x) -  \frac{\alpha}{ \beta} \stepsize \left(1 - \frac{\stepsize \beta  L}{2} \right)  \twonorm{g}^2.
\end{eqnarray*}
\end{proof}

\paragraph{Proof of Theorem~\ref{thm:main-I}}

\begin{proof}
Since $f$ is $\mu$-strongly convex, $\twonorm{\nabla f(x^k)}^2 \geq 2 \mu (f(x^k)-f(x^\star))$. Combining this with Lemma~\ref{lem:1} applied to $x=x^k$ and $g=\nabla f(x^k)$, we get
\begin{eqnarray*}
\Exp{f\left(x^k-\stepsize \cC(\nabla f(x^k))\right)}  - f(x^\star) &\leq &f(x^k) - f(x^\star)-   \frac{\alpha}{\beta} \stepsize \mu \left(2 - \stepsize \beta  L \right)  (f(x^k)-f(x^\star))\\
&=& \left(1- \frac{\alpha}{\beta} \stepsize \mu \left(2 - \stepsize \beta  L \right)\right) (f(x^k)-f(x^\star)).
\end{eqnarray*}
\end{proof}

\subsection{Analysis for $\cC\in \mathbb{B}^2(\gamma,\beta)$}

\begin{lemma} \label{lem:1-II} Assume $f$ is $L$-smooth. Let $\cC\in \mathbb{B}^2(\gamma,\beta)$. Then as long as $0\leq \stepsize \leq \frac{2}{\beta L}$, for each $x\in \R^d$ we have
\[\Exp{f\left(x-\stepsize \cC(\nabla f(x))\right)} \leq f(x) -  \gamma \stepsize \left(1 - \frac{\stepsize \beta L}{2} \right)  \twonorm{\nabla f(x)}^2.\]
\end{lemma}
\begin{proof}
Letting $g=\nabla f(x)$, we have
\begin{eqnarray*}
\Exp{f\left(x - \stepsize \cC(g)\right)} &\leq & \Exp{ f(x) + \dotprod{ g, -\stepsize \cC(g) } + \frac{L}{2}\twonorm{-\stepsize \cC(g) }^2} \\
&=& f(x) - \stepsize \dotprod{ \Exp{\cC(g)}, g } + \frac{\stepsize^2 L}{2} \Exp{\twonorm{\cC(g)}^2}\\
&\overset{\eqref{eq:alpha-betaII}}{\leq }& f(x) - \stepsize \left(1 - \frac{\stepsize \beta   L}{2} \right) \dotprod{ \Exp{\cC(g)}, g} \\
&\overset{\eqref{eq:alpha-betaII}}{\leq }& f(x) -  \gamma \stepsize \left(1 - \frac{\stepsize \beta L}{2} \right)  \twonorm{g}^2.
\end{eqnarray*}
\end{proof}

\paragraph{Proof of Theorem~\ref{thm:main-II}}
\begin{proof}
Since $f$ is $\mu$-strongly convex, $\twonorm{\nabla f(x^k)}^2 \geq 2 \mu (f(x^k)-f(x^\star))$. Combining this with Lemma~\ref{lem:1-II} applied to $x=x^k$ and $g=\nabla f(x^k)$, we get
\begin{eqnarray*}
\Exp{f\left(x^k-\stepsize \cC(\nabla f(x^k))\right)}  - f(x^\star) &\leq &f(x^k) - f(x^\star)-  \mu \gamma \stepsize (2 - \stepsize \beta  L)  (f(x^k)-f(x^\star))\\
&=& \left(1- \mu \gamma \stepsize (2 - \stepsize \beta  L)\right) (f(x^k)-f(x^\star)).
\end{eqnarray*}

\end{proof}

\subsection{Analysis for $ \cC\in \mathbb{B}^3(\delta)$}

\begin{lemma} \label{lem:1-III} Assume $f$ is $L$-smooth. Let $ \cC\in \mathbb{B}^3(\delta)$. Then as long as $0\leq \stepsize \leq \frac{1}{L}$, for each $x\in \R^d$ we have
\[\Exp{f\left(x-\stepsize \cC(\nabla f(x))\right)} \leq f(x) - \frac{\stepsize}{2 \delta}\twonorm{\nabla f(x)}^2.\]
\end{lemma}
\begin{proof}
Letting $g=\nabla f(x)$, note that for any stepsize $\stepsize\in \R$ we have
\begin{eqnarray}
\Exp{f\left(x-\stepsize \cC(g)\right)} &\leq & \Exp{ f(x) + \left \langle g, -\stepsize \cC(g) \right \rangle + \frac{L}{2}\twonorm{-\stepsize \cC(g) }^2} \notag \\
&=& f(x) - \stepsize \langle \Exp{\cC(g)}, g \rangle + \frac{\stepsize^2 L}{2} \Exp{\twonorm{\cC(g)}^2}.\label{eq:niugf7gh--bygTTr}
\end{eqnarray}
Since $\cC \in \mathbb{B}^3(\delta)$, we have 
$ \Exp{\twonorm{ \cC(g)-g}^2} \leq   \left(1- \frac{1}{\delta} \right) \twonorm{g}^2$. Expanding the square, we get
\[ \twonorm{g}^2 - 2 \Exp{ \langle  \cC(g), g \rangle } +  \Exp{\twonorm{\cC(g)}^2} \leq   \left(1- \frac{1}{\delta} \right)\twonorm{g}^2 .\]
Subtracting $\twonorm{g}^2$ from both sides, and multiplying both sides by $\frac{\stepsize}{2}$ (now we assume that $\stepsize>0$), we get 
\[- \stepsize \langle \Exp{\cC(g)}, g \rangle + \frac{\stepsize}{2} \Exp{\twonorm{\cC(g)}^2} \leq  -\frac{\stepsize}{2\delta} \twonorm{g}^2  .\]
Assuming that $\stepsize L \leq 1$, we can combine this with \eqref{eq:niugf7gh--bygTTr} and the lemma is proved.

\end{proof}

\paragraph{Proof of Theorem~\ref{thm:main-III} }
\begin{proof}
Since $f$ is $\mu$-strongly convex, $\twonorm{\nabla f(x^k)}^2 \geq 2 \mu (f(x^k)-f(x^\star))$. Combining this with Lemma~\ref{lem:1-III} applied to $x=x^k$ and $g=\nabla f(x^k)$, we get
\begin{eqnarray*}
\Exp{f\left(x^k-\stepsize \cC( \nabla f(x^k))\right)}  - f(x^\star) &\leq &f(x^k) - f(x^\star)-   \frac{\stepsize \mu}{\delta}  (f(x^k)-f(x^\star)) \\
&=& \left(1-\frac{\stepsize \mu}{\delta} \right) (f(x^k)-f(x^\star)).
\end{eqnarray*}

\end{proof}

\section{Proofs for Section \ref{sec:stat}} \label{apx:stat}

\begin{proof}[Proof of Lemma \ref{lem:stat-top-random}]
{\bf (a)} 
As it was already mentioned, we have the following expressions for $\omega_{rnd}^k$ and $\omega_{top}^k$:
$$
\omega^k_{rnd}(x) = \(1-\frac{k}{d}\)\sum_{i=1}^{d} x_i^2, \quad \omega^k_{top}(x) = \sum_{i=1}^{d-k} x_{(i)}^2.
$$

The expected variance $\Exp{\omega_{rnd}^k}$ for Rand-$k$ is easy to compute as all coordinates are independent and uniformly distributed on $[0,1]$:
\begin{equation}\label{apx:id-15}
\Exp{x_i^2} \equiv \int_{[0,1]^d} x_i^2\,dx = \int_0^1 x_i^2\,d x_i = \frac{1}{3},
\end{equation}
which implies
\begin{equation}\label{apx:id-12}
\Exp{\omega_{rnd}^k(x)} = \(1-\frac{k}{d}\)\sum_{i=1}^{d} \Exp{x_i^2} = \(1-\frac{k}{d}\)\frac{d}{3} = \frac{d-k}{3}.
\end{equation}

In order to compute the expected variance $\Exp{\omega_{top}^k}$ for Top-$k$, we use the following formula from order statistics (see \eqref{eq:arnold222}, \eqref{eq:arnold223} or (2.2.2), (2.2.3) of \cite{arnold}) \footnote{see also \url{https://en.wikipedia.org/wiki/Order_statistic}, \url{https://www.sciencedirect.com/science/article/pii/S0167715212001940}}
\begin{equation}\label{apx:id-14}
\Exp{x^2_{(i)}} \equiv \int_{[0,1]^d} x_{(i)}^2\,dx = \frac{\Gamma(i+2)\Gamma(d+1)}{\Gamma(i)\Gamma(d+3)} = \frac{i(i+1)}{(d+1)(d+2)},
\end{equation}
from which we derive
\begin{align}\label{apx:id-13}
\begin{split}
\Exp{\omega_{top}^k}
&= \sum_{i=1}^{d-k} \Exp{x^2_{(i)}} = \frac{1}{(d+1)(d+2)} \sum_{i=1}^{d-k} i(i+1) \\
&= \frac{1}{(d+1)(d+2)} \cdot \frac{(d-k)(d-k+1)(d-k+2)}{3} \\
&= \frac{d-k}{3}\(1-\frac{k}{d+1}\)\(1-\frac{k}{d+2}\).
\end{split}
\end{align}

Combining (\ref{apx:id-12}) and (\ref{apx:id-13}) completes the first relation. Thus, on average (w.r.t. uniform distribution) Top-$k$ has roughly $\(1-\nicefrac{k}{d}\)^2$ times less variance than Rand-$k$.

For the second relation, we use (\ref{apx:id-15}) and (\ref{apx:id-14}) for $i=d$ and get
\begin{equation*}
\frac{\Exp{s^1_{top}(x)}}{\Exp{s^1_{rnd}(x)}} = \frac{\Exp{x^2_{(d)}}}{\Exp{x_d^2}} = \frac{\tfrac{d(d+1)}{(d+1)(d+2)}}{\tfrac{1}{3}} = \frac{3 d}{d+2}.
\end{equation*}
Clearly, one can extend this for any $k\in[d]$.

{\bf (b)} 
Recall that for the standard exponential distribution (with $\lambda=1$) probability density function (PDF) is given as follows:
\begin{equation*}\label{apx:st-exp-dist}
\phi(t) = e^{-t}, \quad t\in[0,\infty).
\end{equation*}
Both mean and variance can be shown to be equal to $1$. The expected saving $\Exp{s^1_{rnd}}$ can be computed directly:
\begin{equation*}\label{rexp}
\Exp{s^1_{rnd}(x)} = \Exp{x_d^2} = \Var{[x_d]} + \Exp{x_d}^2  = 2.
\end{equation*}

To compute the expected saving $\Exp{s^1_{top}(x)} = \Exp{x^2_{(d)}}$ we prove the following lemma:

\begin{lemma}
Let $x_1,\,x_2,\,\dots,\,x_d$ be an i.i.d.\ sample from the standard exponential distribution and
$$y_i \eqdef (d-i+1)(x_{(i)} - x_{(i-1)}), \quad  1\le i\le d,$$
where $x_{(0)} \eqdef 0$. Then $y_1,\,y_2,\,\dots,\,y_d$ is an i.i.d.\ sample from the standard exponential distribution.
\end{lemma}
\begin{proof}
    The joint density function of $x_{(1)}, \ldots, x_{(d)}$ is given by (see \eqref{eq:arnold223})
    \begin{equation*}
        \phi_{x_{(1)}, \dots, x_{(d)}} (u_1, \dots, u_d) = d! \prod\limits_{i = 1}^d \phi(u_i) = d! \exp{\left(-\sum\limits_{i=1}^d u_i\right)}, \quad 0 \leq u_1 \le \ldots \le u_d < \infty.
    \end{equation*}
    Next we express variables $x_{(i)}$ using new variables $y_i$
    \begin{equation*}
        x_{(1)} = \frac{y_1}{d},\; x_{(2)} = \frac{y_1}{d} + \frac{y_2}{d-1},\; \ldots,\; x_{(d)}=\frac{y_1}{d} + \frac{y_2}{d-1} + \ldots + y_d,
    \end{equation*}
    with the transformation matrix
    \begin{equation*}
    A = 
    \begin{pmatrix}
    \frac{1}{d} & 0 & \ldots & 0\\
    \frac{1}{d} & \frac{1}{d-1} & \ldots & 0\\
    \vdots & \vdots & \ddots & \vdots\\
    \frac{1}{d} & \frac{1}{d-1} & \ldots & 1
    \end{pmatrix}
    \end{equation*}
    Then the joint density $\psi_{y_1, \ldots, y_d} (u) = \psi_{y_1, \ldots, y_d} (u_1,\dots,u_d)$ of new variables $y_1, \ldots, y_d$ is given as follows
    \begin{equation*}
        \psi_{y_1, \ldots, y_d} (u) = \frac{\phi_{x_{(1)}, \ldots, x_{(d)}} (Au)}{\abs{\textrm{det}\,A^{-1}}} = \abs{\textrm{det}\,A}\cdot \phi_{x_{(1)}, \ldots, x_{(d)}} (Au)
    \end{equation*}
    Notice that $\sum\limits_{i=1}^d u_i = \sum\limits_{i=1}^d (Au)_i$ and $|\textrm{det}\, A| = \nicefrac{1}{d!}$. Hence
    \begin{equation*}
        \psi_{y_1, \ldots, y_d} (u) = \exp{\left(-\sum\limits_{i=1}^d u_i\right)}, \quad 0 \leq u_1 \le \ldots \le u_d \le \infty,
    \end{equation*}
    which means that variables $y_1, \ldots y_d$ are independent and have standard exponential distribution.
\end{proof}

Using this lemma we can compute the mean and the second moment of $x_{(d)} = \sum_{i=1}^d \frac{y_i}{d-i+1}$ as follows
\begin{align*}
    \Exp{x_{(d)}} &= \sum_{i=1}^d \Exp{\frac{y_i}{d-i+1}} = \sum_{i=1}^d \frac{\Exp{y_i}}{d-i+1} = \sum\limits_{i=1}^d \frac{1}{i}, \\
    \Var{[x_{(d)}]} &= \sum_{i=1}^d \Var\left[{\frac{y_i}{d-i+1}}\right] = \sum_{i=1}^d \frac{\Var{[y_i]}}{(d-i+1)^2} = \sum\limits_{i=1}^d \frac{1}{i^2},
\end{align*}
from which we conclude the lemma as
\begin{equation*}
    \Exp{s^1_{top}(x)} = \Exp{x^2_{(d)}} = \Var{[x_{(d)}]} + \Exp{x_{(d)}}^2 = \sum\limits_{i=1}^d \frac{1}{i^2} + \left( \sum\limits_{i=1}^d \frac{1}{i}\right)^2 \approx \cO(\log^2 d).
\end{equation*}

\end{proof}



\subsection{Proof of Theorem \ref{thm:sparsified} (Convergence guarantees for Algorithm \ref{alg})}

In this section, we include our analysis for the Distributed SGD with biased compression. Our analysis is closely related to the analysis of \cite{stich2019}.

We start with the definition of some auxiliary objects:
\begin{definition}
    The sequence $\{a^k\}_{k\geq0}$ of positive values is $\tau$-slow decreasing for parameter $\tau$:
    \begin{eqnarray}
    \label{decrease}
    a^{k+1} \leq a^{k}, \quad a^{k+1}\left(1 + \frac{1}{2\tau} \right) \geq a^k, \quad \forall k \geq 0
    \end{eqnarray}
    The sequence $\{a^k\}_{k\geq0}$ of positive values is $\tau$-slow increasing for parameter $\tau$:
    \begin{eqnarray}
    \label{increase}
    a^{k+1} \geq a^{k}, \quad a^{k+1} \leq a^k\left(1 + \frac{1}{2\tau} \right), \quad \forall k \geq 0
    \end{eqnarray}
\end{definition}
And let:
\begin{equation}
    \label{tx}
    \tilde x^k = x^k - \frac{1}{n} \sum\limits_{i=1}^n e_i^k \,, \quad \forall k \geq 0
\end{equation}
\begin{equation}
    \label{gt}
 g^k = \frac{1}{n} \sum\limits_{i=1}^n g_i^k
\end{equation}
It is easy to see:
\begin{eqnarray}
 \label{eq:tilde}
 \tilde x^{k+1} &=& x^{k+1} - \frac{1}{n} \sum\limits_{i=1}^n e_i^{k+1} \nonumber\\
 &\stackrel{\eqref{error}, \eqref{step}}{=} &
 \left(x^k - \frac{1}{n} \sum\limits_{i=1}^n \tilde g_i^k \right) - \left(\frac{1}{n} \sum\limits_{i=1}^n [e_i^k + \stepsize^k g_i^k- \tilde g_i^k ]\right) \nonumber\\
 &=& \tilde x^k - \frac{\stepsize^k}{n} \sum\limits_{i=1}^n  g_i^k
\end{eqnarray}
\begin{lemma}
\label{lemma:main}
If $\stepsize^k \leq \frac{1}{4L\left(1 + \nicefrac{2B}{n}\right)}$, $\forall k \geq 0$, then for $\{\tilde x^k\}_{k \geq 0}$ defined as in~\eqref{tx},
\begin{eqnarray}
\label{eq:main}
  \Exp{\norm{\tilde x^{k+1} - x^*}^2}  &\leq&
 \left(1-\frac{\mu \stepsize^k}{2}\right) \Exp{\norm{\tilde x^{k} - x^*}^2} 
 - \frac{\stepsize^k}{2} \Exp{f(x^k)- f^*} \nonumber\\
  & & ~ + ~3 L \stepsize^k   \Exp{\norm{x^k - \tilde x^k}^2} + (\stepsize^k)^2\frac{C+2BD}{n}
\end{eqnarray}

\end{lemma}
\begin{proof}
We consider the following equalities, using the relationship between  $\tilde x_{k+1}$ and $\tilde x_k$:
\begin{eqnarray*}
  \norm{\tilde x^{k+1} - x^*}^2  &\stackrel{\eqref{gt},\eqref{eq:tilde}}{=}& \norm{\tilde x^{k} - x^*}^2 - 2\stepsize^k \lin{g^k, \tilde x^k-x^*}+ (\stepsize^k)^2 \norm{g^k}^2 \\
  &=& \norm{\tilde x^{k} - x^*}^2 - 2\stepsize^k \lin{g^k, x^k-x^*}+ (\stepsize^k)^2 \norm{g^k}^2  + 2\stepsize^k\lin{g^k, x^k - \tilde x^k}.
\end{eqnarray*}
Taking the conditional expectation conditioned on previous iterates, we get
\begin{eqnarray*}
  && \Exp{\norm{\tilde x^{k+1} - x^*}^2} \\
  &=& \norm{\tilde x^{k} - x^*}^2 - 2\stepsize^k \lin{\Exp{g^k}, x^k-x^*} + (\stepsize^k)^2 \cdot \Exp{\norm{g^k}^2} + 2\stepsize^k\lin{ \Exp{g^k}, x^k - \tilde x^k}  \nonumber\\
  &\stackrel{\eqref{stgr1_main},\eqref{gt}}{=}& \norm{\tilde x^{k} - x^*}^2 - 2\stepsize^k \lin{ \Exp{g^k}, x^k-x^*}  \nonumber\\
  & & + (\stepsize^k)^2 \cdot \Exp{\norm{\nabla f(x^k) + \frac{1}{n} \sum\limits_{i=1}^n \xi_i^k }^2} + 2\stepsize^k \lin{\Exp{g^k}, x^k - \tilde x^k}  \nonumber\\
  &=& \norm{\tilde x^{k} - x^*}^2 - 2\stepsize^k \lin{ \Exp{g^k}, x^k-x^*}  \nonumber\\
  & &+ (\stepsize^k)^2 \cdot \Exp{\norm{\nabla f(x^k)}^2 + 2 \lin{\nabla f(x^k), \frac{1}{n} \sum\limits_{i=1}^n \xi_i^k} + \norm{\frac{1}{n} \sum\limits_{i=1}^n \xi_i^k }^2}  + 2\stepsize^k \lin{\Exp{g^k}, x^k - \tilde x^k}.
\end{eqnarray*}
Given the unbiased stochastic gradient ($\Exp{\xi_i^k} = 0$):
\begin{eqnarray*}  
  \Exp{\norm{\tilde x^{k+1} - x^*}^2} &\stackrel{}{=}& \norm{\tilde x^{k} - x^*}^2 - 2\stepsize^k \lin{\nabla f(x^k), x^k-x^*} \nonumber\\
  &&+ (\stepsize^k)^2 \norm{\nabla f(x^k)}^2 + (\stepsize^k)^2 \cdot \Exp{\norm{\frac{1}{n} \sum\limits_{i=1}^n \xi_i^k }^2}  + 2\stepsize^k \lin{\nabla f(x^k), x^k - \tilde x^k} 
\end{eqnarray*}

Using that $\xi_i^k$ mutually independent and $\Exp{\xi_i^k} = 0$ we have:
\begin{eqnarray}  
  &\stackrel{\eqref{sum_sqr}}{\leq}&  \norm{\tilde x^{k} - x^*}^2 - 2\stepsize^k \lin{\nabla f(x^k), x^k-x^*} \nonumber \\
  &&+ (\stepsize^k)^2 \cdot \norm{\nabla f(x^k)}^2 + (\stepsize^k)^2 \cdot \frac{1}{n^2}\sum\limits_{i=1}^n\Exp{\norm{\xi_i^k}^2} +2\stepsize^k \lin{\nabla f(x^k), x^k - \tilde x^k} \nonumber \\
  &\stackrel{\eqref{stgr3_main}}{\leq}&  \norm{\tilde x^{k} - x^*}^2 - 2\stepsize^k \lin{ \nabla f(x^k), x^k-x^*}  \nonumber\\
  &&+ (\stepsize^k)^2 \cdot \norm{\nabla f(x^k)}^2 + \frac{(\stepsize^k)^2}{n^2}\sum\limits_{i=1}^n\left[B\norm{\nabla f_i(x^k)}^2\right]  + \frac{(\stepsize^k)^2}{n} C \nonumber\\
  &&+ 2\stepsize^k \lin{\nabla f(x^k), x^k - \tilde x^k} \nonumber\\
  &\stackrel{\eqref{L-smooth4}}{\leq}&  \norm{\tilde x^{k} - x^*}^2 - 2\stepsize^k \lin{ \nabla f(x^k), x^k-x^*}  \nonumber\\
  &&+ (\stepsize^k)^2 \cdot 2L(f(x^k) - f(x^*)) + \frac{(\stepsize^k)^2}{n^2}\sum\limits_{i=1}^n\left[B\norm{\nabla f_i(x^k)}^2\right]  + \frac{(\stepsize^k)^2}{n} C \nonumber\\
  &&+ 2\stepsize^k \lin{\nabla f(x^k), x^k - \tilde x^k}.
  \label{long1}
\end{eqnarray}

All $f_i$ are $L$-smooth and $ \mu $-strongly convex, thus $f$ is $L$-smooth and $\mu$-strongly convex. We can rewrite $\frac{1}{n}\sum\limits_{i=1}^n\norm{\nabla f_i(x^k)}^2$:
\begin{eqnarray*}
    \frac{1}{n}\sum\limits_{i=1}^n\norm{\nabla f_i(x^k)}^2
    &=& \frac{1}{n}\sum\limits_{i=1}^n\norm{\nabla f_i(x^k) - \nabla f_i(x_*) + \nabla f_i(x^*)}^2  \nonumber\\ 
    &\stackrel{\eqref{sum_sqr}}{\leq}&  \frac{2}{n}\sum\limits_{i=1}^n \left(\norm{\nabla f_i(x^k) - \nabla f_i(x^*)}^2 + \norm{\nabla f_i(x^*)}^2\right) \nonumber\\ 
    &\stackrel{\eqref{L-smooth3}}{\leq}&  \frac{2}{n}\sum\limits_{i=1}^n \left[ 2L \left(f_i(x^k) - f_i(x^*) -\langle\nabla f_i(x^*), x^k - x^*\rangle \right) + \norm{\nabla f_i(x^*)}^2\right] .
\end{eqnarray*}
Using definition of $D = \frac{1}{n} \sum_{i=1}^n \norm{\nabla f_i(x^*)}^2$:
\begin{eqnarray}
    \frac{1}{n}\sum\limits_{i=1}^n\norm{\nabla f_i(x^k)}^2 &\stackrel{}{\leq}&   4L \left(f(x^k) - \nabla f(x^*)\right) + 2D
    \label{sum_sqr_grad}
\end{eqnarray}

Substituting \eqref{sum_sqr_grad} to \eqref{long1}:
\begin{eqnarray}
  \Exp{\norm{\tilde x^{k+1} - x^*}^2}
  &=&  \norm{\tilde x^{k} - x^*}^2 - 2\stepsize^k \lin{ \nabla f(x^k), x^k-x^*} + (\stepsize^k)^2\cdot 2L\left(1 + \frac{2B}{n}\right)(f(x^k) - f(x^*)) \nonumber\\
  &&   + (\stepsize^k)^2\frac{C+2BD}{n}  + 2\stepsize^k \lin{\nabla f(x^k), x^k - \tilde x^k} 
  \label{long2}
\end{eqnarray}
By \eqref{strong_conv1} we have for $f$:
\begin{eqnarray}
  -2\lin{\nabla f(x^k), x^k-x^*} \leq - \mu \norm{x^k- x^*}^2  - 2(f(x^k)-f^*).
  \label{temp1_lem1}
\end{eqnarray}
Using \eqref{inner_prod} with $\xi=\nicefrac{1}{2L}$ and $L$-smothness of $f$ \eqref{L-smooth4}:
\begin{eqnarray}
2\lin{ \nabla f(x^k), \tilde x^k - x^k}  \leq \frac{1}{2L} \norm{\nabla f(x^k)}^2 + 2L\norm{x^k -\tilde x^k}^2 \leq f(x^k)-f^* + 2L \norm{x^k -\tilde x^k}^2 .
\label{temp2_lem1}
\end{eqnarray}
By \eqref{sum_sqr} for $\norm{\tilde x^k - x^*}^2$, we get:
\begin{eqnarray}
 -\norm{x^k - x^*}^2 \leq - \frac{1}{2} \norm{\tilde x^k - x^*}^2 + \norm{x^k - \tilde x^k}^2.
 \label{temp3_lem1}
\end{eqnarray}
Plugging \eqref{temp1_lem1}, \eqref{temp2_lem1}, \eqref{temp3_lem1} into \eqref{long2}:
\begin{eqnarray*}
 \norm{\tilde x^{k+1} - x^*}^2
  &\leq& \left(1-\frac{\mu \stepsize^k}{2}\right) \norm{\tilde x^{k} - x^*}^2 - \stepsize^k \left[1 - \stepsize^k\cdot 2L\left(1 + \frac{2B}{n}\right)\right] (f(x^k)- f^*)  \\
  && + \stepsize^k (2L+\mu) \norm{x^k - \tilde x^k}^2 + (\stepsize^k)^2\frac{C+2BD}{n}\\
\end{eqnarray*}
The lemma follows by the choice $\stepsize^k \leq \frac{1}{4L\left(1 + \nicefrac{2B}{n}\right)}$ and $L \geq \mu$.
\end{proof}
\begin{lemma} \label{lemma:sparse_final}
$\stepsize^k \leq \frac{1}{14 (2\delta + B) L}$, $\forall k \geq 0$ and $\{(\stepsize^k)^2\}_{k \geq 0}$ -- $2\delta$-slow decreasing. Then
\begin{eqnarray}
 \Exp{\norm{\frac{1}{n}\sum\limits_{i=1}^n e_i^{k+1}}^2} &\leq& \frac{(1-1/\delta)}{49L (2\delta + B)} \sum_{j=0}^k\left[ \left(1-\frac{1}{4\delta}\right)^{k-j} (f(x^j) - f(x^*))\right] \notag \\
 &&\qquad + \stepsize^k \frac{2(\delta-1)}{7L}\left(2D + \frac{C}{2\delta + B} \right) \,. \label{eq:sparse_bound1}
\end{eqnarray}
Furthermore, for any $4\delta$-slow increasing non-negative sequence $\{w^k\}_{k \geq 0}$ it holds:
\begin{eqnarray}
    3L \cdot \sum\limits_{k=0}^K w^k \cdot  \Exp{\norm{\frac{1}{n}\sum\limits_{i=1}^n e_i^{k}}^2} 
    \leq \frac{1}{4} \sum\limits_{k=0}^K w^k (\Exp{f(x^k)} - f(x_*)) +  \left(3\delta D + \frac{3C}{4}\right) \sum\limits_{k=0}^K w^k\stepsize^k. \label{eq:sparse_bound3}
\end{eqnarray}
\end{lemma} 
\begin{proof}

We prove the first part of the statement:
\begin{eqnarray*}
\Exp{\norm{\frac{1}{n}\sum\limits_{i=1}^n e_i^{k+1}}^2}  &\stackrel{\eqref{sum_sqr}}{\leq}& \frac{1}{n} \Exp{\sum\limits_{i=1}^n \norm{e_i^{k+1}}^2} \\ &\stackrel{\eqref{error}}{=}&\frac{1}{n}\Exp{\sum\limits_{i=1}^n \norm{e_i^k + \stepsize^k g^k_i - \tilde g_i^k }^2}\nonumber \\ &\stackrel{\eqref{comp_grad}}{=}& \frac{1}{n} \sum\limits_{i=1}^n \Exp{\norm{e_i^k + \stepsize^k g_i^k - \cC(e_i^k + \stepsize^k g^k_i)}^2} \nonumber \\
&\stackrel{\eqref{eq:biasedIII}}{\leq}& \frac{1-1/\delta}{n} \sum\limits_{i=1}^n \Expg{\norm{e_i^k + \stepsize^k g_i^k}^2} \\ 
&\stackrel{\eqref{stgr1_main}}{=}& \frac{1-1/\delta}{n} \sum\limits_{i=1}^n \Expg{\norm{e_i^k + \stepsize^k\nabla f_i(x^k) + \stepsize^k \xi_i^k }^2}
\end{eqnarray*}
Here we have taken into account that the operator of full expectation is a combination of operators of expectation by the randomness of the operator and the randomness of the stochastic gradient, i.e. $\Exp{\cdot} = \ExpC{\Expg{\cdot}}$.
Given the unbiased stochastic gradient ($\Exp{\xi_i^k} = 0$):
\begin{eqnarray*}
\Exp{\norm{\frac{1}{n}\sum\limits_{i=1}^n e_i^{k+1}}^2} &\stackrel{}{\leq}&\frac{1-1/\delta}{n} \sum\limits_{i=1}^n \left[\norm{e_i^k + \stepsize^k\nabla f_i(x^k)}^2 + \Expg{ \norm{\stepsize^k \xi_i^k}^2}\right] \\
&\stackrel{\eqref{stgr3_main}}{\leq}&\frac{1-1/\delta}{n} \sum\limits_{i=1}^n \left[\norm{e_i^k + \stepsize^k\nabla f_i(x^k)}^2 +(\stepsize^k)^2 \left(B \norm{\nabla f_i(x^k)}^2 + C \right)\right] \nonumber\\
\end{eqnarray*}
Using \eqref{inner_prod_and_sqr} with some $\xi$:
\begin{eqnarray*}
&& \Exp{\norm{\frac{1}{n}\sum\limits_{i=1}^n e_i^{k+1}}^2} \\
&\leq& \frac{1}{n} \Exp{\sum\limits_{i=1}^n \norm{e_i^{k+1}}^2} \\
&\leq& \frac{1- \frac{1}{\delta}  }{n} \sum\limits_{i=1}^n  \left[(1+\xi) \norm{e_i^k}^2 
+ (\stepsize^k)^2 \left(1+ \frac{1}{\xi} \right) \norm{\nabla f_i(x^k)}^2
+(\stepsize^k)^2 B \norm{\nabla f_i(x^k)}^2 + (\stepsize^k)^2 C \right]\\
&=&  \left(1- \frac{1}{\delta} \right)  \left[(1+\xi) \left(\frac{1}{n} \sum\limits_{i=1}^n  \norm{e_i^k}^2\right) + (\stepsize^k)^2 \left( 1+ \frac{1}{\xi} + B \right) \left(\frac{1}{n} \sum\limits_{i=1}^n\norm{\nabla f_i(x^k)}^2 \right) + (\stepsize^k)^2 C\right] \\
&\stackrel{\eqref{sum_sqr_grad}}{\leq}& \left(1- \frac{1}{\delta} \right)  \left[(1+\xi) \left(\frac{1}{n} \sum\limits_{i=1}^n  \norm{e_i^k}^2\right)\right] \\
&&\quad + \left(1- \frac{1}{\delta} \right)  \left[(\stepsize^k)^2 \left( 1+ \frac{1}{\xi} + B \right) \left( 4L(f(x^k) - f(x^*)) + 2D \right) + (\stepsize^k)^2 C\right]
\end{eqnarray*}
Using the recurrence for $\frac{1}{n} \sum\limits_{i=1}^n \norm{e_i^k}^2$ , and let $\xi = \frac{1}{2(\delta - 1)}$, then  $(1+1/\xi) \leq 2\delta$, and $(1-1/\delta)(1+\xi) = (1-\nicefrac{1}{2\delta})$  we have
\begin{eqnarray*}
&&\Exp{\norm{\frac{1}{n}\sum\limits_{i=1}^n e_i^{k+1}}^2} \\
&\leq& \frac{1}{n} \Exp{\sum\limits_{i=1}^n \norm{e_i^{k+1}}^2} \\
&\leq&  \left(1- \frac{1}{\delta} \right) \sum_{j=0}^k (\stepsize^j)^2 \left[ \left(1- \frac{1}{\delta}\right)(1+\xi)  \right]^{k-j}   \left( 1+ \frac{1}{\xi} + B \right)  \left( 4L(\Exp{f(x^j)} - f(x^*)) + 2D \right) \\
&&\quad + \left(1- \frac{1}{\delta} \right) \sum_{j=0}^k (\stepsize^j)^2 \left[ \left(1- \frac{1}{\delta}\right)(1+\xi)  \right]^{k-j} C \\
 &\leq & \left(1- \frac{1}{\delta} \right)  \sum_{j=0}^k  (\stepsize^j)^2 \left(1-\frac{1}{2\delta}\right)^{k-j}  \left( \left( 2\delta + B \right) \left( 4L(\Exp{f(x^j)} - f(x^*)) + 2D \right) + C \right)\,.
\end{eqnarray*} 
For $2\delta$-slow decreasing  $\{(\stepsize^k)^2\}_{k \geq 0}$ by definition \eqref{decrease}  we get that $(\stepsize^{j})^2\leq (\stepsize^k)^2 \left(1+\frac{1}{4\delta} \right)^{k-j}$. Due to the fact that $(1-\nicefrac{1}{2\delta})(1+\nicefrac{1}{4\delta})\leq (1-\nicefrac{1}{4\delta})$, we have:
\begin{eqnarray*}
&&\Exp{\norm{\frac{1}{n}\sum\limits_{i=1}^n e_i^{k+1}}^2} \\
&\leq& \frac{1}{n} \Exp{\sum\limits_{i=1}^n \norm{e_i^{k+1}}^2} \\
&\leq &  \left(1- \frac{1}{\delta} \right)  \sum_{j=0}^k (\stepsize^k)^2\left(1+\frac{1}{4\delta}\right)^{k-j} \left(1-\frac{1}{2\delta}\right)^{k-j}  \left( 2\delta + B \right) \left( 4L(\Exp{f(x^j)} - f(x^*)) + 2D \right) \\
&&\quad + \left(1- \frac{1}{\delta} \right)  \sum_{j=0}^k (\stepsize^k)^2\left(1+\frac{1}{4\delta}\right)^{k-j} \left(1-\frac{1}{2\delta}\right)^{k-j} C \\
&\leq& (\stepsize^k)^2 \left(1- \frac{1}{\delta} \right)   \left( 2\delta + B \right) \sum_{j=0}^k\left[ \left(1-\frac{1}{4\delta}\right)^{k-j}4L\left( \Exp{f(x^j)} - f(x^*)\right) \right]\\
&+& (\stepsize^k)^2 \left(1- \frac{1}{\delta} \right)  4\delta [C + 2D(2\delta + B)] \,.
\end{eqnarray*}
As the last step, we use formula for geometric progression in the following way: $$\sum\limits_{j=0}^k \left(1 - \frac{1}{4\delta}\right)^{k-j} = \sum\limits_{j=0}^k \left(1 - \frac{1}{4\delta}\right)^{j} \leq \sum\limits_{j=0}^{\infty} \left(1 - \frac{1}{4\delta}\right)^{j} = 4\delta$$

By observing that the choice of the stepsize $\stepsize^k \leq \frac{1}{14 (2\delta + B) L}$:
\begin{eqnarray*}
\Exp{\norm{\frac{1}{n}\sum\limits_{i=1}^n e_i^{k+1}}^2} &\leq& \frac{1}{n} \Exp{\sum\limits_{i=1}^n \norm{e_i^{k+1}}^2} \\
&\leq& \frac{(1-1/\delta)}{49L (2\delta + B)} \sum_{j=0}^k\left[ \left(1-\frac{1}{4\delta}\right)^{k-j}(\Exp{f(x^j)} - f(x^*)) \right] \\
&&\qquad + \stepsize^k \frac{2(\delta-1)}{7 L} \left(2D + \frac{C}{2\delta + B}\right) \,,
\end{eqnarray*}
which concludes the proof of \eqref{eq:sparse_bound1}. For the second part, we use the previous results. Summing over all $k$:
\begin{eqnarray*}
 \sum_{k=0}^K w^k \cdot \Exp{\norm{\frac{1}{n}\sum\limits_{i=1}^n e_i^{k}}^2} &\stackrel{\eqref{eq:sparse_bound1}}{\leq}&  \frac{(1-1/\delta)}{49L (2\delta + B)} \sum_{k=0}^K w^k \sum_{j=0}^{k-1} \left(1-\frac{1}{4\delta}\right)^{k-j-1} \left(\Exp{f(x^j)} - f(x^*)\right) \\
 &&+ \frac{2(\delta-1)}{7 L} \left(2D + \frac{C}{2\delta + B}\right) \sum_{k=0}^K w^k \stepsize^{k-1}
\end{eqnarray*}
For $2\delta$-slow decreasing $\{(\stepsize^k)^2\}_{k \geq 0}$, it holds  $(\stepsize^{k-1})^2\leq (\stepsize^k)^2 \bigl(1+\frac{1}{4\delta} \bigr)$ which follows from \eqref{decrease} and $\stepsize^{k-1}\leq \stepsize^k \bigl(1+\frac{1}{4\delta} \bigr)$ and for $4\delta$-slow increasing $\{w^k\}_{k \geq 0}$ by \eqref{increase} we have $w^k \leq w^{k-j} \bigl(1+\frac{1}{8\delta}\bigr)^j$. Then
\begin{eqnarray*}
 \sum_{k=0}^K w^k \cdot \Exp{\norm{\frac{1}{n}\sum\limits_{i=1}^n e_i^{k}}^2} &\stackrel{\eqref{eq:sparse_bound1}}{\leq}&  \frac{(1-1/\delta)}{49L (2\delta + B)} \sum_{k=0}^K w^k  \sum_{j=0}^{k-1} \left(1-\frac{1}{4\delta}\right)^{k-j-1} \left(\Exp{f(x^j)} - f(x^*)\right) \\
 &&+ \frac{2(\delta-1)}{7 L} \left(2D + \frac{C}{2\delta + B} \right)\left(1+\frac{1}{4\delta}\right) \sum_{k=0}^K w^k \stepsize^{k} \\
 &\leq&  \frac{(1-1/\delta)}{49L (2\delta + B)} \sum_{k=0}^K  \sum_{j=0}^{k-1} w^{j} \left(1+\frac{1}{8\delta}\right)^{k-j} \left(1-\frac{1}{4\delta}\right)^{k-j} \left(\Exp{f(x^j)} - f(x^*)\right) \\
 &&+  \frac{\delta-1}{2 L} \left(2D + \frac{C}{2\delta + B}\right) \sum\limits_{k=0}^K w^k \stepsize^k\\
 &\leq& \frac{(1-1/\delta)}{49L (2\delta + B)}  \sum_{k=0}^K\sum_{j=0}^{k-1} w_{j} \left(1-\frac{1}{8\delta}\right)^{k-j}  \left(\Exp{f(x^j)} - f(x^*)\right) \\
 &&+  \frac{\delta-1}{2 L} \left(2D + \frac{C}{2\delta + B}\right) \sum\limits_{k=0}^K w^k \stepsize^k \\
 &\leq&  \frac{(1-1/\delta)}{49L (2\delta + B)} \sum_{k=0}^K w^k  \left(\Exp{f(x^k)} - f(x^*)\right) \sum_{j=0}^{\infty} \left(1-\frac{1}{8\delta}\right)^{j}  \\
 &&+  \frac{\delta-1}{2 L} \left(2D + \frac{C}{2\delta + B}\right) \sum\limits_{k=0}^K w^k \stepsize^k\,.
\end{eqnarray*}
Observing $\sum_{j=0}^\infty (1-\nicefrac{1}{8\delta})^j \leq 8\delta$ and using $\nicefrac{\delta - 1}{2\delta + B} \leq \nicefrac{1}{2}$ concludes the proof.
\end{proof}

\begin{lemma}[Lemma 11, \cite{stich2019}]
\label{lemma:decreasing}
For decreasing stepsizes $\bigl\{\stepsize^k \eqdef \frac{2}{a (\kappa + k)} \bigr\}_{k \geq 0}$, and weights $\{w_k \eqdef (\kappa + k)\}_{k \geq 0}$ for parameters $\kappa \geq 1$, it holds  for every non-negative sequence $\{r^k\}_{k \geq 0}$  and any  $a > 0$, $c \geq 0$ that
\begin{eqnarray*}
 \Psi^K \eqdef \frac{1}{W^K}\sum_{k=0}^K \left( \frac{w^k}{\stepsize^k} \left(1-a \stepsize^k \right) r^k - \frac{w^{k}}{\stepsize^k} r^{k+1} + c \stepsize^k w^k \right) \leq \frac{a \kappa^2 r^0}{K^2} + \frac{4c}{aK}\,,
\end{eqnarray*}
where $W^K \eqdef \sum_{k=0}^K w^k$.
\end{lemma}

\begin{proof}
We start by observing that
\begin{eqnarray}
 \frac{w^k}{\stepsize^k} \left(1-a \stepsize^k \right) r^k =  \frac{a}{2} (\kappa + k) (\kappa + k - 2) r^k = \frac{a}{2} \left((\kappa + k - 1)^2 -1 \right) \leq \frac{a}{2} (\kappa + k - 1)^2\,. \label{eq:424}
\end{eqnarray}
By plugging in the definitions of $\stepsize^k$ and $w^k$ in $\Psi^K$, we end up with the following telescoping sum:
\begin{eqnarray*}
\Psi^K \stackrel{\eqref{eq:424}}{\leq} \frac{1}{W^K} \sum_{k=0}^K \left(\frac{a}{2}  (\kappa + k - 1)^2 r^k - \frac{a}{2}(\kappa + k)^2 r^{k+1}  \right) + \sum_{k=0}^K  \frac{2c}{a W^K} \leq \frac{a (\kappa-1)^2 r^0}{2 W^K}  + \frac{2c(K+1)}{a W^K} .
\end{eqnarray*}
The lemma now follows from $(\kappa-1)^2 \leq \kappa ^2$ and
 $W^K = \sum_{k=0}^K(\kappa + k) = \frac{(2\kappa + K)(K+1)}{2} \geq \frac{K(K+1)}{2} \geq \frac{K^2}{2}$.
\end{proof}

\begin{lemma}[Lemma 12, \cite{stich2019}]
\label{lemma:constant}
For every non-negative sequence $\{r^k\}_{k\geq 0}$ and any parameters $d \geq a > 0$, $c \geq 0$, $K \geq 0$, there exists a constant $\stepsize \leq \frac{1}{d}$, such that for constant stepsizes $\{\stepsize^k = \stepsize\}_{k \geq 0}$ and weights $w^{k}\eqdef(1-a\stepsize)^{-(k+1)}$ it holds
\begin{eqnarray*}
\Psi^K \eqdef \frac{1}{W^K}\sum_{k=0}^K \left( \frac{w^k}{\stepsize^k} \left(1-a \stepsize^k \right) r^k - \frac{w^{k}}{\stepsize^k} r^{k+1} + c \stepsize^k w^k \right) = \tilde \cO \left(d r_0 \exp\left[- \frac{aK}{d} \right] + \frac{c}{aK}  \right)\,.
\end{eqnarray*}
\end{lemma}
\begin{proof}
By plugging in the values for $\stepsize^k$ and $w^k$, we observe that we again end up with a telescoping sum and estimate
\begin{eqnarray*}
 \Psi^K = \frac{1}{\stepsize W^K} \sum_{k=0}^K \left(w^{k-1}r^k - w^k r^{k+1} \right) + \frac{c\stepsize}{W^K}\sum_{k=0}^K w^k \leq \frac{r^0}{\stepsize W^K} +  c \stepsize \leq \frac{r^0}{\stepsize} \exp\left[-a \stepsize K \right] + c \stepsize\,,
\end{eqnarray*}
where we used the estimate $W^K \geq w^{K} \geq (1-a \stepsize)^{-K} \geq \exp[a \stepsize K]$ for the last inequality. The lemma now follows by carefully tuning $\stepsize$. 
\end{proof}

\begin{lemma}[Lemma 13, \cite{stich2019}]
\label{lemma:weakly}
For every non-negative sequence $\{r^k\}_{k\geq 0}$ and any parameters $d \geq 0$, $c \geq 0$, $K \geq 0$, there exists a constant $\stepsize \leq \frac{1}{d}$, such that for constant stepsizes $\{\stepsize^k = \stepsize\}_{k \geq 0}$ it holds:
\begin{eqnarray*}
\Psi^K  \eqdef \frac{1}{K+1} \sum_{k=0}^K \left( \frac{(1 - a \stepsize^k)r^k}{\stepsize^k} - \frac{r^{k+1}}{\stepsize^k} + c \stepsize^k \right) \leq \frac{(d-a) r^0}{K+1} + \frac{2\sqrt{c r^0}}{\sqrt{K+1}}
\end{eqnarray*}
\end{lemma}

\begin{proof}
For constant stepsizes $\stepsize^t = \stepsize$ we can derive the estimate
\begin{eqnarray*}
 \Psi^K = \frac{1}{\stepsize (K+1)} \sum_{k=0}^K \left( (1 - a \stepsize)r^k - r^{k+1} \right) + c \stepsize  \leq \frac{(1 - a \stepsize)r^0}{\stepsize(K+1)} + c \stepsize \,.
\end{eqnarray*}
We distinguish two cases: if $\frac{r^0}{c (K+1)} \leq \frac{1}{d^2}$, then we chose the stepsize $\stepsize = \sqrt{\frac{r^0}{c (K+1)}}$ and get
\begin{eqnarray*}
 \Psi^K \leq \frac{\sqrt{r^0}}{(K+1)}(2\cdot\sqrt{c(K+1)} - a\sqrt{r^0}) \,,
\end{eqnarray*}
on the other hand, if $\frac{r^0}{c (K+1)} > \frac{1}{d^2}$, then we choose $\stepsize=\frac{1}{d}$ and get
\begin{eqnarray*}
 \Psi^K \leq \frac{r^0(d - a)}{K+1} + \frac{c}{d} \leq \frac{r^0(d - a)}{K+1} + \frac{\sqrt{c r^0}}{\sqrt{K+1}}\,,
\end{eqnarray*}
which concludes the proof.
\end{proof}

The proof of the main theorem follows

\begin{proof}[Proof of the Theorem \ref{thm:sparsified}]
It is easy to see that $\nicefrac{1}{14(2\delta + B)L} \leq \nicefrac{1}{4L(1 + 2B/n)}$. This means that the Lemma~\ref{lemma:main} is satisfied. With the notation $r^k \eqdef \Exp{\norm{\tilde x^{k+1}-x^\star}^2}$ and $s^k \eqdef \Exp{f(x^k)}-f^\star$ we have for any $w^k > 0$:
\begin{eqnarray*}
 \frac{w^k}{2} s^k \stackrel{\eqref{eq:main}}{\leq} \frac{w^k}{\stepsize^k} \left(1-\frac{\mu \stepsize^k}{2}\right) r^k - \frac{w^k}{\stepsize^k} r^{k+1} + \stepsize^k w^k \frac{C+2BD}{n} + 3 w^k L \cdot \Exp{\norm{\frac{1}{n}\sum\limits_{i=1}^n e_i^k}^2}\,.
\end{eqnarray*}
Substituting \eqref{eq:sparse_bound3} and summing over $k$ we have:
\begin{eqnarray*}
 \frac{1}{2} \sum_{k=0}^K w^k s^k \leq \sum_{k=0}^K \left( \frac{w^k}{\stepsize^k} \left(1-\frac{\mu \stepsize^k}{2}\right) r^k - \frac{w^{k}}{\stepsize^k} r^{k+1} + \stepsize^k w^k \tilde C \right) + \frac{1}{4}\sum_{k=0} ^K w^k s^k  \,.
\end{eqnarray*}
where $\tilde C = C\left(1 + \frac{1}{n}\right) + D \left(\frac{2B}{n} + 3\delta\right)$ .

This can be rewritten as
\begin{eqnarray*}
 \frac{1}{W^K} \sum_{k=0}^K w^k s^k \leq  \frac{4}{W^K} \sum_{k=0}^K \left( \frac{w^k}{\stepsize^k} \left(1-\frac{\mu \stepsize^k}{2}\right) r^k - \frac{w^{k}}{\stepsize^k} r^{k+1} +  \stepsize^k w^k \tilde C \right).
\end{eqnarray*}
First, when the stepsizes $\stepsize^k = \frac{4}{\mu (\kappa + k)}$, it is easy to see that $\stepsize^k \leq \frac{1}{14(2\delta + B) L}$:
$$\stepsize^k \leq \stepsize^0 = \frac{4}{\mu \kappa} \leq  \frac{4}{\mu} \cdot \frac{\mu}{56(2\delta +B)L} =\frac{1}{14(2\delta + B) L}$$
Not difficult to check that $\{(\stepsize^k)^2\}_{k\geq 0}$ is $2\delta$ slow decreasing:
\begin{eqnarray*}
\frac{(\stepsize^{k+1})^2}{(\stepsize^k)^2}=\left(\frac{\kappa + k +1}{\kappa + k}\right)^2 \leq \left(1 + \frac{1}{\kappa + k}\right)^2 \leq \left(1 + \frac{1}{\kappa} \right)^2 = \left(1 + \frac{\mu}{56 (2\delta + B)L}\right)^2 \leq 1 + \frac{1}{4\delta}
\end{eqnarray*}
Furthermore, the weights $\{w^k = \kappa + k\}_{k \geq 0}$ are $4\delta$-slow increasing:
\begin{eqnarray*}
\frac{w^{k+1}}{w^k}=\frac{\kappa + k +1}{\kappa + k} = 1 + \frac{1}{\kappa + k} \leq 1 + \frac{1}{\kappa} = 1 + \frac{\mu}{56 (2\delta + B)L} \leq 1 + \frac{1}{8\delta}.
\end{eqnarray*}

The conditions for Lemma~\ref{lemma:decreasing} are satisfied, and we obtain the desired statement. For the second case, the conditions of Lemma~\ref{lemma:constant} are easy to check (see the previous paragraph). The claim follows by this lemma. Finally, for the third claim, we invoke Lemma~\ref{lemma:weakly}.
\end{proof}

\vskip 0.2in
\bibliography{literature}

\end{document}